\documentclass{article}

\def\E{\mathop{\mathbb{E}}} 

\newcommand\indep{\protect\mathpalette{\protect\independenT}{\perp}}
\def\independenT#1#2{\mathrel{\rlap{$#1#2$}\mkern2mu{#1#2}}}

\usepackage[usenames,dvipsnames]{xcolor}
\usepackage{graphicx}
\usepackage[labelfont=bf]{caption}

\usepackage{booktabs}

\usepackage{listings}
\usepackage{fancyvrb}
\fvset{fontsize=\normalsize}

\usepackage[acronym,shortcuts,smallcaps,nowarn,nohypertypes={acronym,notation}]{glossaries}

\lstdefinestyle{mystyle}{
    commentstyle=\color{OliveGreen},
    keywordstyle=\color{BurntOrange},
    numberstyle=\tiny\color{black!60},
    stringstyle=\color{MidnightBlue},
    basicstyle=\ttfamily,
    breakatwhitespace=false,
    breaklines=true,
    captionpos=b,
    keepspaces=true,
    numbers=left,
    numbersep=5pt,
    showspaces=false,
    showstringspaces=false,
    showtabs=false,
    tabsize=2
}
\usepackage[colorlinks,linktoc=all]{hyperref}
\usepackage[all]{hypcap}
\hypersetup{citecolor=MidnightBlue}
\hypersetup{linkcolor=MidnightBlue}
\hypersetup{urlcolor=MidnightBlue}

\usepackage[nameinlink]{cleveref}

\usepackage{tikz}
\usetikzlibrary{shapes,decorations,arrows,calc,arrows.meta,fit,positioning}
\tikzset{
    -Latex,auto,node distance =1 cm and 1 cm,semithick,
    state/.style ={circle, draw, minimum width = 0.7 cm},
    detstate/.style ={rectangle, draw, minimum width = 0.7 cm, minimum height = 0.7 cm},
    point/.style = {circle, draw, inner sep=0.04cm,fill,node contents={}},
    bidirected/.style={Latex-Latex,dashed},
    el/.style = {inner sep=2pt, align=left, sloped}
}

\usepackage{wrapfig}
\usepackage{mathtools}

\newacronym{KL}{kl}{Kullback-Leibler}
\newacronym{ELBO}{elbo}{\emph{evidence lower bound}}
\newacronym{POPELBO}{pop-elbo}{\emph{population evidence lower bound}}

\newacronym{SVI}{svi}{stochastic variational inference}
\newacronym{conc}{conc}{concordance}

\newacronym{svm}{svm}{svm}

\newacronym{BUMPVI}{bump-vi}{bumping variational inference}

\newacronym{GMM}{gmm}{Gaussian mixture model}
\newacronym{LDA}{lda}{latent Dirichlet allocation}

\newacronym{SUTVA}{sutva}{stable unit treatment value assumption}

\newacronym{KSD}{ksd}{{kernelized Stein discrepancy}}
\newacronym{KCC-SD}{kcc-sd}{kernelized complete conditional Stein discrepancy}

\newacronym{OPVI}{opvi}{operator variational inference}
\newacronym{SVGD}{svgd}{Stein variational gradient descent}

\newacronym{vde}{vde}{variational decoupling}
\newacronym{cfn}{cfn}{control-function method}
\newacronym{gcfn}{gcfn}{generalized control-function method}
\newacronym{2sls}{2sls}{two-stage least-squares method}
\newacronym{gmm}{gmm}{generalized method of moments}
\newacronym{iv}{iv}{instrumental variable}
\newacronym{cdf}{cdf}{cumulative distribution function}

\newacronym{x-cal}{x-cal}{explicit calibration}

\newacronym{d-cal}{d-calibration}{distributional calibration}
\newacronym{d-cal-short}{d-cal}{d-cal}

\newacronym{nll}{nll}{negative log likelihood}
\newacronym{ipcw}{ipcw}{inverse probability of censor-weighting}
\newacronym{bs}{bs}{Brier score}

\newacronym{fbs}{fbs}{F Brier score}
\newacronym{gbs}{gbs}{G Brier score}

\newacronym{fbscw}{fbscw}{Weighted F Brier score}
\newacronym{gbscw}{gbscw}{Weighted G Brier score}

\newacronym{bl}{bl}{Bernoulli likelihood}
\newacronym{bll}{bll}{Bernoulli log likelihood}
\newacronym{auc}{auc}{area under curve}
\newacronym{km}{km}{Kaplan-Meier}

\newacronym{gan}{gan}{Generative Adversarial Network}

\newacronym{support}{support}{Study to Understand Prognoses Preferences Outcomes and Risks of Treatment}
\newacronym{metabric}{metabric}{Molecular Taxonomy of Breast Cancer International Consortium}
\newacronym{rott}{rott}{Rotterdam Tumor Bank}
\newacronym{gbsg}{gbsg}{German Breast Cancer Study Group}
\newacronym{rott-gbsg}{rott. \& gbsg}{Rotterdam \& GBSG}
\newacronym{flchain}{flchain}{The Assay Of Serum Free Light Chain}
\newacronym{nwtco}{nwtco}{National Wilm’s Tumor Study}
\newacronym{crash2}{crash-2}{Clinical Randomization of an Antifibrinolyticin Significant Hemorrhage 2}
\newacronym{uw-dcal}{uw-dcal}{uniform-weighted d-cal}
\newacronym{ipcw-dcal}{ipcw-dcal}{\gls{ipcw} d-cal}
\newacronym{ipcw-xcal}{ipcw-xcal}{\gls{ipcw} x-cal}
\newacronym{ibs}{ibs}{integrated brier score}
\newacronym{ipcw-ibs}{ipcw-ibs}{\gls{ipcw}-\gls{ibs}}

\newacronym{crps}{crps}{continuous ranked probability score}
\newacronym{s-crps}{s-crps}{Survival-\acrshort{crps}}
\newacronym{ifd}{ifd}{individual failure distribution}

\newacronym{hl}{hl}{Hosmer-Lemeshow}
\newacronym{gb}{gb}{Grønnesby-Borgan}
\newacronym{dn}{dn}{D’Agostino-Nam}

\newacronym{ni}{ni}{Not-Interpolated}
\newacronym{i}{i}{Interpolated}
\newacronym{mimic-iii}{mimic-iii}{Medical Information Mart for Intensive Care}
\newacronym{mnist}{mnist}{Modified National Institute of Standards and Technology database}
\newacronym{tcga}{tcga}{The Cancer Genome Atlas}
\newacronym{mtlr}{mtlr}{Multi-Task Logistic Regression}
\newacronym{aft}{aft}{Accelerated Failure Times}

\renewcommand{\mid}{~\vert~}

\newcommand{\g}{\mid}
\newcommand{\gtight}{\vert}

\crefname{lemma}{lemma}{lemmas}
\crefname{prop}{proposition}{propositions}

\DeclareRobustCommand{\indicator}[1]{\ensuremath{\mathbbm{1}\left[#1\right]}}

\newcommand{\titrue}{\theta^\star_{Ti}}
\newcommand{\tihat}{\hat{\theta}_{Ti}}
\newcommand{\tkplusonetrue}{\theta^\star_{T(k+1)}}
\newcommand{\tkplusonehat}{\hat{\theta}_{T(k+1)}}

\newcommand{\citrue}{\theta^\star_{Ci}}
\newcommand{\cihat}{\hat{\theta}_{Ci}}
\newcommand{\ckplusonetrue}{\theta^\star_{C(k+1)}}
\newcommand{\ckplusonehat}{\hat{\theta}_{C(k+1)}}

\usepackage{algorithm}
\usepackage{algorithmic}
\usepackage{amsthm}       
\theoremstyle{plain}

\theoremstyle{definition}
\newtheorem*{assumption*}{Assumption}

\newtheorem{proposition}{Proposition}
\newtheorem*{proposition*}{Proposition}

\usepackage{bbm}
\usepackage{siunitx}
\usepackage{amssymb}
\usepackage[footnotes,definitionLists,hashEnumerators,smartEllipses,hybrid]{markdown}
\usepackage{graphicx}
\usepackage{subfigure}
\usepackage{color}
\usepackage{soul}

\usepackage[final]{neurips_2021}

\usepackage[utf8]{inputenc} 
\usepackage[T1]{fontenc}    
\usepackage{hyperref}       
\usepackage{url}            
\usepackage{booktabs}       
\usepackage{amsfonts}       
\usepackage{nicefrac}       
\usepackage{microtype}      
\usepackage{xcolor}         

\title{Inverse-Weighted Survival Games}

\author{%
Xintian Han
  \thanks{Equal Contribution.} \\
  NYU\\
  \texttt{xintian.han@nyu.edu} \\
 \And  
 Mark Goldstein
 \footnotemark[1] \\ 
 NYU \\ 
 \texttt{goldstein@nyu.edu} \\ 
 \And 
 Aahlad Puli \\ 
 NYU \\ 
 \texttt{aahlad@nyu.edu}
 \AND 
 Thomas Wies\\
 NYU \\ 
 \texttt{wies@cs.nyu.edu}\\ 
 \And 
  Adler J. Perotte\\
 Columbia University\\
  \texttt{adler.perotte@columbia.edu} \\
   \And 
    Rajesh Ranganath\\
  NYU\\
  \texttt{rajeshr@cims.nyu.edu} \\
}

\begin{document}

\maketitle

\begin{abstract}
Deep models trained through maximum likelihood have achieved state-of-the-art results for survival analysis.
Despite this training scheme,
practitioners evaluate models
under other criteria, such as
binary classification losses at a chosen set of time horizons,
e.g.  \gls{bs} and \gls{bll}.
Models trained with maximum likelihood may have poor \gls{bs} or \gls{bll} since maximum likelihood does not directly optimize these criteria.
Directly optimizing criteria like \gls{bs} 
requires inverse-weighting by the censoring distribution.
However, estimating the censoring model under these metrics
requires inverse-weighting by the failure distribution.
The objective for each model requires the other,
but neither are known.
To resolve this dilemma, we introduce \textit{Inverse-Weighted Survival Games}. 
In these games,
objectives for each model are built from re-weighted estimates
featuring the other model, where the 
latter
is held fixed during training.
When the loss is proper, we show that the games always have the
true failure and censoring 
distributions as a stationary point. This means models in the game do not leave the correct distributions once reached. 
We construct one case where this stationary point is unique.
We show that these games optimize 
\gls{bs} on simulations and then apply these principles on real world
cancer and critically-ill patient data.
\end{abstract}

\glsresetall

\section{Introduction}
Survival analysis is the modeling of time-to-event distributions and is widely used in healthcare to predict time from diagnosis to death, risk of disease recurrence, and 
changes in
level of care. In survival data, events, known as \textit{failures}, are often right-censored, i.e., only a lower bound on the time is observed, for instance, when a patient leaves a study before failing. Under certain assumptions, maximum likelihood estimators are consistent for survival modeling~\citep{kalbfleisch2002}.

Recently, deep survival models have obtained state-of-the-art results 
\citep{ranganath2016deep,alaa2017deep,katzman2018deepsurv,kvamme2019time,zhong2019survival}. 
Common among these are 
discrete-time models
\citep{yu2011learning,lee2018deephit,fotso2018deep,lee2019temporal,ren2019deep,kvamme2019continuous,kamran2021estimating,goldstein2020x,sloma2021empirical}
even when data are continuous
because they 
can borrow classification architectures and
flexibly
approximate continuous densities~\citep{miscouridou2018deep}.

Though training is often based on maximum likelihood, 
criteria such as 
\gls{bs} and \gls{bll} have been used to evaluate survival models
\citep{haider2020effective}.
The \gls{bs} and \gls{bll} are 
classification losses 
adapted for survival
by treating the model as a binary classifier at various time horizons (\textit{will the event occur before or after 5 years?})
 \citep{kvamme2019continuous,lee2019temporal,steingrimsson2020deep}.
\Gls{bs} can also be motivated by calibration (\cref{sec:timedependentloss})
which is valuable because survival probabilities are used 
to communicate risk \citep{sullivan2004presentation}. 
However \gls{bs} and \gls{bll}
are challenging to estimate
because they require
\gls{ipcw}, which depends on the true censoring distribution \citep{van2003unified}.

Though consistent, due to finite data, maximum likelihood may lead to models with poor  \gls{bs} and \gls{bll}. 
But directly optimizing these criteria
is challenging
because \gls{ipcw} estimation
requires solving an additional survival modeling problem to estimate the unknown censoring distribution. 
This poses a re-weighting dilemma: each model is required 
for training the other under these criteria but neither are known.

To resolve the dilemma,
we introduce  
\textit{Inverse-Weighted Survival Games}
for training with respect to criteria such as \gls{bs} and \gls{bll}. 
We pose survival analysis as a game
with the failure and censoring models as players. Each model's loss is built from \gls{ipcw} estimates featuring the other model.  Inspired by game theory \citep{neumann2007theory,letcher2019differentiable}, we ask:
should the censoring model's re-weighting role in the failure objective be considered part of the censoring objective? We find the answer to be no.
In each step of training, 
each model follows
gradients of its loss with the other model held fixed to compute weights.

When the loss is \textit{proper}
(e.g. \gls{bs}, \gls{bll}) \citep{gneiting2007strictly}, we show
that games
have the true failure and censoring distributions as a stationary point. This means the models in the game do not leave the correct distributions once reached.
We then describe one case
where this stationary point is unique.
Finally, we show 
that inverse-weighted game training 
achieves better \gls{bs} and \gls{bll} than maximum likelihood methods on simulations and  real world cancer and ill-patient data.\footnote{Code is available at \href{https://github.com/rajesh-lab/Inverse-Weighted-Survival-Games}{https://github.com/rajesh-lab/Inverse-Weighted-Survival-Games}}

\section{Notation and background on \acrshort{ipcw} \label{sec:notationassumptions}}

\paragraph{Notation.}  Let $T$ be a failure time with \acrshort{cdf} $F(t)=P(T \leq t)$ , density $f$, survival function $\overline{F}=1-F$,
and model $F_{\theta_T}$.  Let $C$
be a censoring time with \acrshort{cdf} $G$, density $g$, $\overline{G}=1-G$, and model $G_{\theta_C}$. This means $\overline{G}(t)=P(C >t)$.
Let $\overline{G}(t^-)$ denote $P(C \geq t)$. 
We observe features $X$, time $U=\min(T,C)$ and $\Delta=\indicator{T \leq C}$.
For discrete models over $K$ times, let
$\theta_{Tt}=P_\theta(T=t)$
and
$\theta_{Ct}=P_\theta(C=t)$.

\paragraph{Models.}
We focus on deep discrete models
like those studied in
\cite{lee2018deephit,kvamme2019continuous}.
The model maps inputs $X$ to a categorical distribution over times. When
the observations are continuous, a discretization scheme is necessary.
Following \cite{kvamme2019continuous,goldstein2020x}, we set bins to correspond to quantiles of observed times. We represent all times by the lower boundary of their respective interval.

\paragraph{Assumptions.} We assume i.i.d. data and random censoring: $T \indep C \g X$ \citep{kalbfleisch2002}.
We also require the censoring positivity assumption
\citep{gerds2013estimating}.
Let $f=dF$. Then:
\begin{align}
    \label{eq:positivity}
        \exists \epsilon 
    \quad \text{ s.t. } \quad 
   \forall x \,
   \forall t \in 
        \{t\leq t_{\text{max}} \mid  f(t|x) > 0\},
\quad
    \overline{G}(t^{-} \gtight x) \geq
    \epsilon > 0,
\end{align}
i.e.
it is possible that censoring events occur late-enough for us to observe failures up until a maximum time $t_{\text{max}}$.
Truncating at a maximum time is necessary in practice for continuous distributions because datasets may have no samples in the tails, leading to practical positivity violations \citep{gerds2013estimating}.
This truncation happens implicitly for categorical models by choosing the last bin.

In this work, we model the censoring distribution.
This task is dual to the original survival problem: the roles of censoring and failure times are reversed.
Therefore, to observe censoring events properly,
we also require a version of 
\cref{eq:positivity}
to hold with the roles of $F$ and $G$ reversed  (\cref{appsec:notationassumptionsappendix}).

\paragraph{\Acrshort{ipcw} estimators.} \glsreset{ipcw} \Gls{ipcw} is a method for estimation under censoring
\citep{van2003unified,bang2005doubly}. Consider the marginal mean $\E[T]$. \Gls{ipcw} reformulates such
expectations in terms of observed data. Using \gls{ipcw}, we can show that:
\begin{align*} 
     \E
     [T] 
       = 
       \E_X\E_{T|X}\Big[
       \frac{\E[\indicator{T \leq C|X}]}
       {\E[\indicator{T \leq C|X}]}
       T
       \Big]= 
       \E_X \E_{T|X} \E_{C|X}
    \Bigg[\frac{\indicator{T \leq C }}
    {\overline{G}(T^{-} \gtight X)} T \Bigg] =
    \E_{T,C,X}
    \Bigg[ 
    \frac{\Delta U}{\overline{G}(U^{-}\gtight X)} 
    \Bigg] 
\end{align*}
We derive this fully in \cref{appsec:ipcwprimer}.
The second equality holds because $\Delta= 1 \implies U=T$ and means we can identify
$\E[T|X]$
provided that we know $G$ and that
random censoring and positivity 
hold.

\section{Time-dependent survival evaluations
\label{sec:timedependentloss}}
\glsreset{bs} 
\glsreset{bll}
\begin{figure}
    \centering
    \includegraphics[width=90mm]{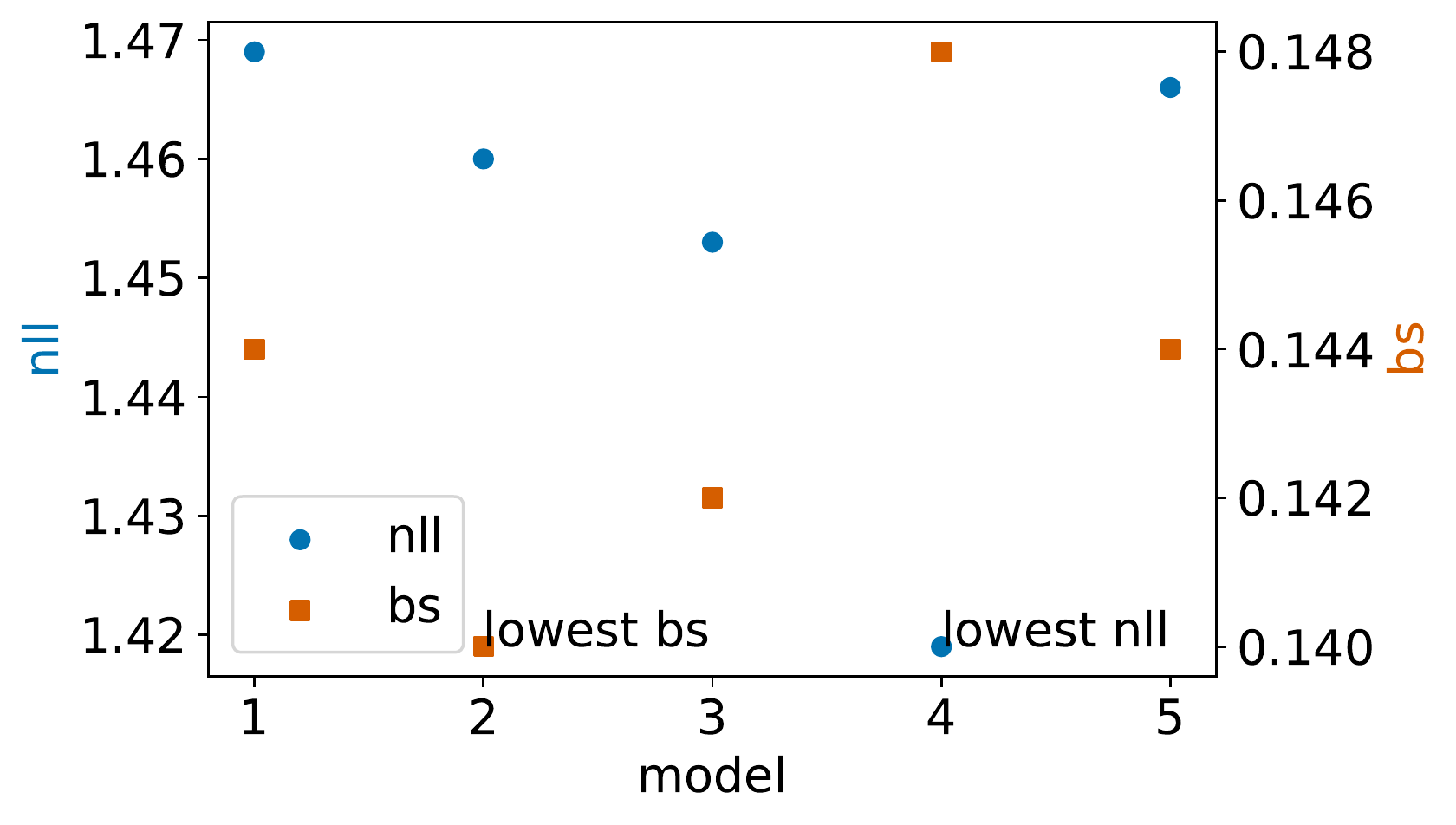}
    \caption{Test \acrshort{nll} and \acrshort{bs} for 5 different models, each trained with a different learning rate
    \label{fig:diff_plot}}
\end{figure}
\Gls{bs} \citep{brier1951verification} is 
\textit{proper} for classification,
meaning that it has a minimum at the true data distribution~\citep{gneiting2007strictly}. The \gls{bs} is often adapted for survival evaluations~\citep{lee2019temporal,kvamme2019time,haider2020effective}.
For time $t$, it computes 
differences
between the \acrshort{cdf} and true event status at $t$, turning survival analysis into a classification problem at a given time horizon:
\begin{align}
\label{eq:fbs}
 \text{BS}(t;\theta) &= \E
 \Big[
   \Big( F_{\theta_T}(t \mid  X) - 
    \indicator{T \leq t}\Big)^2
   \Big]
\end{align}
\Gls{bs} is often used as a proxy for marginal calibration error \citep{kumar2018trainable,lee2019temporal}, 
which
measures differences between
\acrshort{cdf} levels $\alpha \in [0,1]$ and observed proportions of datapoints with  $F_\theta(T|X) < \alpha$
\citep{demler2015tests}. This usage of \gls{bs} stems from its decomposition
into calibration plus a refinement (discriminative) term~\citep{degroot1983comparison}.

Unfortunately one cannot compute \gls{bs} unmodified 
since $\indicator{T \leq t}$
is unobserved for a point censored before $t$.
\Gls{ipcw} \gls{bs} \citep{graf1999assessment,gerds2006consistent} estimates
\gls{bs}$(t)$ under censoring:
\begin{align}
    \label{eq:fbscw}
    \text{BS}(t;\theta)
   &= \E
   \Big[\frac{\overline{F}_{\theta_T}(t \mid X)^2 \Delta  \indicator{U \leq t}}{\overline{G}(U^{-} \mid X)}
        + \frac{F_{\theta_T}(t \mid X)^2 \indicator{U > t}}{\overline{G}(t \mid X)}\Big].
\end{align}
\cref{eq:fbscw} is equivalent to \cref{eq:fbs} (\cref{appsec:deriveipcwbrier}).
Negative \gls{bll} 
is similar,
but with log loss 
(\cref{appsec:bll}).
\gls{bs} and \gls{bll} are proper for classification at each time $t$, so
their sum or integral over $t$
is still proper 
(\cref{appsec:sumgame}). 

\paragraph{Proper objectives differ.}
Though \gls{nll}, \gls{bs} and \gls{bll}  all have the same true distribution at optimum with infinite data, they may yield significantly different solutions in practice.
For example, \gls{nll}-trained models may not achieve good \gls{bs}~\citep{kvamme2019brier}. In \cref{fig:diff_plot}, we show test set \gls{nll} and \gls{bs} for 5 models trained with \gls{nll}
at different learning rates on Gamma-simulated data (described in \cref{sec:gamma}). \Gls{nll} does not align with \gls{bs}: models that have low \gls{nll} may not have low \gls{bs}. Model 4 has the lowest \gls{nll} but not the lowest \gls{bs}.
When a practitioner requires good performance under 
\gls{bs} or \gls{bll}, they should optimize directly for those metrics.

\paragraph{Re-weighting dilemma.} Censoring introduces challenges
because we must use \gls{ipcw} to estimate \gls{bs} and \gls{bll}. Crucially, the $G$ in \cref{eq:fbscw} is the true censoring distribution rather than a model, but during training, we only have access to models.
This poses a dilemma: can the models be used 
in re-weighting estimates during training 
to successfully optimize these criteria under censoring? 
\section{Inverse-Weighted Survival Games}

A reasonable attempt to solve the dilemma 
is to jointly optimize the sum of $F_\theta$ and $G_\theta$'s 
losses where each model re-weights the other's loss.
The expectation is that both models will improve over training and yield
reliable
\gls{ipcw} estimates for each other.  Concretely,
consider this for \cref{eq:fbscw} plus the same objective with the roles of
$F_\theta$ and $G_\theta$  reversed. 
Unfortunately, there exist solutions to this optimization problem with smaller loss than for the distributions from which the data was generated,
making this summed objective improper for the pair of distributions.
In \cref{fig:gradminplot}, we plot this for \gls{ipcw} \gls{bs}$(t=1)$ for models over two timesteps\footnote{$\gls{bs}(t=1)$ is proper for distributions with support over two timesteps because $\gls{bs}(t=K)$ for a model with support over $K$ timesteps is always $0$, so the summed \gls{bs} equals $\gls{bs}(1)$.} as a function of each model's single parameter.
 
\begin{figure}[t]
    \centering
    \subfigure[Contours of summed objective]{
    \label{fig:minplot}
        \includegraphics[height=41.5mm]{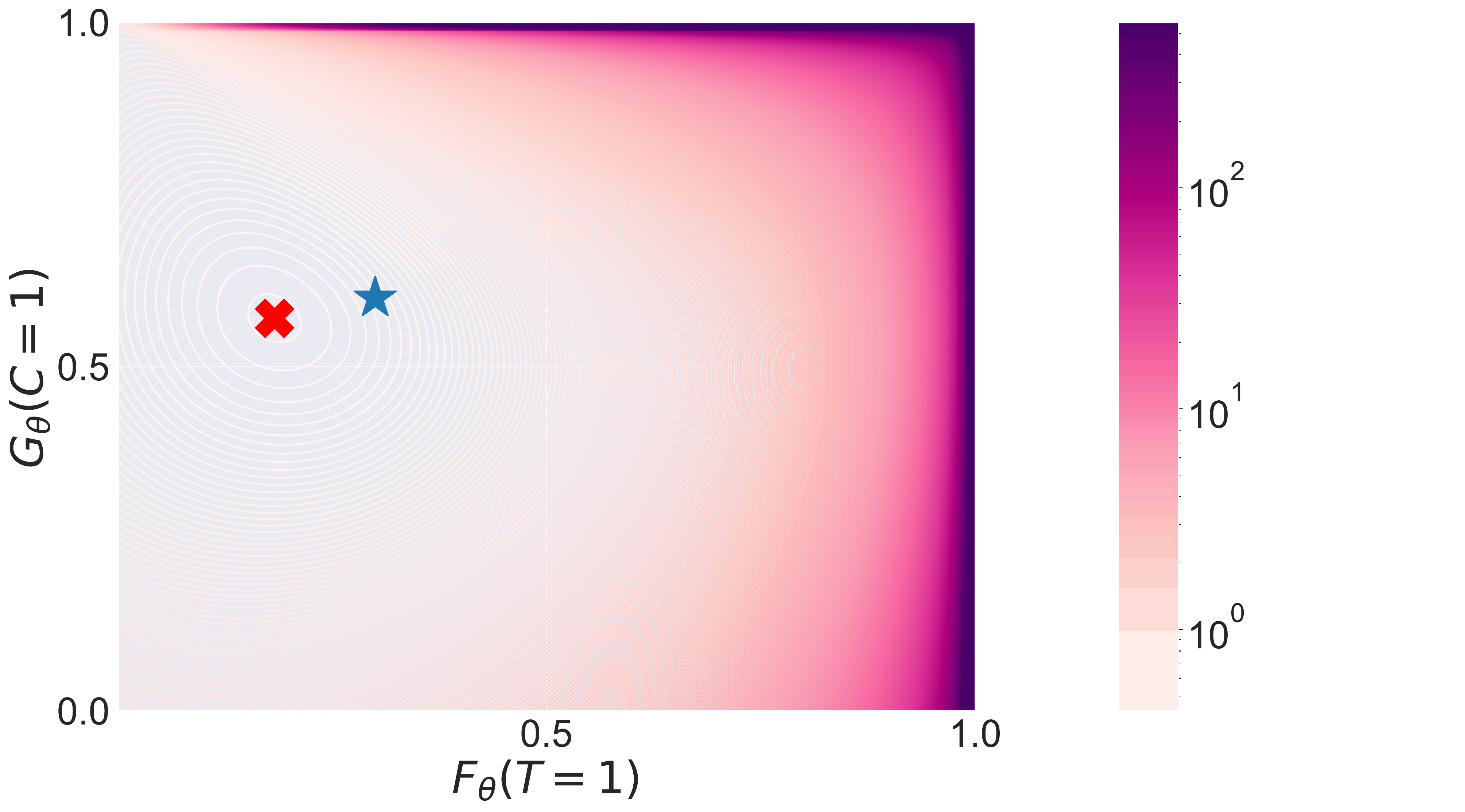}
    }
    \hspace{-4.0em}
    \subfigure[Gradients of the game]{
    \label{fig:gradplot}
        \includegraphics[height=41.5mm]{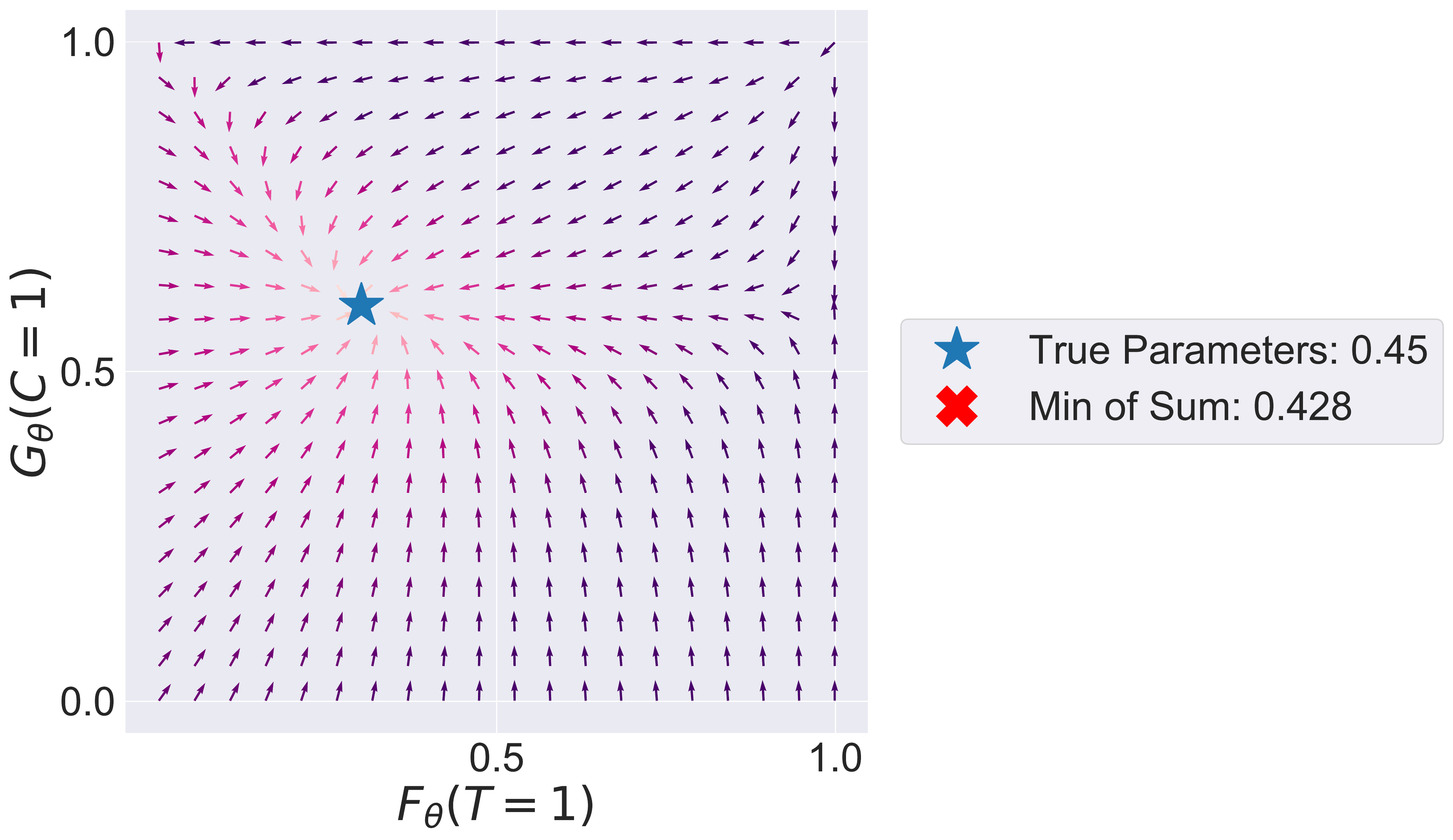}
    }
    \caption{
    \Cref{fig:minplot}: the sum of $F$ and $G$'s \gls{ipcw} \gls{bs}(1) scores is an improper joint objective for the failure and censoring models.
    \Cref{fig:gradplot}: in contrast, as shown in the gradient field,
    the one timestep game has a unique stationary point at the true data generating distributions.\label{fig:gradminplot}
 }
\end{figure}

To address this phenomenon, we introduce \textit{Inverse-Weighted Survival Games}.
In these games,
a  \textit{failure player} and \textit{censor player}
simultaneously minimize their own
loss function.
The failure and censoring model
are featured in both loss functions
and playing the game results
in a trained failure and censoring model. We show in experiments that these games produce models with good \gls{bs}, \gls{bll}, and concordance relative to those trained with maximum likelihood.

For simplicity, we present the games for marginal categorical models. The analysis can be extended to conditional parameterizations with the usual
caveats shared by maximum likelihood. Our experiments explore the conditional setting.
\subsection{Basic definition of game} 

We follow the setup in \cite{letcher2019differentiable}.
A differentiable $n$-player game consists of $n$ players each with loss $\ell_i$ and parameter (or state) $\theta_i$.
Player $i$ controls \textit{only} $\theta_i$ and aims
to minimize $\ell_i$. However, each $\ell_i$ is a function of the whole state $\theta=(\theta_i,\theta_{-i})$\footnote{For two players, when $\ell_1 = -\ell_2$ the game is called \textit{zero-sum} and can be written as a minimax game, sometimes referred to as \textit{adversarial} (e.g. as in \glspl{gan}
\citep{goodfellow2014generative}). Games with $\ell_1 \neq -\ell_2$ are called non-zero-sum. }.
The \textit{simultaneous gradient} is the gradient of the losses w.r.t. each players' respective parameters:
\begin{align*}
    \xi(\theta) = 
    [ 
    \nabla_{\theta_1} \ell_1
    ,\ldots,
    ,\nabla_{\theta_n} \ell_n
    ]
\end{align*}
The \textit{dynamics} of the game
refers to following $-\xi$.
The solution concepts in games are equilibria (the game analog of optima) and stationary points.
One necessary condition for 
equilibria is 
finding a \textit{stationary point}
$\theta^\star$ such that 
$\xi(\theta^\star)=0$.
 The
simplest algorithm follows the dynamics to find stationary points and is called simultaneous gradient descent. With learning rate $\eta$,
\begin{align*}
   \theta \leftarrow \theta - \eta \xi(\theta). 
\end{align*}
This can be interpreted as each player taking their best possible move at each instant.  
\subsection{Constructing survival games}
We specify an Inverse-Weighted Survival Game as follows.
First, choose a loss $L$
used to construct the
 losses for each player.
Next, derive the \gls{ipcw} form $L_I$ that can be
used to compute $L$ under censoring:
for true failure and censoring distributions $F^\star$ and $G^\star$, the losses $L$ and $L_I$ are related through $L_I(F_{\theta_T};G^\star) = L(F_{\theta_T})$ and $L_I(G_{\theta_C};F^\star) = L(G_{\theta_C})$.
The loss functions for the two players are defined as:
\begin{align}
\label{eq:iwgame}
\begin{split}
    \ell_{F}(\theta) \triangleq 
 L_{I}(F_{\theta_T};G_{\theta_C}), \quad 
\ell_{G}(\theta) \triangleq
    L_{I}(G_{\theta_C}; F_{\theta_T})
    \end{split}
\end{align}
Compared to \cref{eq:fbscw}, we have replaced the true re-weighting distributions with models. Finally, the failure player and censor player minimize their
respective loss functions w.r.t. only their own parameters:
\begin{align*}
\texttt{failure player:}
    \min_{\theta_T} \ell_F, 
    \quad \quad 
    \texttt{censor player:}
    \min_{\theta_C} \ell_G
\end{align*} 
One 
example
of these games is the \gls{ipcw} \gls{bs}$(t)$ game, derived in \cref{appsec:deriveipcwbrier}.
With $\overline{\Delta}=1-\Delta$,
\begin{align}
\label{eq:iwbs}
\begin{split}
\ell_{F}^t(\theta)
&=
 \E
 \Big[\frac{\overline{F}_{\theta_T}(t)^2 \Delta  \indicator{U \leq t}}{\overline{G}_{\theta_C}(U^{-} )}
        + \frac{F_{\theta_T}(t)^2 \indicator{U > t}}{\overline{G}_{\theta_C}(t )}\Big]\\
\ell_{G}^t(\theta)
 &=
   \E \Big[\frac{\overline{G}_{\theta_C}(t)^2
   \overline{\Delta} \indicator{U \leq t}}{\overline{F}_{\theta_T}(U)}
        + \frac{G_{\theta_C}(t )^2 \indicator{U > t}}{\overline{F}_{\theta_T}(t )}\Big].
  \end{split}
\end{align}
In \cref{sec:prop1section},
we show that
this formulation (\cref{fig:gradplot}) has formal advantages over the optimization
in \cref{fig:minplot} for particular choices of $L$.
\paragraph{Multiple Timesteps.} 
The example is specified for a given $t$,
but the games can be designed for multiple timesteps.
We use \gls{bs} for a $K$ timestep model
to demonstrate.
$\gls{bs}(K)$ is $0$ for any model: the left terms contain $\overline{F}_{\theta_T}$
and $\overline{G}_{\theta_C}$, which are are both $0$ when evaluated at $K$; in the right terms, $\indicator{U>K}$ is always $0$. One option is to 
define summed games with:
\begin{align*} \ell_F=\sum_{t=1}^{K-1} \ell_F^t,
 \quad 
 \ell_G=\sum_{t=1}^{K-1} \ell_G^t
\end{align*}
The summed game is shown in \cref{alg:sum}.
Alternatively, instead of the sum,
it is possible to find solutions
for all timesteps 
with respect to one pair of models $(F_\theta,G_\theta)$.
For $K\text{-}1$ timesteps this can be formalized as a $2(K\text{-}1)$-player game: there is a failure player and censor player for the loss at each $t$:
\begin{align*}
   t^{th}\text{-}\texttt{failure player:}
\min_{\theta_{Tt}}
\ell_F^t,
\quad 
\quad 
t^{th}\text{-}\texttt{censor player:}
\min_{\theta_{Ct}}
\ell_G^t
\end{align*}
We study theory that applies to both approaches in \cref{sec:prop1section},
namely that the true failure and censoring distribution are stationary points in both types of games. We prove additional properties 
about uniqueness of the stationary point
for a special case of the multi-player game
in \cref{sec:prop2section}.
Summed games are more stable to optimize in practice
because they optimize objectives at all time steps w.r.t all parameters, while multiplayer games can only improve each time step's loss
w.r.t. one parameter.
We study the summed games
empirically
in \cref{sec:experiments}.
\begin{algorithm}[h]
\begin{algorithmic}
 \STATE {\bfseries Input:}
    Choice of losses $\ell_F,\ell_G$,
    learning rate $\gamma$
 \STATE {\bfseries Initialize:} $T$ model parameters $\theta_{T}$ and $C$ model parameters $\theta_{C}$ randomly
 \REPEAT 
 \STATE{Set $g_T=0$ and $g_C=0$}
     \FOR{$t=1$ {\bfseries to} $K-1$ }
        \STATE $g_T = g_T + d\ell_F^t/d \theta_T $
     \STATE $g_C = g_C + d\ell_G^t/d \theta_C $
     \ENDFOR
   \STATE $\theta_T \leftarrow \theta_T - \gamma g_T$ and $\theta_C \leftarrow \theta_C - \gamma g_C$
 \UNTIL{convergence}
\STATE{ \bfseries Output: $\theta_T,\theta_C$}
\end{algorithmic}
\vskip -0.05in
\caption{\label{alg:sum}Following Gradients in Summed Games}
\end{algorithm}

\subsection{\Gls{ipcw} games have a stationary point at data distributions \label{sec:prop1section}} 
Among a game's stationary points
should be the true failure and censoring distributions.

\begin{proposition}
\label{prop: exist}
Assume $\exists \theta_T^\star \in \Theta_T,\exists \theta_C^\star \in \Theta_C$ such that $F^\star=F_{\theta_T^\star}$ and
$G^\star=G_{\theta_C^\star}$.
Assume the game losses $\ell_F,\ell_G$ are based on proper losses $L$
and that the games are only computed at times for which positivity holds.
Then $(\theta_T^\star,\theta_C^\star)$ is a stationary point of the game \cref{eq:iwgame}.
\end{proposition}
The proof is in \cref{appsec:ipcwgames}. 
The result holds for summed and multi-player games
using
\gls{bs}, \gls{bll},
or other proper scoring rules such as \acrshort{auc}.\footnote{Though often reported, the time-dependent concordance$(t)$ is not proper \citep{blanche2019c}.}
When games are built from such objectives, 
the set of solutions includes $(\theta_T^\star,\theta_C^\star)$ and models do not leave this correct
solution when reached.
Under the stated assumptions,
this result holds for discrete and continuous distributions.
However, as mentioned in \cref{sec:notationassumptions},
in practice a truncation time must be picked to ensure the assumptions are met for continuous distributions.

\subsection{Uniqueness for Discrete Brier Games \label{sec:prop2section}} 
We provide a stronger result for the
\gls{bs} game in \cref{eq:iwbs} when solving all timesteps with multi-player games: its only stationary point is located at the true failure and censoring distributions. 

\begin{proposition}
\label{prop: uniq}
For discrete models over $K$ timesteps,
assuming that $\theta^\star_{Tt}>0$ and $\theta^\star_{Ct}>0$,
the solution
 $(\theta_{T}^\star, \theta_{C}^\star)$ is the only stationary point for the multi-player \gls{bs} game
 shown in \cref{alg:mul-step}
 for times $t \in \{1,\ldots,K-1\}$ 
\end{proposition}
The proof is in \cref{appsec:brierstationary}. To illustrate this,
\cref{fig:gradplot} shows that,
unlike the minimization in 
\cref{fig:minplot},
the \gls{ipcw} \gls{bs} game moves to the correct solution at its unique stationary point.

\section{Experiments \label{sec:experiments}}

We run experiments on a simulation with conditionally Gamma 
times, a semi-simulated survival dataset based on \acrshort{mnist}, several sources of cancer data, and data on critically-ill hospital patients.

\paragraph{Losses.}
We build categorical models in 3 ways: the standard \gls{nll} method
(\cref{eq:failurecensornll}), the \gls{ipcw} \gls{bs} game and the negative \gls{ipcw} \gls{bll} game.

\paragraph{Metrics.}
For these models we report \gls{bs} (uncensored for simulations and \gls{km}-weighted for real data), \gls{bll} (also uncensored or weighted), concordance which measures the proportion of pairs whose predicted risks are ordered correctly \citep{harrell1996multivariable},
and \acrshort{nll}. We report mean and standard deviation of the metrics over 5 different seeds. In all plots, the middle solid line represents the mean and the shaded band represents the standard deviation.

\paragraph{Model Description.} In all experiments except for \acrshort{mnist}, we use
a 3-hidden-layer ReLU network that outputs $20$ categorical bins
(more bin choices in \cref{appsec:numbins}).
For \acrshort{mnist} we first use a small convolutional network and follow with the same fully-connected network.

Model and training details including learning rate and model selection can be found in \cref{appsec:experiments}.
\subsection{Simulation Studies \label{sec:gamma}} 

\paragraph{Data.} We draw $X \in \mathbb{R}^{32} \sim \mathcal{N}(0,10I)$
and $T \sim \text{Gamma}(\text{mean}=\mu_t)$
where $\mu_t$ a log-linear function of $X$.
The censoring times are also gamma
with mean $0.9*\mu_t$. Both distributions have constant variance $0.05$. It holds that $T \indep C \g X$. Each random seed draws a new dataset.

\paragraph{Results.} 
\cref{fig:gamma}
demonstrates that the games optimize the true uncensored \gls{bs}, and, though more slowly with respect to training size, log-likelihood does too. The games have better test-set performance on all metrics for small training size. All methods converge on similar performance when there is enough data (though \textit{enough} is highly-dependent on dimensionality and model class).
\begin{figure}
    \centering
    \subfigure[Uncensored \acrshort{bs}]{
        \includegraphics[height=25mm]{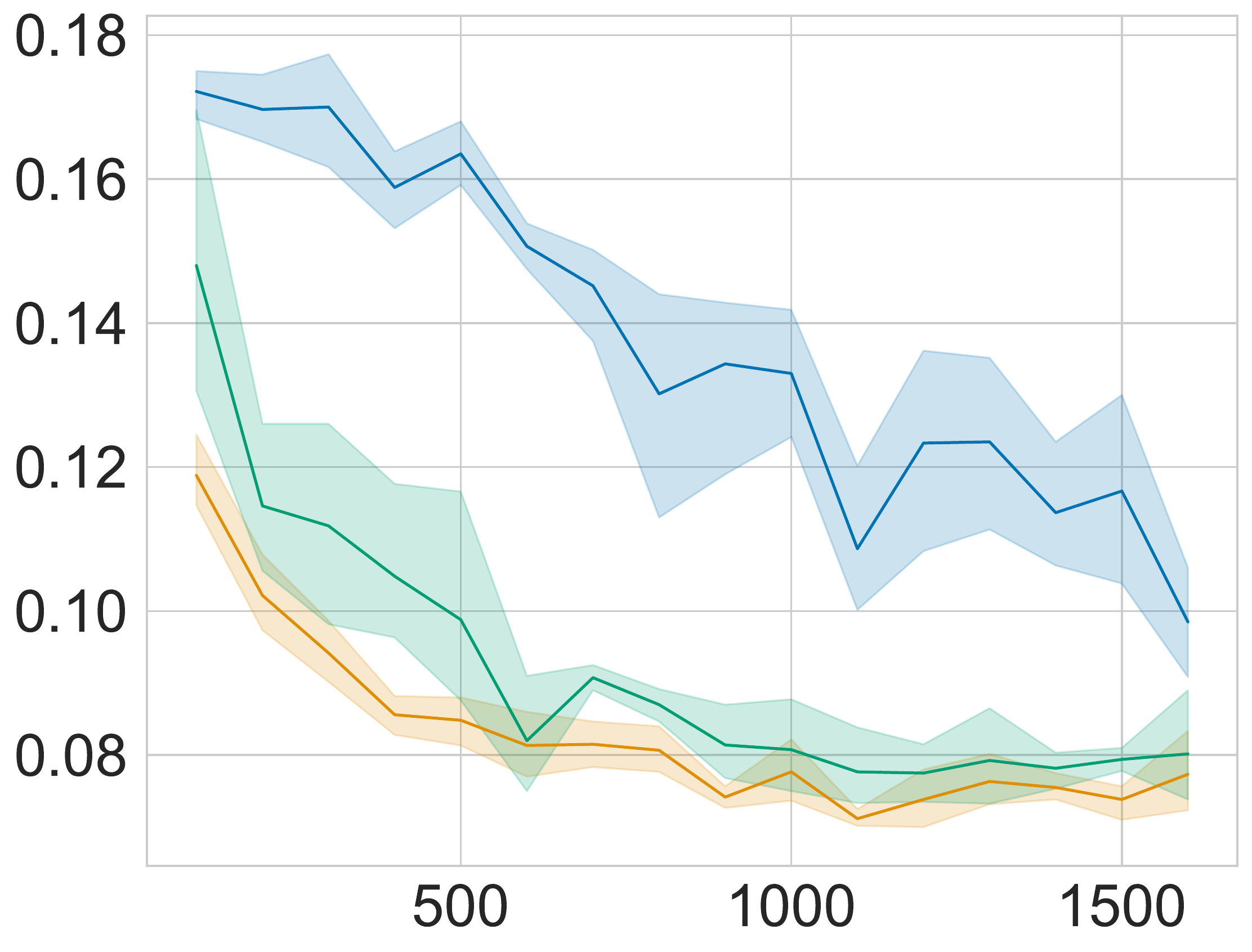}
    }
    \subfigure[Uncensored Neg \acrshort{bll}]{
        \includegraphics[height=25mm]{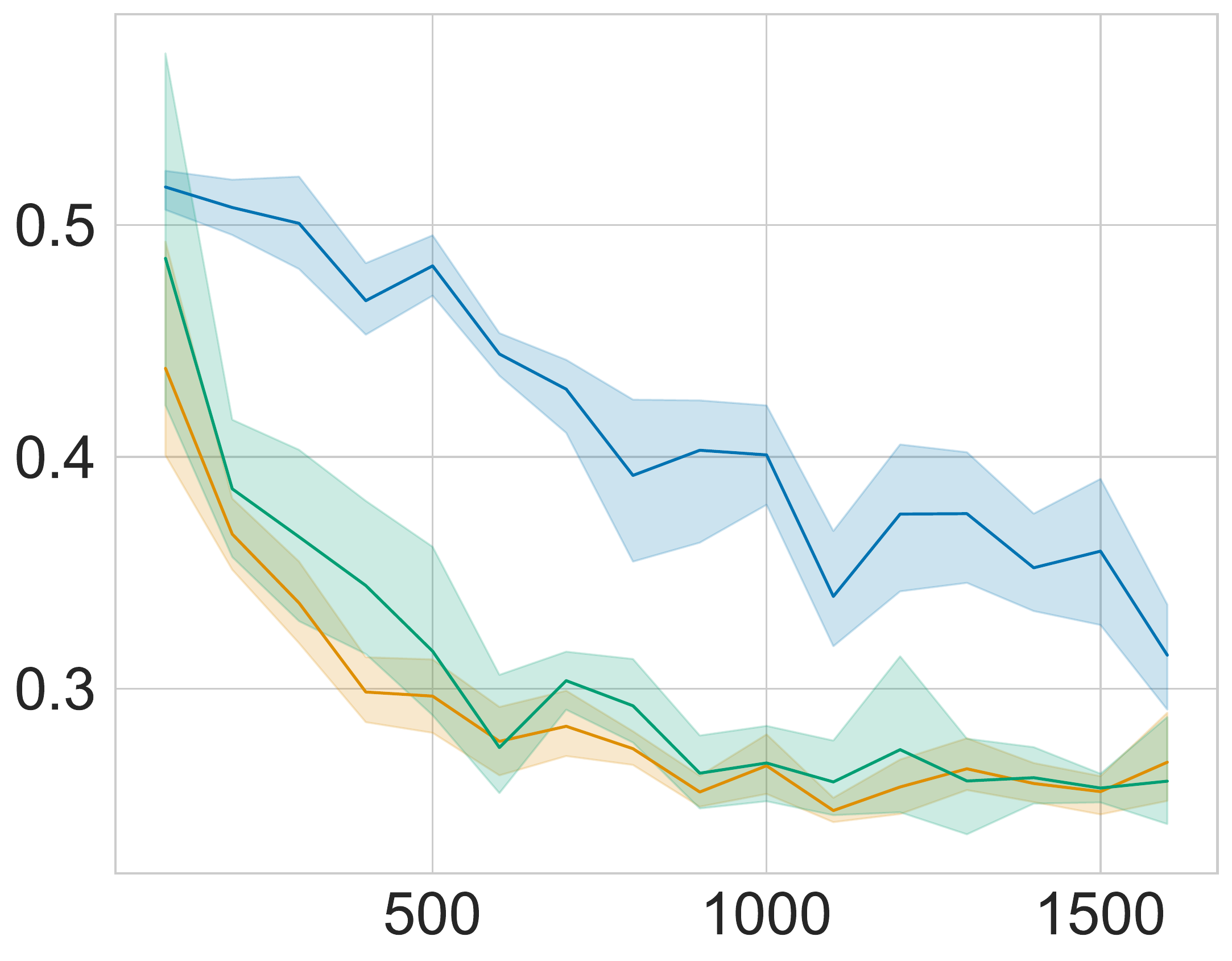}
    }
    \subfigure[Concordance]{
        \includegraphics[height=25mm]{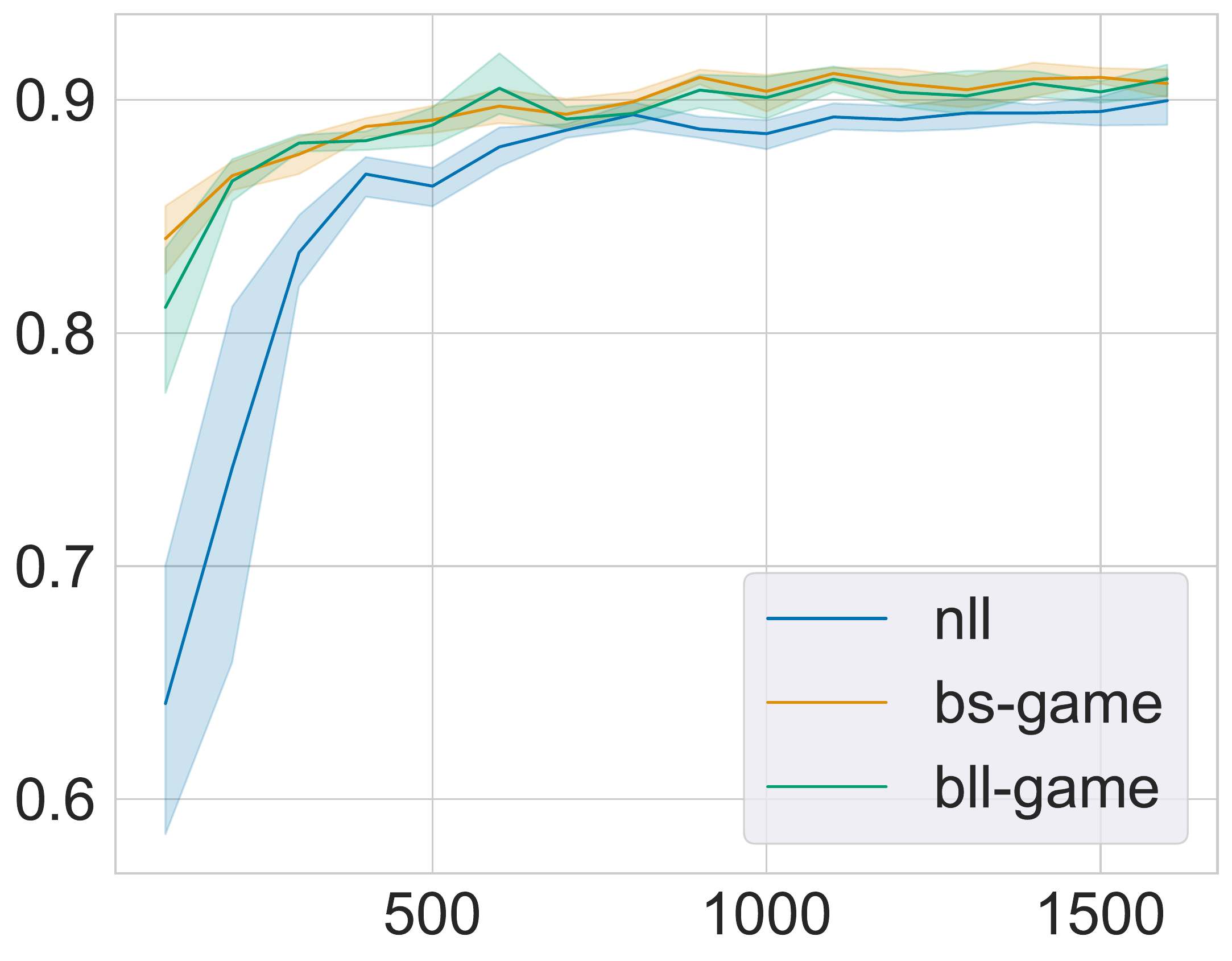}
    }
   \subfigure[Categorical \acrshort{nll}]{
        \includegraphics[height=25mm]{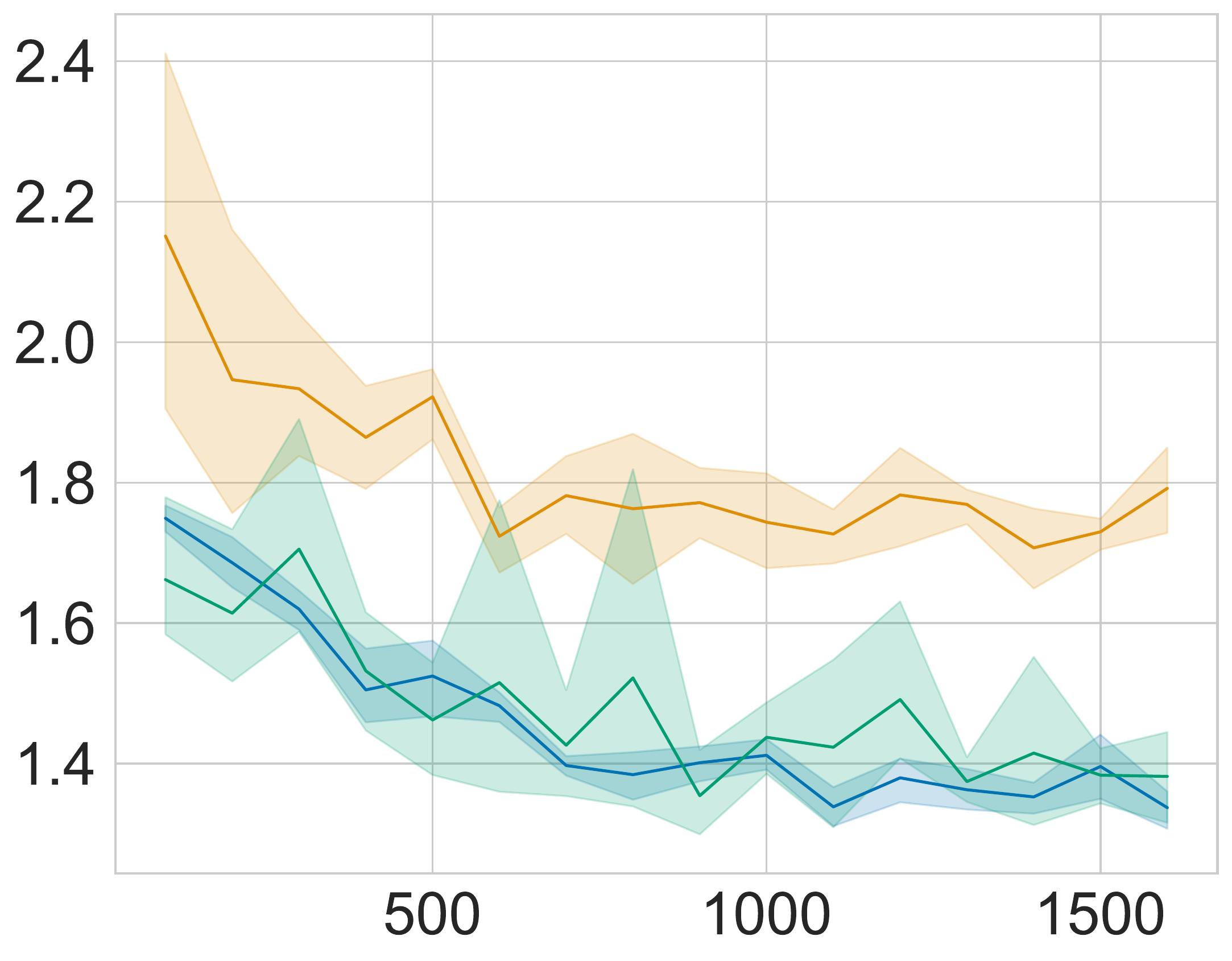}
    }
    \caption{Test set evaluation metrics (y-axis) on the Gamma simulation versus number of training points (x-axis) for three methods. 
    Each point in the plot represents the evaluation metric value of a fully trained model with that number of training points. Lower is better for all the metrics except for concordance. \label{fig:gamma}
 }
\end{figure}

\paragraph{Calibration.} We include a qualitative 
investigation of 
model calibration on the gamma simulation trained with 2000 datapoints. \cref{fig:calibrationplot} shows that the \gls{bs} game achieves near-perfect calibration while the two likelihood-based methods suffer some error. This is expected since likelihood-based methods do not directly optimize calibration while
\gls{bs} does (\cref{sec:timedependentloss}).
 \begin{figure}
    \centering
    \includegraphics[width=70mm]{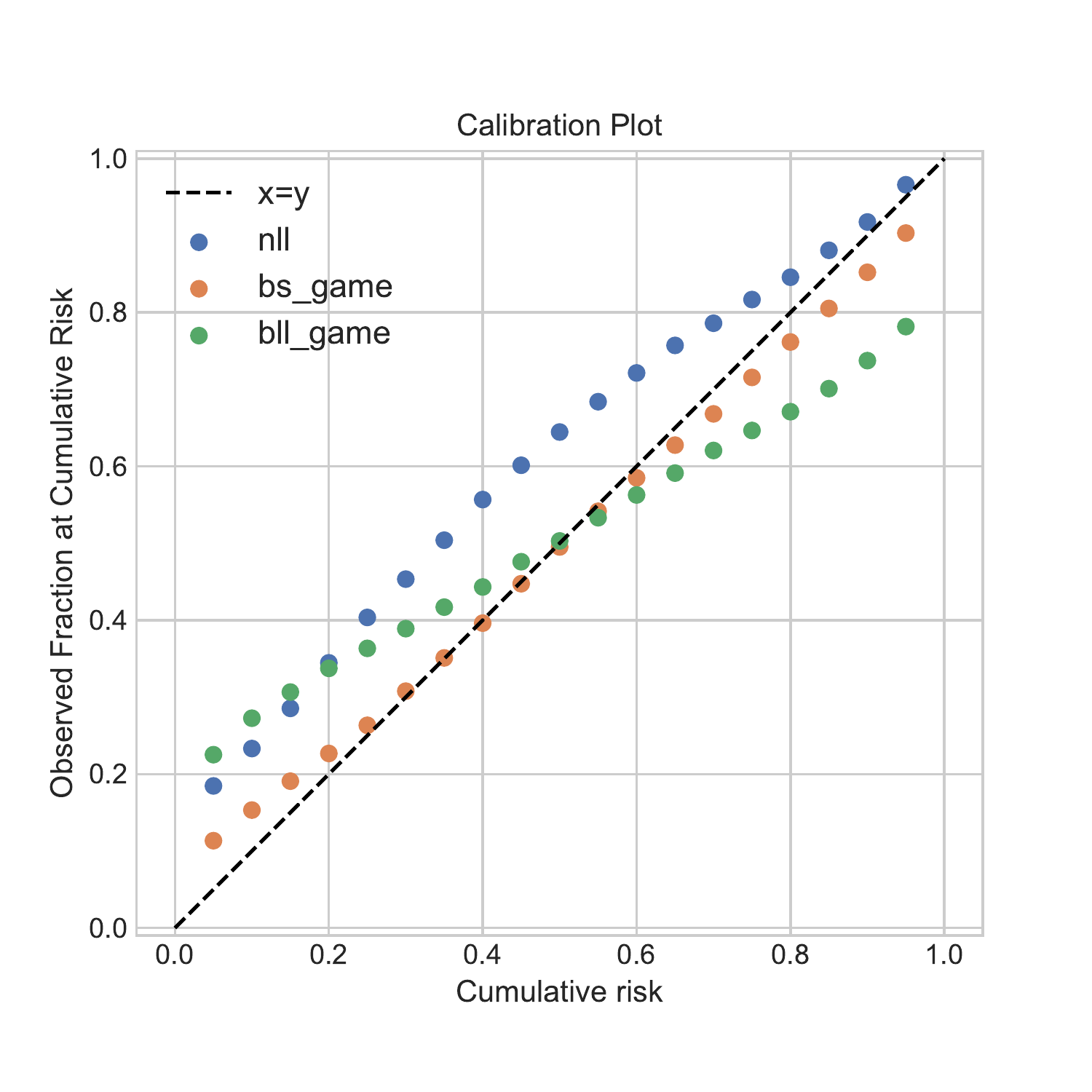}
    \caption{Calibration curves \citep{avati87countdown} comparing game-training and \gls{nll}-training.  \label{fig:calibrationplot}.}
\end{figure}

\subsection{Semi-simulated studies}

\paragraph{Data.} Survival-\acrshort{mnist}
\citep{gensheimer2019scalable,polsterl2019survivalmnist}
draws times  conditionally on \acrshort{mnist} label $Y$. This means digits define risk groups and $T \indep X \mid Y$. Times within a digit are i.i.d. The model only sees the image pixels $X$ as covariates so it must learn to classify digits (risk groups) to model times. 
We follow \cite{goldstein2020x} and use Gamma times. $T \sim \text{Gamma}(\text{mean}=10*(Y+1))$. We set the variance constant to $0.05$. Lower digit labels $Y$ yield earlier event times. $C$ is drawn similarly but with mean $9.9*(Y+1)$. 
Each random seed draws a new dataset.

\paragraph{Results.} This experiment demonstrates
that better \gls{nll} does not correspond to better performance on \gls{bs}, \gls{bll}, and concordance. Similarly to the previous experiment,
\cref{fig:mnist} shows that game methods attain better uncensored test-set \gls{bs} and \gls{bll} on survival-\acrshort{mnist} than likelihood-based training does. The games likewise attain higher concordance. \Gls{nll} training performs better at the metric it directly optimizes. This experiment also establishes that it is possible to optimize through deep convolutional models with batch norm and pooling using the game training methods.
\begin{figure}
    \centering
    \subfigure[Uncensored \acrshort{bs}]{
        \includegraphics[height=24mm]{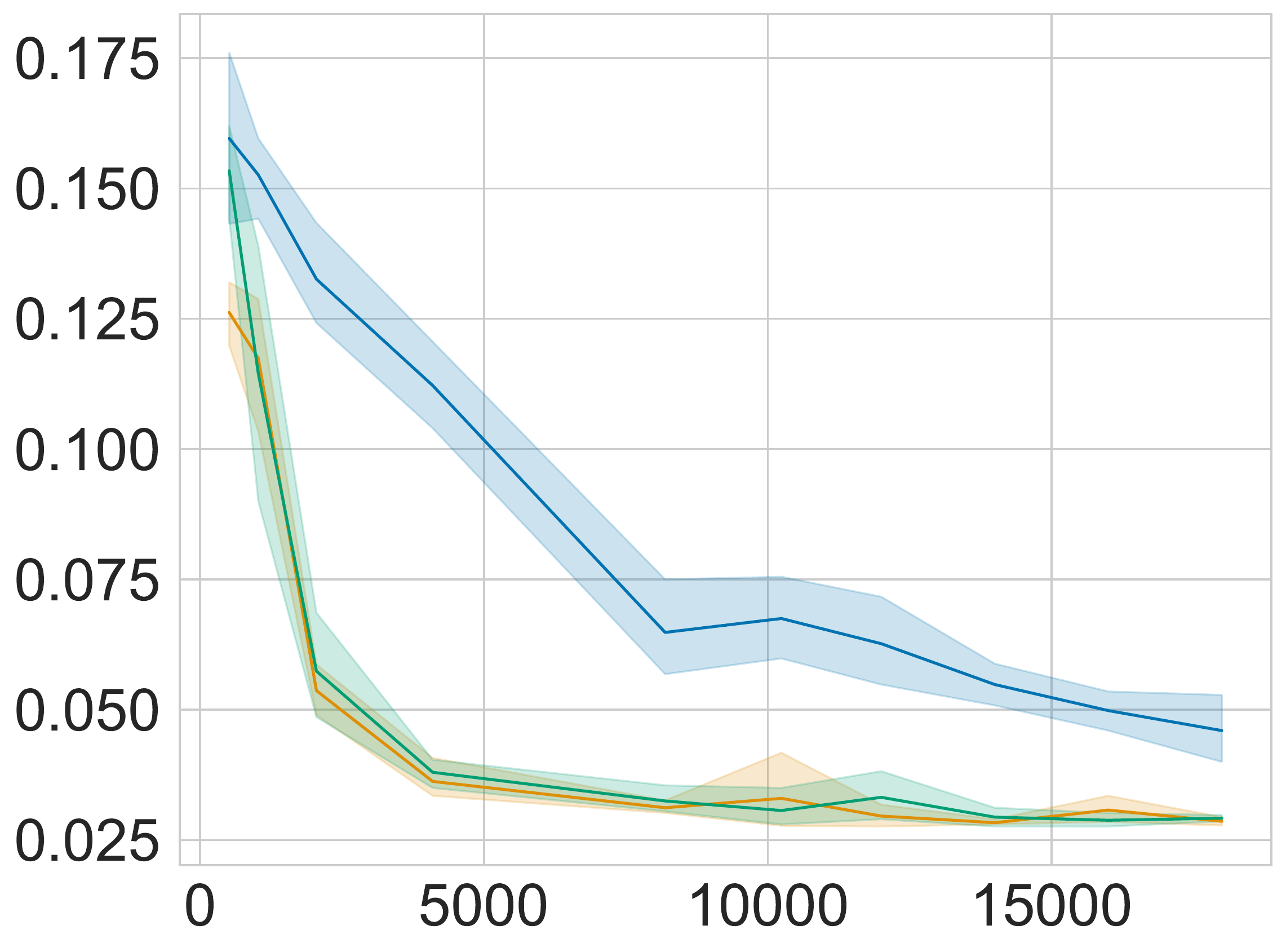}
    }
    \subfigure[Uncensored Neg \acrshort{bll}]{
        \includegraphics[height=24mm]{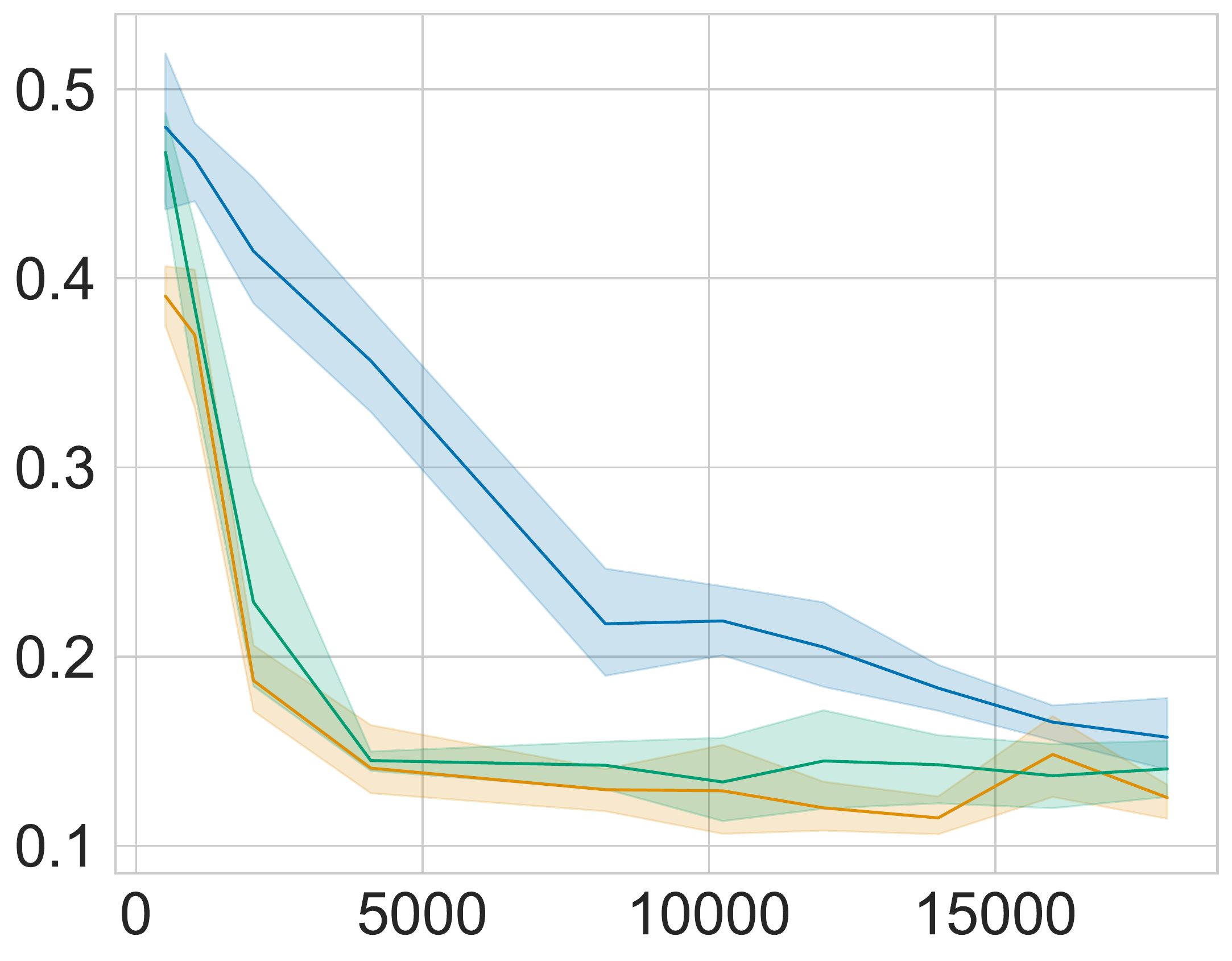}
    }
        \subfigure[Concordance]{
        \includegraphics[height=24mm]{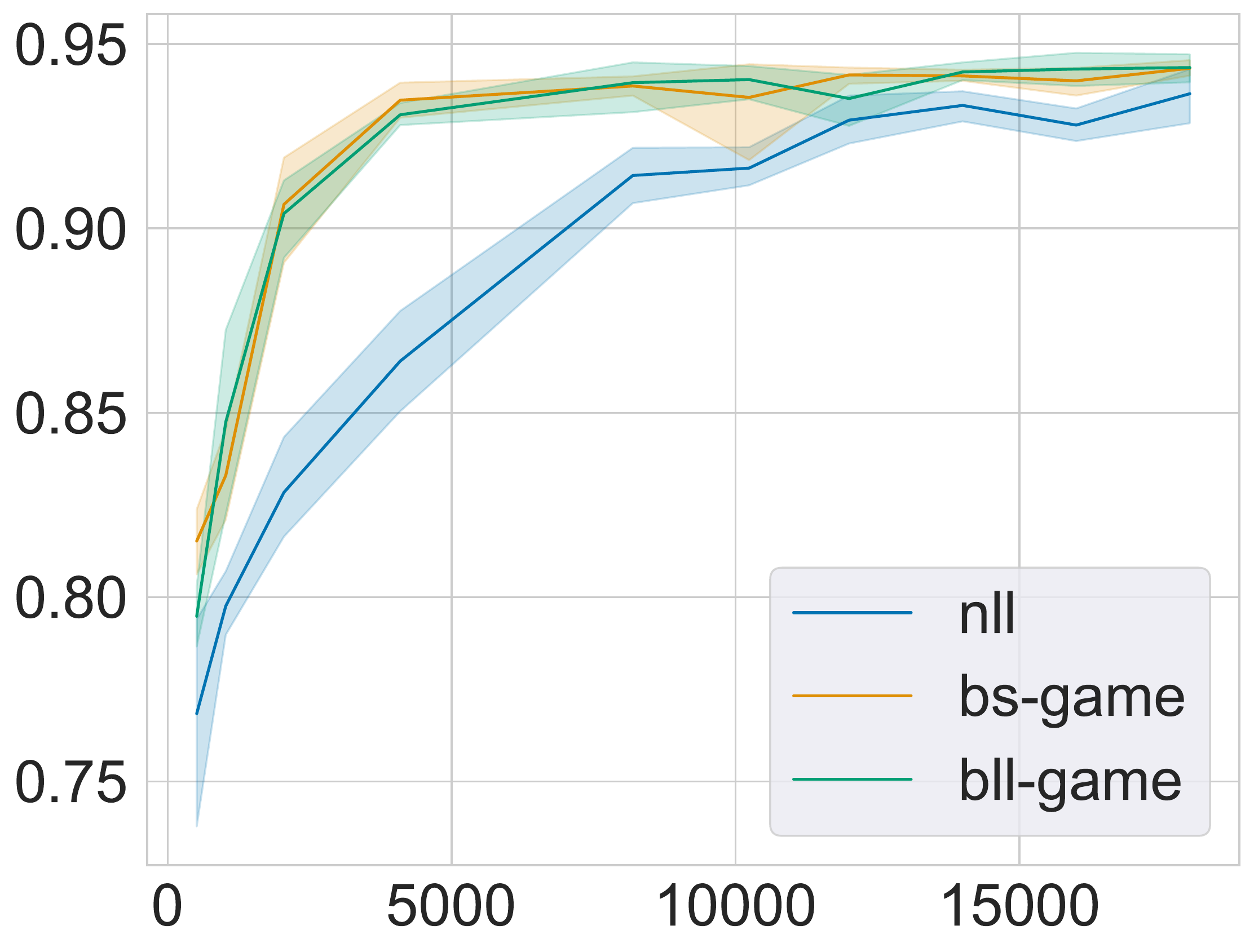}
    }
       \subfigure[Categorical \acrshort{nll}]{
        \includegraphics[height=24mm]{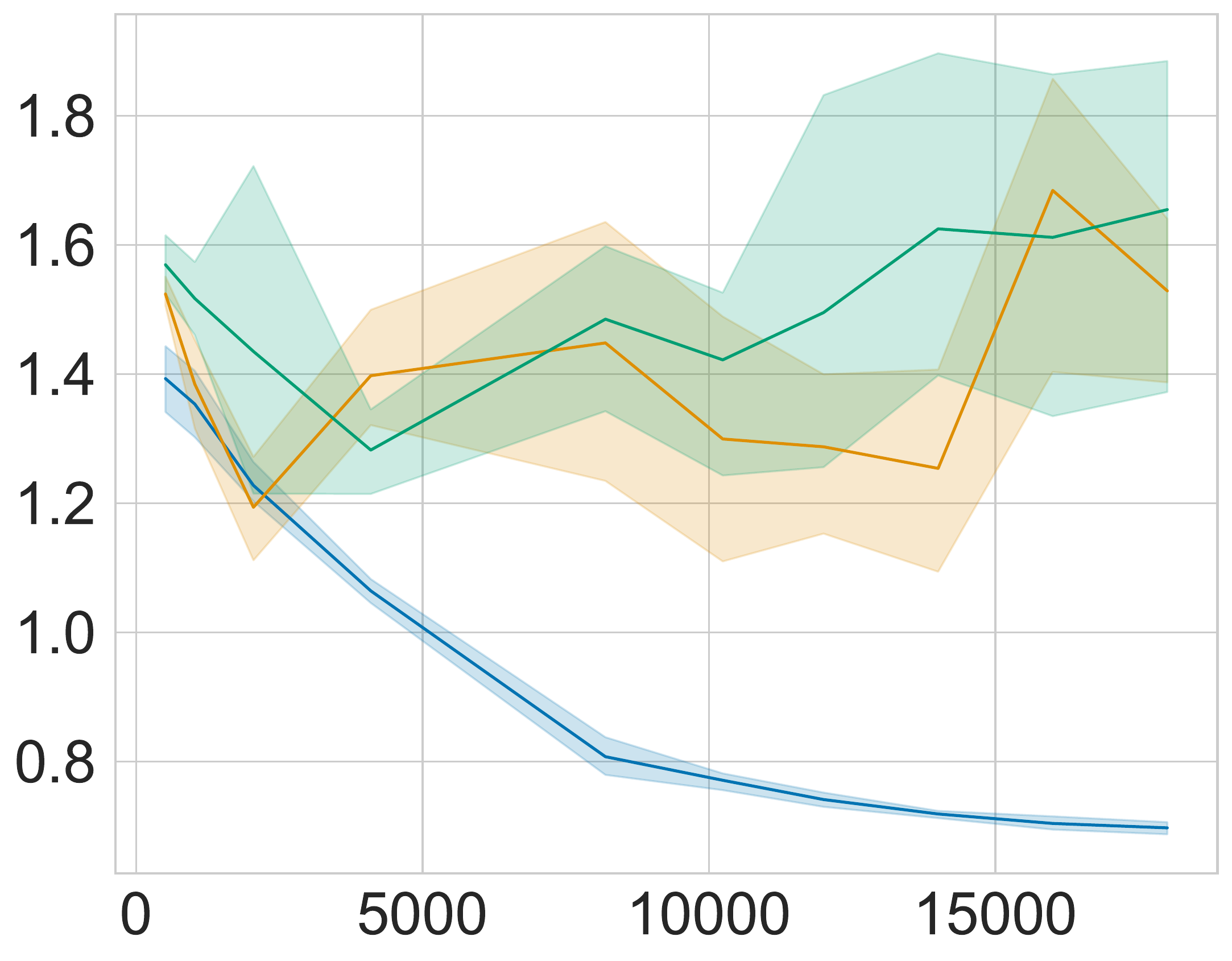}
    }
    \caption{Test set evaluation metrics (y-axis) on survival-\acrshort{mnist} versus number of training points (x-axis) for three methods. Each point in the plot represents the evaluation metric value of a fully trained model with that number of training points. Lower is better for all the metrics except for concordance. 
    \label{fig:mnist}}
\end{figure}

\subsection{Real Datasets}

\paragraph{Data.} We use
several datasets 
used in recent papers \citep{chen2020deep,kvamme2019time} and available in the python packages DeepSurv \citep{katzman2018deepsurv} and PyCox \citep{kvamme2019time}, and the R Survival \citep{therneau2021survival}.
The datasets are:
\begin{itemize}
    \item \gls{metabric} \citep{curtis2012genomic}
        \item  \gls{rott} \citep{foekens2000urokinase}
    and \gls{gbsg} \citep{schumacher1994randomized} combined into one dataset (\acrshort{rott-gbsg})
    \item \gls{support} \citep{knaus1995support} which includes severely ill hospital patients
\end{itemize}
For more description see \cref{appsec:experiments}.
For real data, there is no known ground truth for the censoring distribution,
which means evaluation requires assumptions. Following the experiments in \cite{kvamme2019time}, we assume that censoring is marginal estimate with \gls{km} to evaluate models.\footnote{This is also the route taken in the R packages Survival~\citep{therneau2021survival},
PEC~\citep{gerds2012pec}, and riskRegression~\citep{gerds2020package}.}

\paragraph{Results.}
On \gls{metabric}, games
attain lower (better) \gls{km}-weighted \gls{bs} and \gls{bll}
than \gls{nll}-training when the number of datapoints is small, and have better concordance and \gls{nll} though they do not directly optimize them. 
As data size increases,
all methods converge to similar performance. On \acrshort{rott-gbsg}, the trend is similar: games optimize the \gls{bs} and \gls{bll}
more rapidly as a function of training set size than \gls{nll}-training does.
Again, all methods converge to similar performance in all metrics when the number of datapoints is large enough. All methods perform similarly on \gls{support}.

\paragraph{Caveats.} First, though popular, these survival datasets are low-dimensional, so any of the objectives can perform well on the metrics with just several hundred points. We see that this is distinct from \acrshort{mnist},  where thousands of points were required to improve performance. Second, 
though possible, it may not be true that censoring is marginal on these datasets, 
which would mean that the \gls{bs} and \gls{bll} results only have their interpretation conditional on a particular set of covariates. Our method is also correct when the censoring is marginal though.
Lastly, no method is stable for all metrics, for all training sizes, on all seeds for all datasets.

\begin{figure}[t]
    \centering
    \subfigure[\gls{km}-weighted \gls{bs}]{
        \includegraphics[height=25mm]{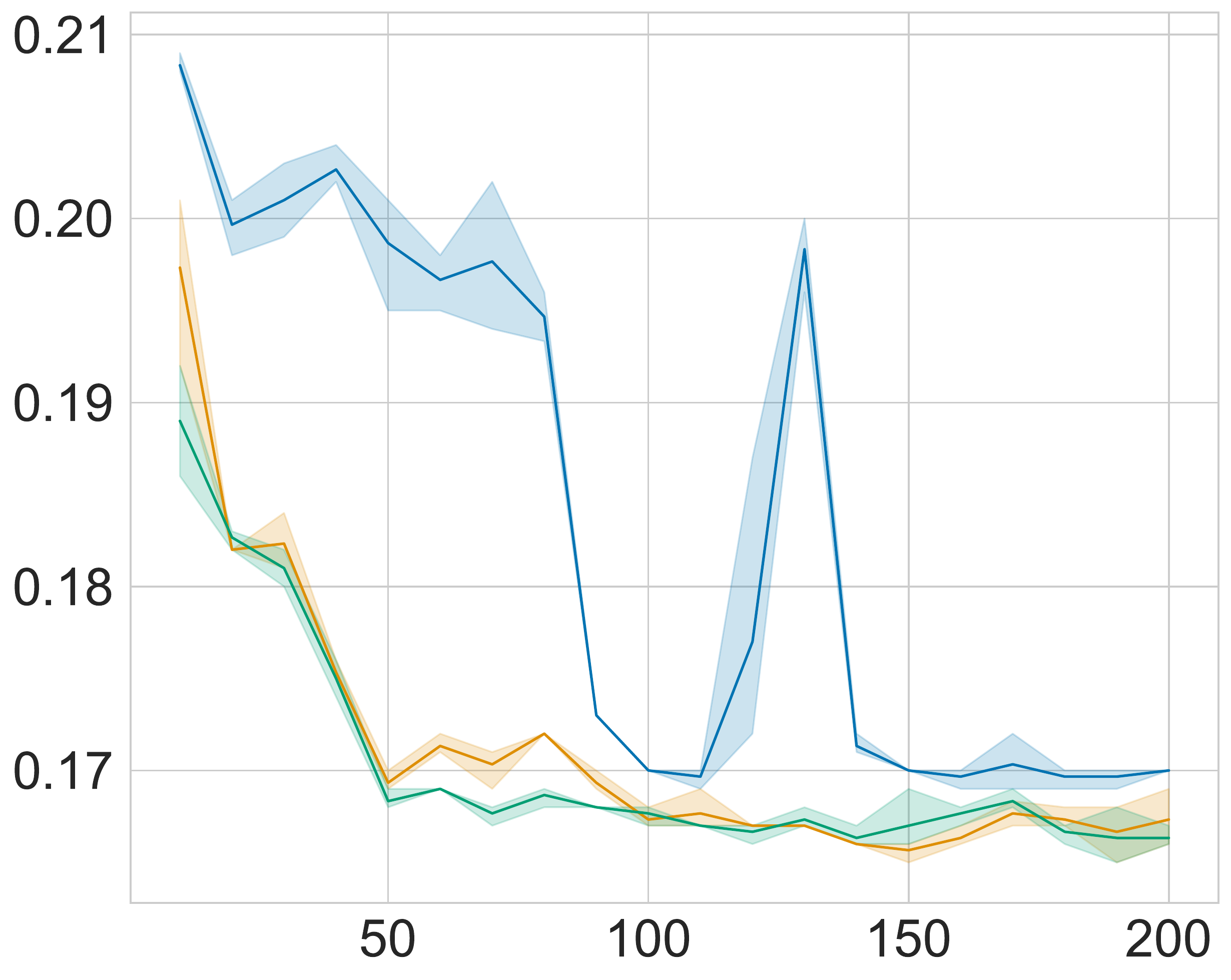}
    }
    \subfigure[\gls{km}-weighted Neg \gls{bll}]{
        \includegraphics[height=25mm]{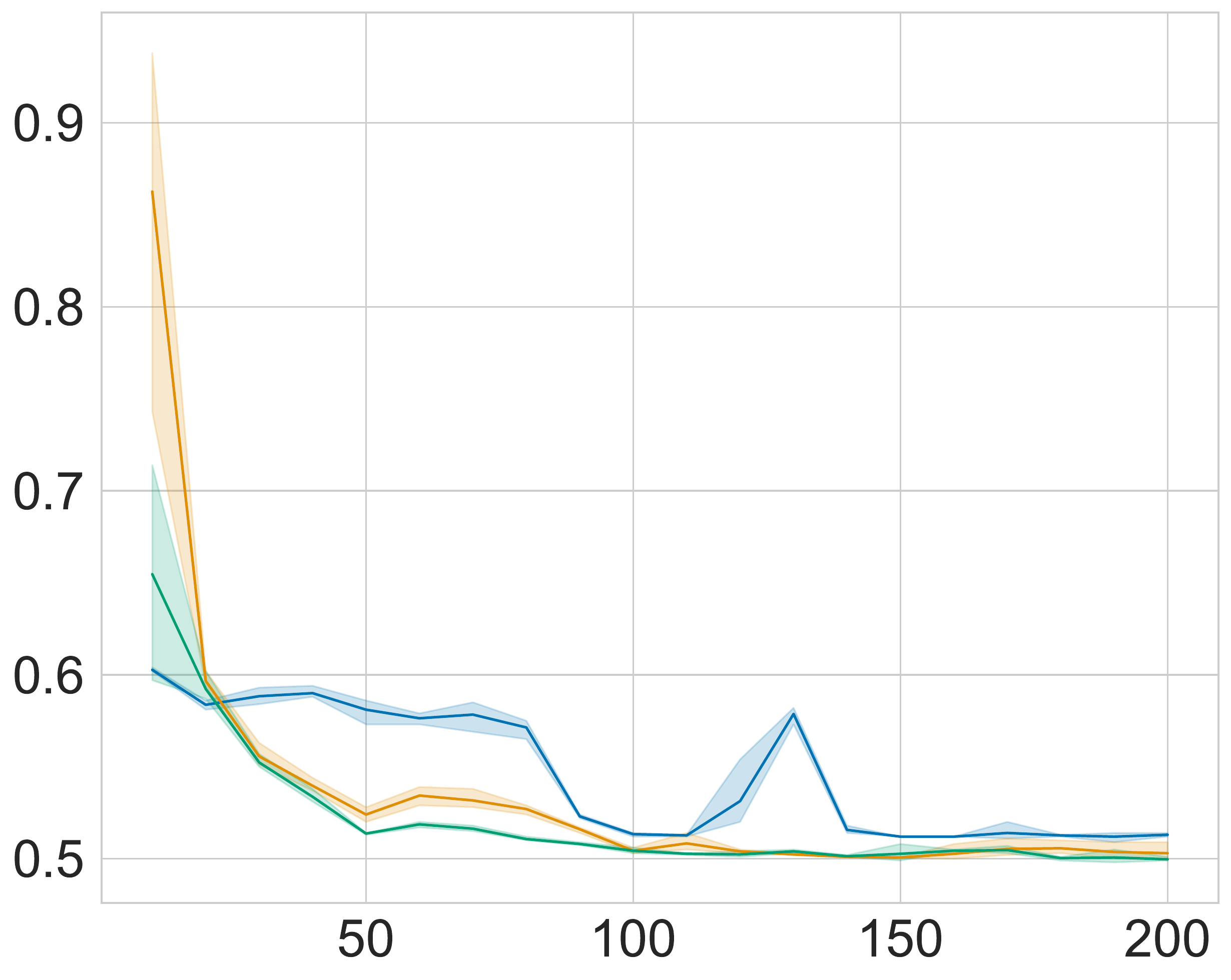}
    }
        \subfigure[concordance]{
        \includegraphics[height=25mm]{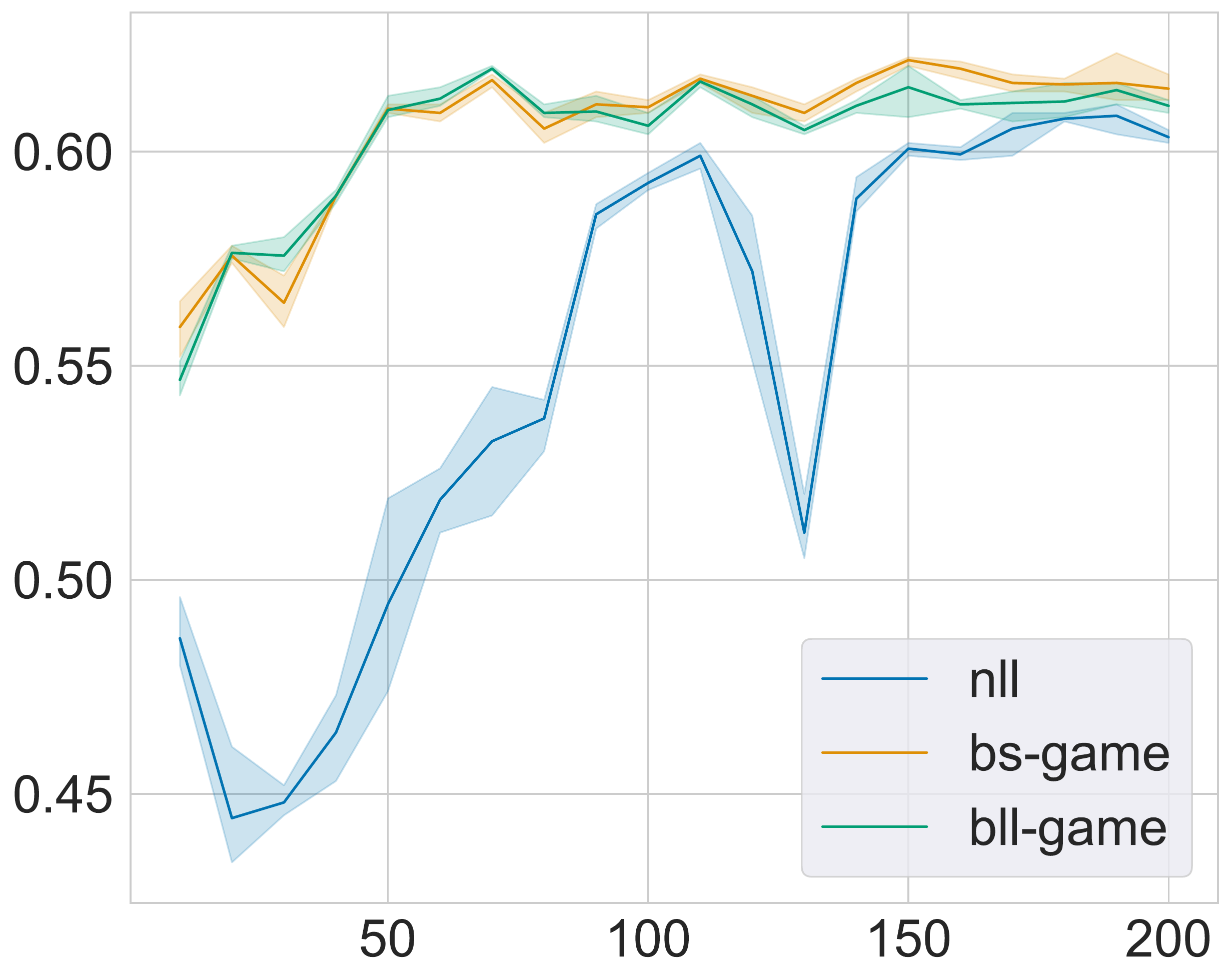}
    }
       \subfigure[categorical \acrshort{nll}]{
        \includegraphics[height=25mm]{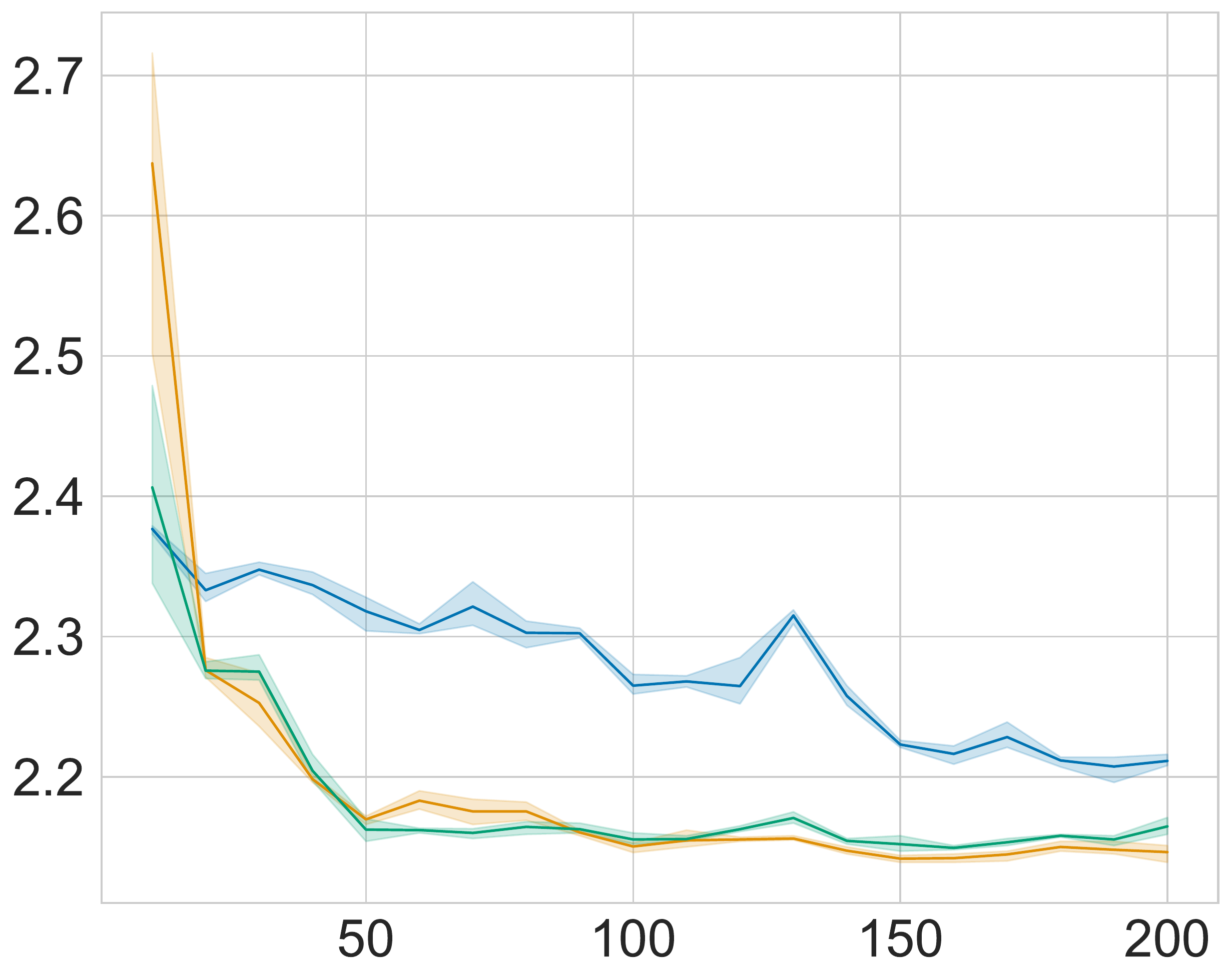}
    }
    \caption{
    Test set evaluation metrics (y-axis) on \acrshort{metabric} versus number of training points (x-axis) for three methods. Each point in the plot represents the evaluation metric value of a fully trained model with that number of training points. Lower is better for all the metrics except for concordance. 
    \label{fig:metabric}}
\end{figure}

\begin{figure}[h!]
    \centering
    \subfigure[\gls{km}-weighted \gls{bs}]{
        \includegraphics[height=25mm]{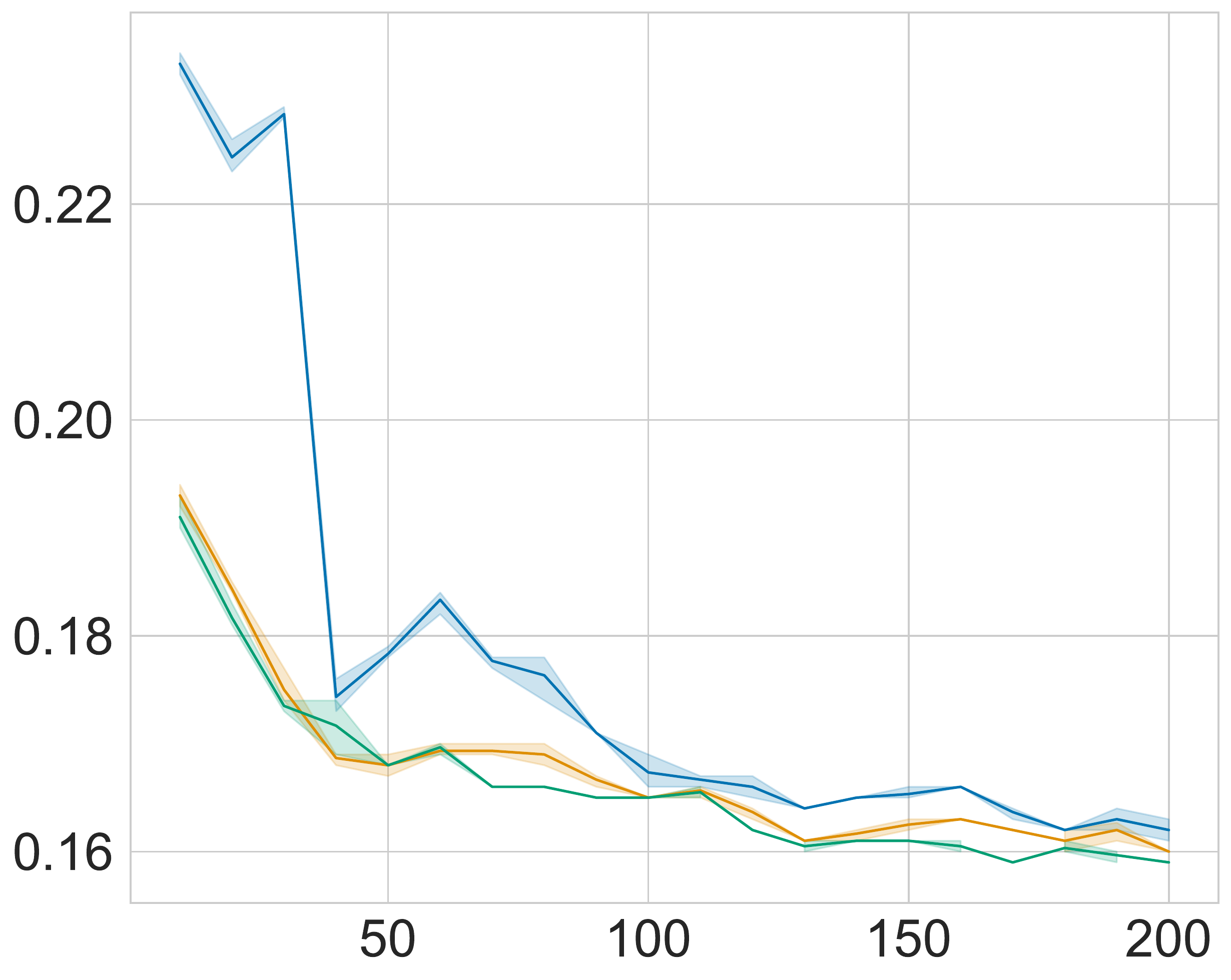}
    }
    \subfigure[\gls{km}-weighted Neg \gls{bll}]{
        \includegraphics[height=25mm]{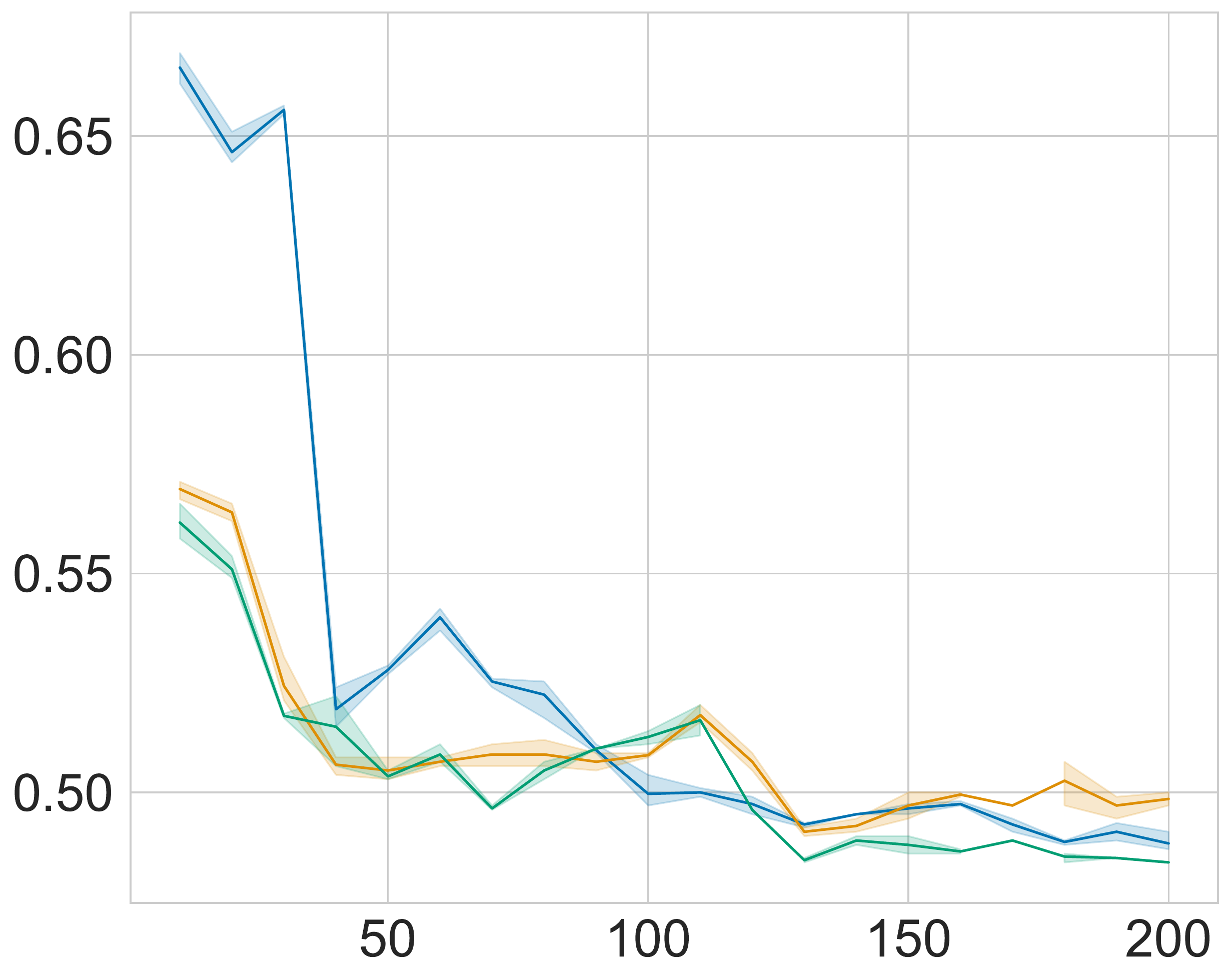}
    }
        \subfigure[concordance]{
        \includegraphics[height=25mm]{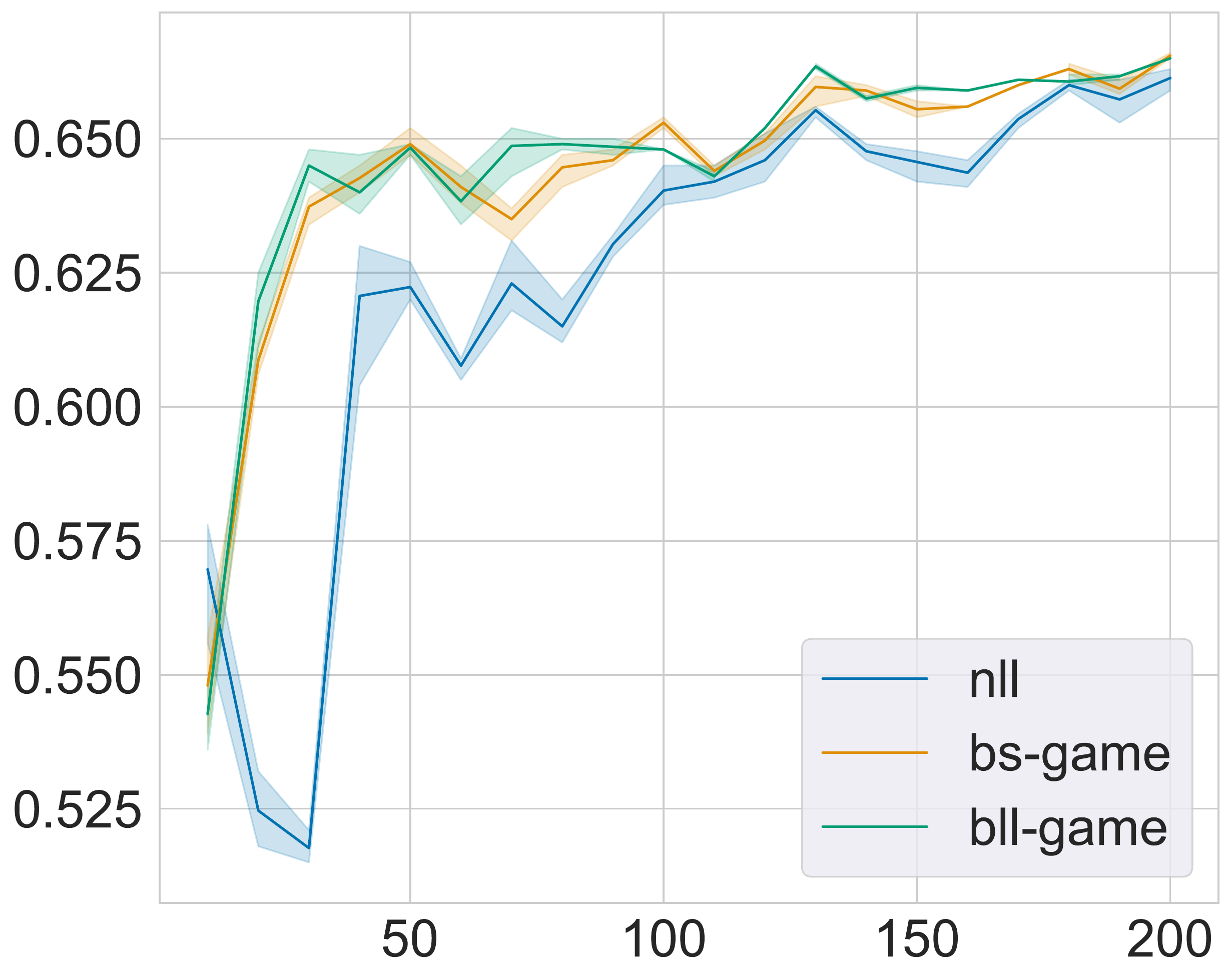}
    }
       \subfigure[categorical \acrshort{nll}]{
        \includegraphics[height=25mm]{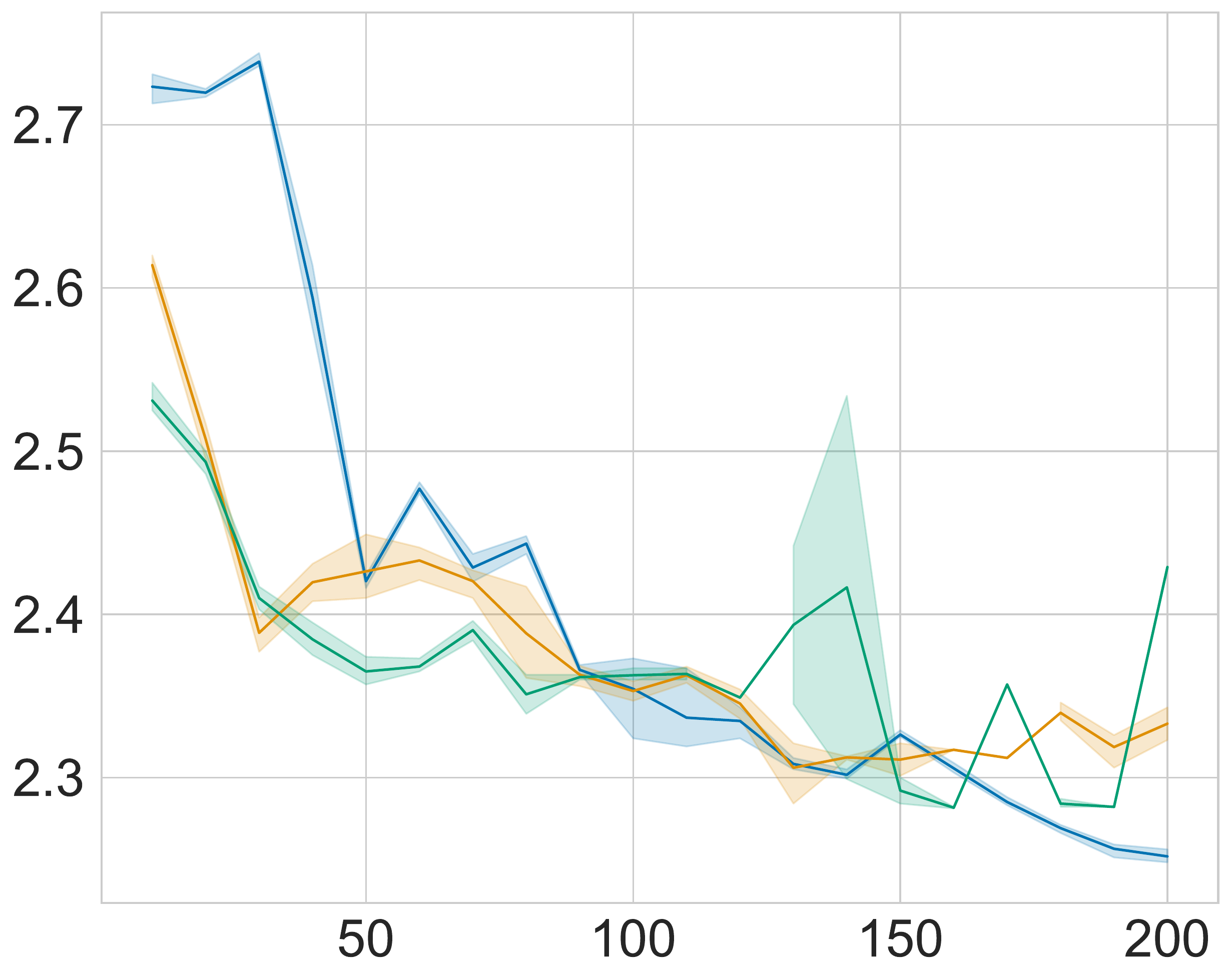}
    }
    \caption{
    Test set evaluation metrics (y-axis) on \acrshort{rott-gbsg} versus number of training points (x-axis) for three methods. Each point in the plot represents the evaluation metric value of a fully trained model with that number of training points. Lower is better for all the metrics except for concordance. 
    \label{fig:gbsg}}
\end{figure}

\begin{figure}[h!]
    \centering
    \subfigure[\gls{km}-weighted \gls{bs}]{
        \includegraphics[height=25mm]{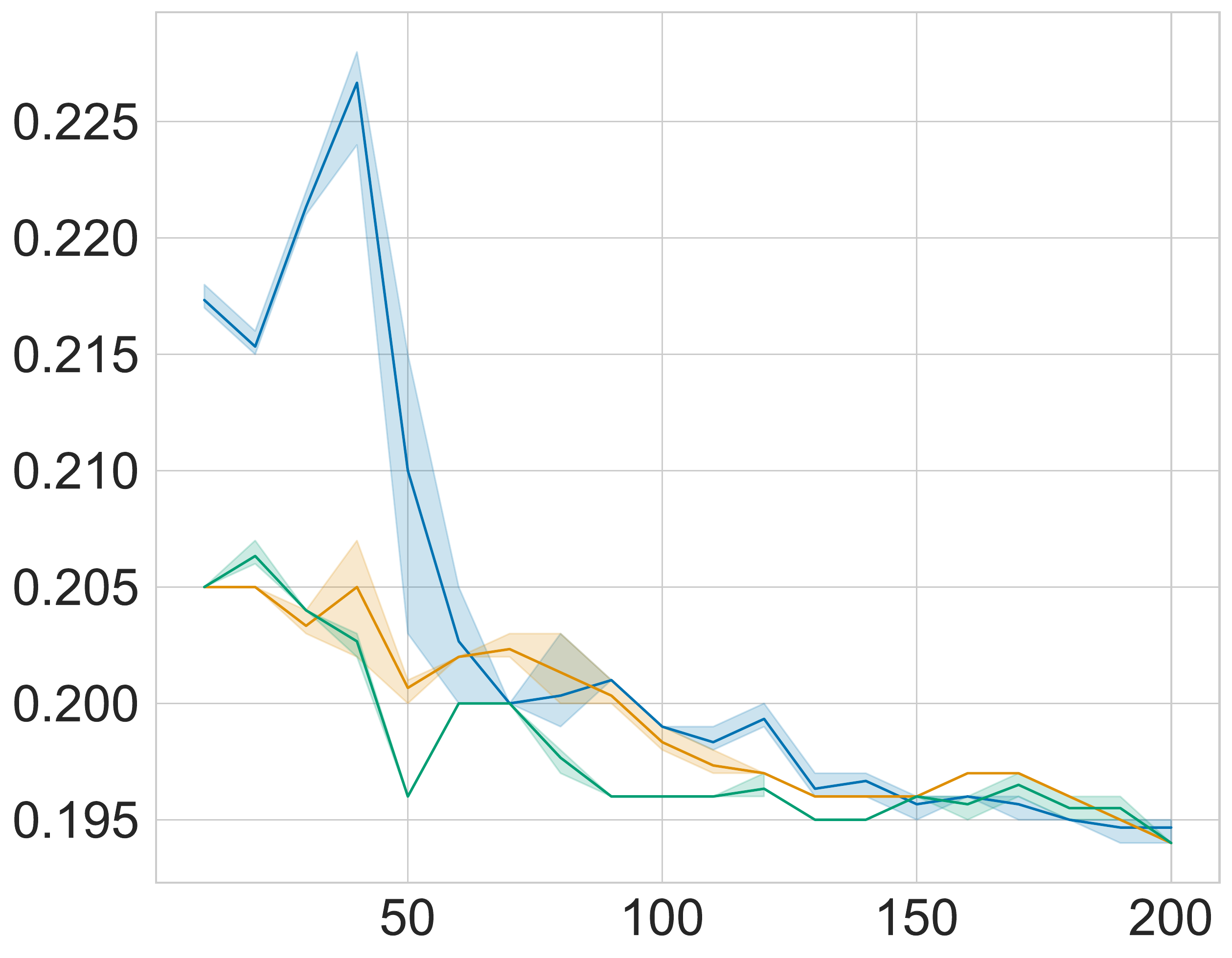}
    }
    \subfigure[\gls{km}-weighted Neg \gls{bll}]{
        \includegraphics[height=25mm]{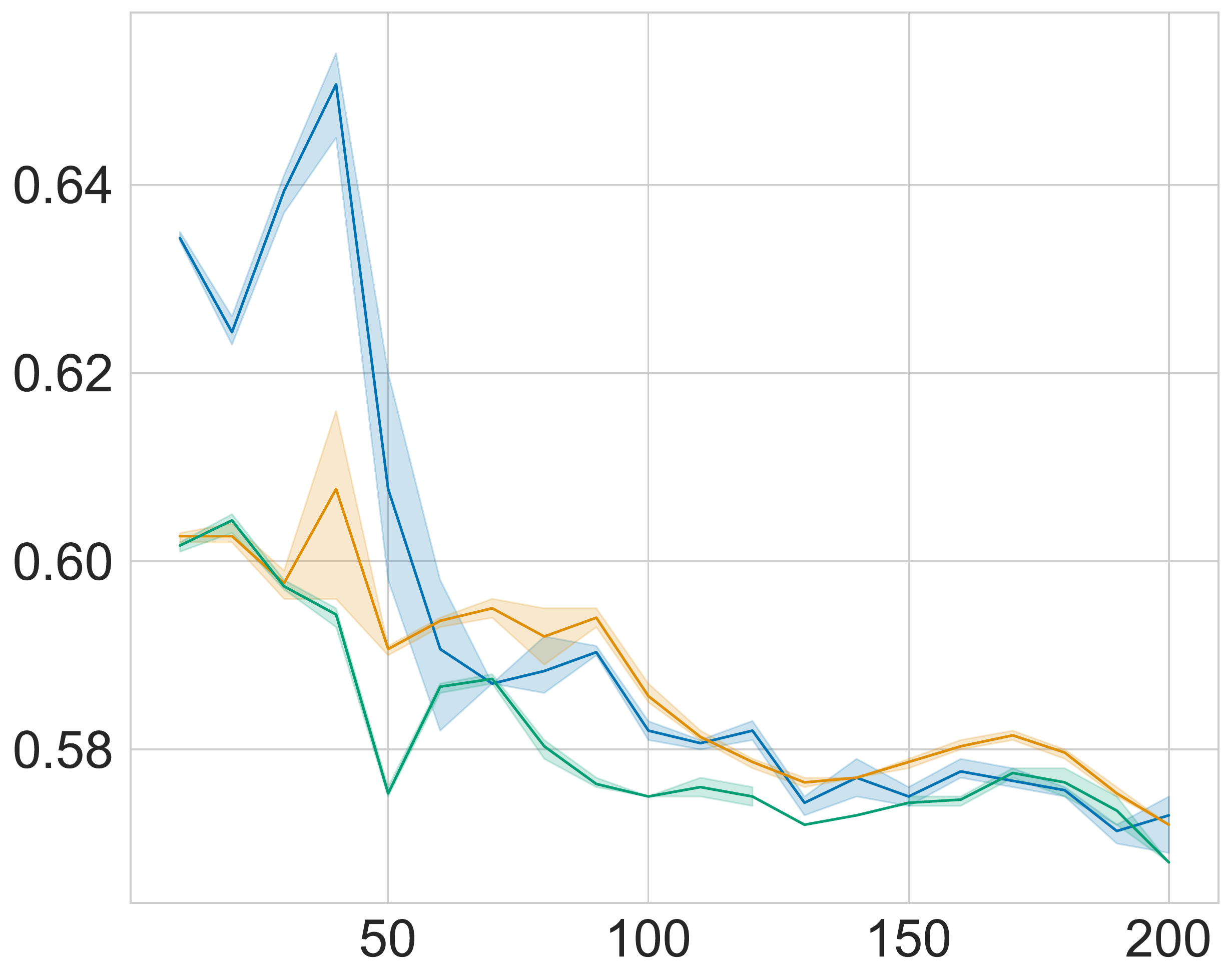}
    }
        \subfigure[concordance]{
        \includegraphics[height=25mm]{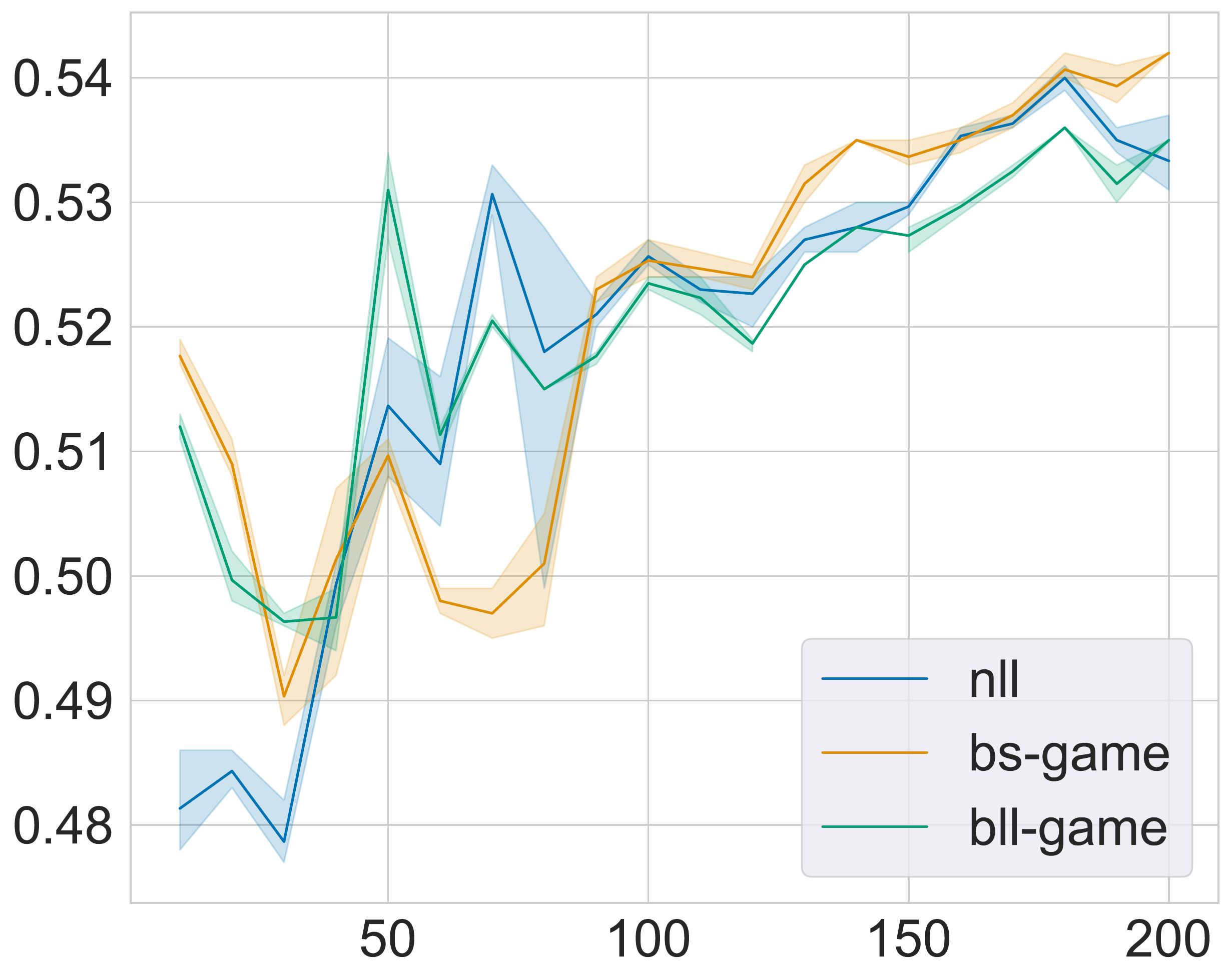}
    }
       \subfigure[categorical \acrshort{nll}]{
        \includegraphics[height=25mm]{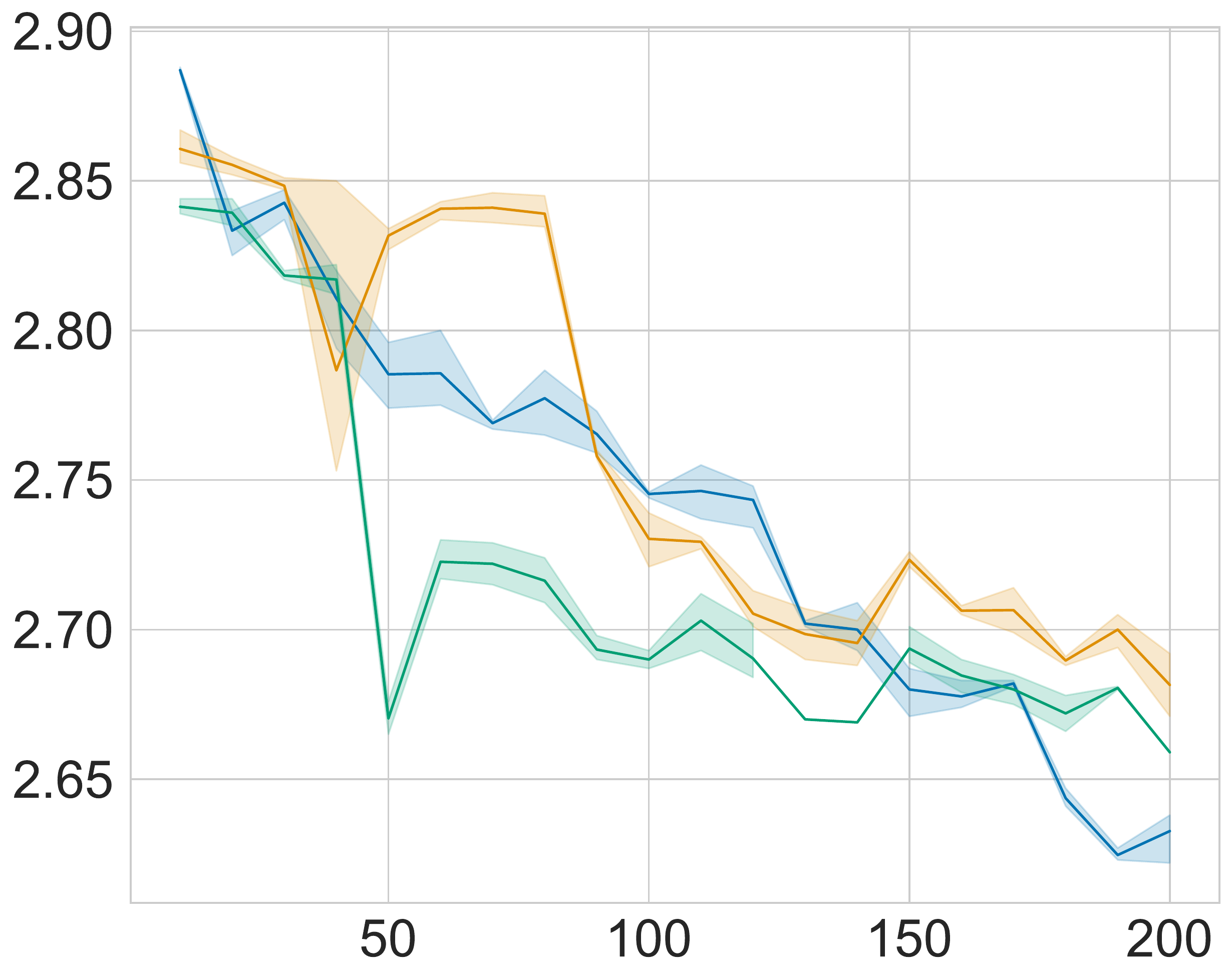}
    }
    \caption{
    Test set evaluation metrics (y-axis) on \acrshort{support} versus number of training points (x-axis) for three methods. Each point in the plot represents the evaluation metric value of a fully trained model with that number of training points. Lower is better for all the metrics except for concordance. 
    \label{fig:support}}
\end{figure}

\section{Related Work} 

\paragraph{Nuisance parameters.}
Under non-informative censoring, the censoring distribution is unrelated to the failure distribution, but estimating it can help improve learning the failure distribution; here, the censoring distribution is a \textit{nuisance parameter}.
Existing causal estimation methods propose two-stage procedures where the first stage estimates the nuisance-parameter (e.g. propensity score) and the second stage uses the learned nuisance-parameter as-is to define an estimator or loss function for the target parameter (causal effect). \citep{van2006targeted,van2011targeted,chernozhukov2018double, foster2019orthogonal}.
In this work, we instead show that estimating the target (failure model objective or failure model itself)  
can benefit from a \textit{coupled} estimation procedure where the nuisance parameter (censoring model) is also trained simultaneously.
The failure model needs the censoring distribution to compute \gls{bs} but censoring estimation needs the failure model, and despite this circular dependence, 
we characterize a case where the game training leads to the true data generating distributions.

\paragraph{Double Robust Censoring Unbiased Transformations.}
For functions $h$,
\cite{rubin2007doubly}
estimate conditional mean
$\E[h(T,X)|X]$ under censoring
using a double-robust estimator: given estimates of the conditional failure and censoring \glspl{cdf} $\hat{F}(t|X)$ and $\hat{G}(c|X)$,
the estimator of $\E[h(T,X)|X]$ is unbiased when either nuisance \acrshort{cdf} is correct.
However, here we are concerned with estimating a quantity to be used as a loss for learning $\hat{F}$.
We therefore presumably do not already have an estimate of $\hat{F}$ to be used in a doubly-robust estimator.

\paragraph{Censoring Unbiased Losses for Deep Learning.}
\cite{steingrimsson2020deep} build failure model loss functions
based on the estimators from \cite{rubin2007doubly}.
Their \gls{bs} loss extends \gls{ipcw} \gls{bs} estimation
to the doubly-robust case and to our knowledge is the first instance of \gls{ipcw}-based estimation procedures
being used in a general purpose way to define loss functions for deep survival analysis.

However, their censoring distribution is estimated 
once before training and held fixed rather than incorporated into a joint training procedure as in the games in this work. The fixed censoring estimate is implemented by
\gls{km}, which assumes a marginal censoring distribution. Making the marginal assumption when censoring is truly conditional should not yield a performant model under the \gls{bs} criteria since the training objective does not directly estimate or optimize the true \gls{bs} that would be measured under no censoring.  When marginal censoring does hold, \gls{km} estimation, which is non-parametric, may be a simpler and stable choice versus the game, depending on sample size, data variance, and conditional parameterization assumptions. But since it is in general unknown if censoring is marginal, we use conditional models which are also correct under marginal censoring.

\section{Discussion} 
In this work, we propose a new training method for survival models under censored data.
We argue that on finite data, it is important to close the gap between
training methodology and the desired evaluation criteria. We showed in the experiments that better \gls{nll} does not correspond to better performance on \gls{bs}, \gls{bll}, and concordance, all evaluations of interest in survival analysis. 

The main trend in our experimental results was that data size matters: smaller meant the game methods performed better than \gls{nll} and enough data meant that they perform similarly, which is expected since all objectives are proper. However \textit{enough} data is hard to define: it depends on dimensionality and on the data generating distribution and model class. It is a great direction to build more precise understanding on how objectives behave differently even when they have the same optimum on infinite data: 
though likelihood is known to be asymptotically efficient for survival analysis, more analysis is necessary for comparing likelihood and Brier score's trade-offs on small sample sizes.

In the experiments, we focus on categorical models. On the other hand, \cref{prop: exist} applies to continuous
distributions as well, provided that positivity can be satisfied. However, this is rare
in practice because most survival data has a final follow-up time, and even before this time there may be very few samples with late times \citep{gerds2013estimating}. For this reason, working with continuous distributions requires picking a truncation time and playing games only up to that time.

Evaluation on real data under censoring requires assumptions. It is important to further consider how to better assess test-set performance on metrics such as \gls{bs}, \gls{bll}, and concordance. 
Because concordance is not proper~\citep{blanche2019c}, we do not build objectives from it here, but it too is not invariant to censoring. 
Regarding games, we showed properties about stationary points. More analysis is necessary to describe important convergence properties of optimizing these games.

\paragraph{Social Impact.} Survival models are deployed in hospital settings and have high impact on public health. In this work, we saw benefits of a new training approach for these models, but no training method is a panacea. Practitioners of survival analysis must take great care to consider various training and validation approaches,
as well as consider possible test distribution shifts, prior to deployment.

\begin{ack}
This work was supported by:
\begin{itemize}
    \item NIH/NHLBI Award R01HL148248 
    \item NSF Award 1922658 NRT-HDR: FUTURE Foundations, Translation, and Responsibility for Data Science.
    \item NSF Award 1514422 TWC: Medium: Scaling proof-based verifiable computation
    \item NSF Award 1815633 SHF
\end{itemize}
\end{ack}

\clearpage
\bibliography{biblio}
\bibliographystyle{abbrvnat}

\clearpage
\appendix

\section{Notation, Assumptions, and Likelihoods in More Detail \label{appsec:notationassumptionsappendix}}
\subsection{Notation}
Let $T$ be a failure time with \acrshort{cdf} $F$.
$T$'s \textit{survival function} is defined by $\overline{F}=1-F$.
We denote failure models by $F_{\theta_T}$.  Let $C$
be a censoring time with \acrshort{cdf} $G$, survival function $\overline{G}$, and model $G_{\theta_C}$. 
Under right-censoring, define $U=\min(T,C)$, $\Delta=\indicator{T \leq C}$ and we observe $(X_i,U_i,\Delta_i)$. 
We use $\overline{G}(t^-)$ to denote $P(C \geq t)$.
 
\subsection{Assumptions} 
We assume i.i.d. data and random censoring: $T \indep C \g X$ \citep{kalbfleisch2002}.
Derivations in this work also require
the censoring positivity assumption \citep{gerds2013estimating}. 
Let $f=dF$ (a failure density) and $g=dG$ (a censoring density). Then we assume
\begin{align}
    \exists \epsilon 
    \quad \text{ s.t. } \quad 
   \forall x \,
   \forall t \in 
        \{t\leq t_{\text{max}} \mid  f(t|x) > 0\},
\quad
    \overline{G}(t^{-} \gtight x) \geq
    \epsilon > 0,
\end{align}
for some truncation time $t_{\text{max}}$.
Truncating at a maximum time is necessary in practice for continuous distributions because datasets may have no samples in the tails, leading to practical positivity violations \citep{gerds2013estimating}.
This truncation happens implicitly for categorical models by choosing the bins.

To observe censoring events properly,
we also require a version of  \cref{eq:positivity}
to hold with the roles of $F$ and $G$ reversed:
\begin{align}
   \exists \epsilon 
    \quad \text{ s.t. } \quad 
   \forall x \,
  \forall t \in 
    \{t\leq t_{\text{max}} \mid  g(t|x) > 0\},
\quad
    \overline{F}(t \gtight x)  \geq
    \epsilon > 0.
\end{align}
$t_{\text{max}}$ should be chosen so that these two conditions hold.

\subsection{Likelihoods}
As mentioned, we assume data are i.i.d. and censoring is random $T \indep C \g X$.
Under these assumptions, the likelihood, by definition \citep{andersen2012statistical}, is:
   \begin{align}
     \label{eq:failurecensornll}
        L(\theta_T,\theta_C) &= \prod_i
        \Big[ f_{\theta_T}(U_i)\overline{G}_{\theta_C}(U_i^{-})\Big]^{\Delta_i}\Big[g_{\theta_C}(U_i)\overline{F}_{\theta_T}(U_i)\Big]^{1-\Delta_i},
    \end{align}
When a failure is observed, $\Delta_i=\indicator{T_i \leq C_i} = 1$ so we compute the failure density or mass $f$ at the observed time $U_i=T_i$. In this case, the only thing we know about the censoring time is $C_i \geq T_i = U_i$. We therefore compute $P(C_i \geq T_i) = P(C_i \geq U_i) = 1 - G_{\theta_C}(U_i^{-}) = \overline{G}_{\theta_C}(U_i^{-})$. Likewise, when a censoring time is observed, $\Delta_i=0$ so we compute the censoring density or mass $g$ at the observed
censoring time $U_i=C_i$. In this case, the only thing we know about the failure time is that
$T_i > C_i$. We therefore compute $P(T_i > C_i) = P(T_i > U_i) = 1 - F(U_i) = \overline{F}(U_i)$.

Under the additional assumption of non-informativeness -that $F$ and $G$ don't share parameters and therefore $\theta_T,\theta_C$ are distinct- the $g/G$ terms are constant wrt $\theta_T$ and the $f/F$ terms are constant wrt $\theta_C$.
In this case, when one is modeling failures, they can use the partial failure likelihood:
  \begin{align*}
        L(\theta_T)^{\text{partial}} &= \prod_i
        \Big[ f_{\theta_T}(U_i)\Big]^{\Delta_i}
       \Big[ \overline{F}_{\theta_T}(U_i)\Big]^{1-\Delta_i}
    \end{align*}
And when one is modeling censoring they can use the partial censoring likelihood:
  \begin{align*}
        L(\theta_C)^{\text{partial}} &= \prod_i 
        \Big[
        \overline{G}_{\theta_C}(U_i^{-})
        \Big]^{\Delta_i}
       \Big[ 
        g_{\theta_C}(U_i)
       \Big]^{1-\Delta_i}
    \end{align*}

\subsection{Failure partial likelihood depends on true censoring distribution \label{appsec:censoringaffectslikelihood} }

We now show that the failure partial likelihood's scale depends on the true sampling distribution of censoring times, even if the censoring model has dropped as a constant in the objective.
The expected likelihood is:
\begin{align*} 
\E_{\substack{T \sim F_{\theta_T^*}, C \sim G_{\theta_C^*} \\ U=\min(T,C),\Delta=\indicator{T \leq C}}} \Big[f_{\theta_T}(U)^{\indicator{\Delta=1}}  \overline{F}_{\theta_T}(U)^{\indicator{\Delta=0}} \Big]
\end{align*} 

The reason is that $\Delta$ and $U$ depend on T and C (therefore on $F_{\theta_T^*}$ and $G_{\theta_C^*}$). We now constructively show that the failure model’s \gls{nll} can vary with the true censoring distribution. Let us consider a marginal survival analysis problem (no features) and random censoring. The log \gls{nll} is:
\begin{align*}
  \E_{F_{\theta_T^*},G_{\theta_C^*}}[ \Delta \log f_{\theta_T}(U)] + \E_{F_{\theta_T^*},G_{\theta_C^*}}[(1-\Delta)  \log \overline{F}_{\theta_T}(U)]   
\end{align*}

Now consider an $F_{\theta_T^*}$ whose support starts at time $1$ (e.g. uniform over 1,2,3) and $G_{\theta_C^*}$ such that there is probability $\rho$ that $C=0$ and probability $1-\rho$ that $C$ take a value above the support of $T$ (e.g. >3). Points are therefore only censored at time 0 or uncensored.

\begin{align*}
    &\E_{F_{\theta_T^*},G_{\theta_C^*}}[ \Delta \log f_{\theta_T^*}(U)] + \E_{F_{\theta_T^*},G_{\theta_C^*}}[(1-\Delta)  \log \overline{F}_{\theta_T^*}(U)]\\
    &= (1-\rho) \E_{F_{\theta_T^*}}[\log f_{\theta_T^*}(T)] + \rho \E_{G_{\theta_C^*}}[\log \overline{F}_{\theta_T^*}(C)]\\
    &= (1-\rho)  \E_{F_{\theta_T^*}}[\log f_{\theta_T^*}(T)] + \rho \E_{G_{\theta_C^*}}[\log \overline{F}_{\theta_T^*}(0)]\\
    &= (1-\rho) \E_{F_{\theta_T^*}}[\log f_{\theta_T^*}(T)] + \rho \E_{G_{\theta_C^*}}[\log 1]\\
    &= (1-\rho) \E_{F_{\theta_T^*}}[\log f_{\theta_T^*}(T)] + \rho \E_{G_{\theta_C^*}}[0]\\
    &= (1-\rho) \E_{F_{\theta_T^*}}[\log f_{\theta_T^*}(T)]
\end{align*}

This quantity depends on $\rho$. This shows that the failure model's \gls{nll} depends on the true sampling distribution of censoring times.

\section{\Gls{ipcw} Primer \label{appsec:ipcwprimer}}
\Gls{ipcw} is a technique for estimation under censoring \citep{gerds2006consistent}. 
Consider estimating the marginal mean of $T: \E[T] = \E_X \E_{T\gtight X}[T]$. $T$ is not observed for all datapoints. Instead, we observe $U=\min(T,C)$ and $\Delta=\indicator{T \leq C}$.   \Gls{ipcw} reformulates such
expectations in terms of observed data. Using this method, we can show that:
\begin{align*} 
\begin{split}
     \E_X \E_{T\gtight X}[T]
    &=
    \E_{X} \E_{T \gtight X} 
    \Bigg[ \frac{\E_{C \gtight X}\indicator{T\leq C}}
    {\E_{C' \gtight X}\indicator{T \leq C'}} T \Bigg]\\
    &=
    \E_{X} \E_{T \gtight X}  \E_{C \gtight X}
    \Bigg[\frac{\indicator{T \leq C}}
    {\E_{C^\prime \gtight X}\indicator{T \leq C^\prime}} T \Bigg]\\
    &=
    \E_{T,C,X}
    \Bigg[\frac{\indicator{T \leq C}}
    {\E_{C^\prime \gtight X}\indicator{T \leq C^\prime}} T \Bigg]\\
    &=
    \E_{T,C,X}
    \Bigg[\frac{\indicator{T \leq C}}
    {\mathbb{P}(C^\prime  \geq T \gtight X)} T \Bigg]\\
    &=
    \E_{T,C,X}
    \Bigg[\frac{\indicator{T \leq C}}
    {\overline{G}(T^{-} \gtight X)} T \Bigg]\\
     &=
    \E_{U,\Delta,X}
    \Bigg[ 
    \frac{\Delta U}{\overline{G}(U^{-} \gtight X)}
    \Bigg] 
\end{split}
\end{align*}
We have used $C^\prime$ in the denominator to emphasize that it is not a function of $C$ in the integral over the numerator indicator once that expectation is moved out.
We have used random censoring to go from $\E_{T\gtight X} \E_{C \gtight X}$ to the joint $\E_{T,C \gtight X}$.
The last equality changes from the complete data distribution to the observed distribution and
holds because $\Delta=1 \implies U=T$. This means we can estimate the expectation, provided that we know $G$ and that
random censoring and positivity (\cref{eq:positivity}) hold. In practice, we must learn the censoring distribution, a challenging task as it is also censored.

\cite{graf1999assessment} develop the \gls{ipcw} \gls{bs}. \cite{gerds2006consistent} extend it to conditional censoring and \cite{kvamme2019brier} specialize to administrative censoring.
\cite{gerds2013estimating,wolbers2014concordance} develop the \gls{ipcw} concordance. 
\gls{ipcw} estimators for several forms of \gls{auc} have been
studied in 
\cite{hung2010estimation,hung2010optimal,blanche2013review,blanche2019c,uno2007evaluating}.
\cite{yadlowsky2019calibration} derive an \gls{ipcw} estimator for binary survival calibration. 

\section{Deriving \gls{ipcw} Brier Scores \label{appsec:deriveipcwbrier}} 

We derive the \gls{ipcw} \gls{bs}
introduced by \cite{graf1999assessment,gerds2006consistent}. In the below let F-BS be the F model's BS and let  F-BS-CW be its censor-weighted version.
The censor-weighted failure \gls{bs}:
\begin{align*}
    \text{F-BS-CW}(t)= \E_{T,C} \Big[\frac{(1 - F_\theta(t))^2 \indicator{T \leq C} \indicator{U \leq t}}{P_\theta(C^\prime \geq U)}
        + \frac{F_\theta(t)^2 \indicator{U > t}}{P_\theta(C^\prime > t)}\Big]\\
\end{align*}
where $U=\min(T,C)$ and $F_\theta=P_\theta(T \leq \cdot)$, It's relationship to the regular \gls{bs} is:
\begin{align*}
   \text{F-BS}(t) &= \E_{T}\Big[
   \Big( F_\theta(t) - 
    \indicator{T \leq t}\Big)^2
   \Big]\\
   &=
   \E_{T}\Big[
   (1-F_\theta(t))^2\indicator{T \leq t}
   +
   F_\theta(t)^2
   \indicator{T > t}
   \Big]\\
    &=
      \E_{T}\Big[
      \frac{\E_{C}\indicator{T \leq C}}{\E_{C'}\indicator{T \leq C'}}
      (1-F_\theta(t))^2
   \indicator{T \leq t}
   +
   \frac{\E_{C}\indicator{C>t}}{\E_{C'}\indicator{C'>t}}
   F_\theta(t)^2
   \indicator{T > t}
   \Big]\\
       &=
      \E_{T,C}\Big[
      \frac{ (1-F_\theta(t))^2\indicator{T \leq C}\indicator{T \leq t}}{\E_{C'}\indicator{T \leq C'}}
   +
   \frac{F_\theta(t)^2
   \indicator{T > t}\indicator{C>t}}{\E_{C'}\indicator{C'>t}}
   \Big]\\
         &=
      \E_{T,C}\Big[
      \frac{(1-F_\theta(t))^2\indicator{T \leq C}\indicator{T \leq t}}{P_\theta(C^\prime \geq T)}
   +
   \frac{F_\theta(t)^2
   \indicator{T > t}\indicator{C>t}}{P_\theta(C^\prime > t)}
   \Big]\\
           &=
      \E_{T,C}\Big[
      \frac{(1-F_\theta(t))^2\indicator{T \leq C}\indicator{U \leq t}}{P_\theta(C^\prime \geq U)}
   +
   \frac{F_\theta(t)^2
   \indicator{U > t}}{P_\theta(C^\prime > t)}
   \Big]\\
   &= \text{F-BS-CW}(t)
\end{align*} 
The expectation comes out due to $T \indep C$. The last line follows from $T \leq C \implies U=T$ (in the left term) and 
$\indicator{T>t}\indicator{C>t} = \indicator{U>t}$ (in the right term).
Define likewise the failure-weighted censor \gls{bs}
\begin{align*}
    \text{G-BS-CW}(t)= \E_{T,C} \Big[\frac{(1 - G_\theta(t))^2 \indicator{C < T} \indicator{U \leq t}}{P_\theta(T^\prime > U)}
        + \frac{G_\theta(t)^2 \indicator{U > t}}{P_\theta(T^\prime > t)}\Big]\\
\end{align*}
where $G_\theta=P_\theta(C \leq \cdot)$. 
The relationship to the censoring distribution's \gls{bs} is:
\begin{align*}
   \text{G-BS}(t) &= \E_{C}\Big[
   \Big( G_\theta(t) - 
    \indicator{C \leq t}\Big)^2
   \Big]\\
   &=
   \E_{C}\Big[
   (1-G_\theta(t))^2\indicator{C \leq t}
   +
   G_\theta(t)^2
   \indicator{C > t}
   \Big]\\
    &=
      \E_{C}\Big[
      \frac{\E_{T}\indicator{C < T}}{\E_{T'}\indicator{C < T'}}
      (1-G_\theta(t))^2
   \indicator{C \leq t}
   +
   \frac{\E_{T}\indicator{T>t}}{\E_{T'}\indicator{T'>t}}
   G_\theta(t)^2
   \indicator{C > t}
   \Big]\\
      &=
      \E_{T,C}\Big[
      \frac{ (1-G_\theta(t))^2\indicator{C < T} \indicator{C \leq t}}{\E_{T'}\indicator{C < T'}}
   +
   \frac{G_\theta(t)^2\indicator{T>t} \indicator{C > t}}{\E_{T'}\indicator{T'>t}}
   \Big]\\
        &=
      \E_{T,C}\Big[
      \frac{ (1-G_\theta(t))^2\indicator{C < T} \indicator{C \leq t}}{P_\theta(T^\prime > C)}
   +
   \frac{G_\theta(t)^2\indicator{T>t} \indicator{C > t}}{P_\theta(T^\prime > t)}
   \Big]\\
        &=
      \E_{T,C}\Big[
      \frac{ (1-G_\theta(t))^2\indicator{C < T} \indicator{U \leq t}}{P_\theta(T^\prime > U)}
   +
   \frac{G_\theta(t)^2  \indicator{U>t}}{P_\theta(T^\prime > t)}
   \Big]\\
   &= \text{G-BS-CW}(t)
\end{align*}
The expectation comes out due to $T \indep C$. The last line follows from $C < T \implies U=C$ (in the left term)
and  $\indicator{T>t}\indicator{C>t} = \indicator{U>t}$ (in the right term). 
\section{Negative Bernoulli Log Likelihood \label{appsec:bll}} 
Negative \gls{bll} is similar to \gls{bs},
but replaces the squared error with negated log loss:
\begin{align*}
 \text{NBLL}(t;\theta)
   &= \E_{T,C,X}
   	 \Big
	 	[-\log({F_{\theta_T}}(t \mid X)) \indicator{T \leq t}
       - \log(\overline{F}_{\theta_T}(t \mid X)) \indicator{T > t}
       \Big]
\end{align*}
\Gls{ipcw} \gls{bll} can likewise  be written as \citep{kvamme2019time}:
\begin{align*}
    \text{F-NBLL-CW}(t;\theta)
   &= \E_{T,C,X} \Big[\frac{-\log({F_{\theta_T}}(t \mid X)) \Delta  \indicator{U \leq t}}{G(U^{-} \mid X)}
       + \frac{-\log(\overline{F}_{\theta_T}(t \mid X)) \indicator{U > t}}{G(t \mid X)}\Big]
\end{align*}

\section{Game Algorithm \label{appsec:gamealg}}

\begin{algorithm}[h]
\begin{algorithmic}
 \STATE {\bfseries Input:}
    Choice of losses $\ell_F,\ell_G$,
    learning rate $\gamma$
 \STATE {\bfseries Initialize} $\theta_{Tt}$ and $\theta_{Ct}$ randomly for $t=1, \dots, K-1$
         \REPEAT 
             \STATE{\textbf{// } for each parameter of each player}
     \FOR{$t=1$ {\bfseries to} $K-1$ }

        \STATE $g_{Tt} \leftarrow  d 
        \ell_F^t/d \theta_{Tt}$
     \STATE $g_{Ct} \leftarrow d \ell_G^t/d \theta_{Ct}$
     \ENDFOR
        \STATE{\textbf{// } for each parameter of each player}
         \FOR{$t=1$ {\bfseries to} $K-1$ }

        \STATE $\theta_{Tt} \leftarrow  \theta_{Tt} - \gamma g_{Tt}$
     \STATE $\theta_{Ct} \leftarrow  \theta_{Ct} - \gamma g_{Ct}$
     \ENDFOR
     \UNTIL convergence
\end{algorithmic}
\vskip -0.05in
\caption{\label{alg:mul-step} Following Gradients in Multi-Player Games}
\end{algorithm}

\section{Experiments \label{appsec:experiments}}

\subsection{Data}

\paragraph{Gamma Simulation}
We draw $x$ from a 32 dimensional multivariate normal $\mathcal{N}(0,10I)$. We simulate conditionally gamma failure times
with mean $\mu_t$ a log-linear function of $x$
with coefficients for each feature drawn $\text{Unif}(0,0.1)$. The censoring times are also conditionally gamma
with mean $0.9*\mu_t$. Both distributions have constant variance $0.05$. $\alpha,\beta$ parameterization of the gamma is recovered from mean, variance by
$\alpha=\mu^2/\sigma^2$ and $\beta=\mu/\sigma^2$.
$T$ and $C$ are conditionally independent given $X$. Each random seed draws a new dataset.

We report metrics as a function of training size.
We use training sizes [200,400,600,800,1000].
We use validation size 1024 and testing size 2048.

\paragraph{Survival \acrshort{mnist}} Survival-\acrshort{mnist}
\citep{gensheimer2019scalable,polsterl2019survivalmnist}
draws times  conditionally on \acrshort{mnist} label $Y$. This means digits define risk groups and $T \indep X \mid Y$. Times within a digit are i.i.d. The model only sees the image pixels $X$ as covariates so it must learn to classify digits (risk groups) to model times. The PyCox package \citep{kvamme2019time} uses Exponential times. We follow \cite{goldstein2020x} and use Gamma times. $T$'s mean
is $10*(Y+1)$ so that lower labels $Y$ mean sooner event times. We set the variance constant to $0.05$. $C$ is drawn similarly but with $9.9*(Y+1)$. Each random seed draws a new dataset.

We report metrics as a function of training size.
We use training sizes [512, 1024, 2048, 4096, 8192, 10240].
We use validation size 1024 and testing size 2048.

\paragraph{Real Data}

We report results on
\begin{itemize}
    \item \gls{support}
    \citep{knaus1995support}
    which includes severely ill hospital patients.
    There are 14 features. we split into
    5,323 for training, 
    1774 for validation, and 
    1776 for testing.

    \item \gls{metabric} \citep{curtis2012genomic}.
    There are 9 features.
    We split into 1,142 for training, 380 for validation,
    and 382 for testing.
    
    \item \gls{rott} \citep{foekens2000urokinase}
    and \gls{gbsg} \citep{schumacher1994randomized} combined into one dataset (\acrshort{rott-gbsg}).
    There are 7 features. We split into 1,339 for training,
    446 for validation, and 447 for testing.

\end{itemize}
    For more description see
\cite{therneau2021survival,katzman2018deepsurv,chen2020deep}.

In the main text, we report results on a subset of these datasets with metrics as a function of training size.
We use training sizes
[10, 20, 30, 40, 50, 60, 70, 80, 90, 100, 110, 120, 130, 140, 150, 175, 200].
We use validation size 300 and always use the entire testing set. We standardize all real data with the training set mean and standard deviation.

\subsection{Models}
In all experiments except for \acrshort{mnist}, we use
a 3-hidden-layer ReLU network. The hidden sizes are 
[128, 64, 64] for the Gamma simulation and
[128,256,64] for the real data. We output $20$ categorical bins.
See \cref{appsec:numbins} for different choices of number of bins,
which did not show any significant differences in results.
For \acrshort{mnist} we first use a small convolutional network and follow with the same fully-connected network, but
using hidden sizes [512,256,64]. 

\subsection{Training}

We use learning rate $0.001$ in all experiments for all losses using the Adam optimizer. We train for 300 epochs for the simulated data and 200 for the real data. For all data and all losses, this was enough to overfit on the training data. We use no weight decay or dropout.

\subsection{Model Selection}
We select the best model on the validation set using the following approach:
\begin{enumerate}
    \item Save the $F$ and $G$ models from all the epochs in $F$-set and $G$-set.
    \item Randomly choose a model $\tilde{F}$ in the $F$-set.
    \item Use $\tilde{F}$ as the weight for $\ell_G$. Find the model $\tilde{G}$ from $G$-set to minimize $\ell_G$ weighted by $\tilde{F}$.
        \item Use $\tilde{G}$ as the weight for $\ell_F$. Find the model $\tilde{F}$ from $F$-set to minimize $\ell_F$ weighted by $\tilde{G}$.
        \item Repeat steps 3 and 4 until convergence.
\end{enumerate}
Once converged, we use $\tilde{F}$ and $\tilde{G}$ as our best model to evaluate at the test set. The above approach plays as similar role as the game. Instead of gradient descent, this time we select a model from a set to play the game. We first fix $F$ to find the best $G$ based on $\ell_G$ and then fix $G$ to find the best $F$ based on $\ell_F$.

\newpage 

\section{Ablations} 

\subsection{Changing number of bins on MNIST \label{appsec:numbins}} 

\textbf{Changing number of categorical bins (K) in [10,20,30,40,50]. Cannot directly compare between two choices of K due to changing meaning of likelihood/BS/Concordance but can compare \gls{nll} and \gls{bs}-Game at each K. Trends similar across all choices of K.}
\begin{figure}[h]
    \centering
    \subfigure[Uncensored \acrshort{bs}]{
        \includegraphics[width=32mm]{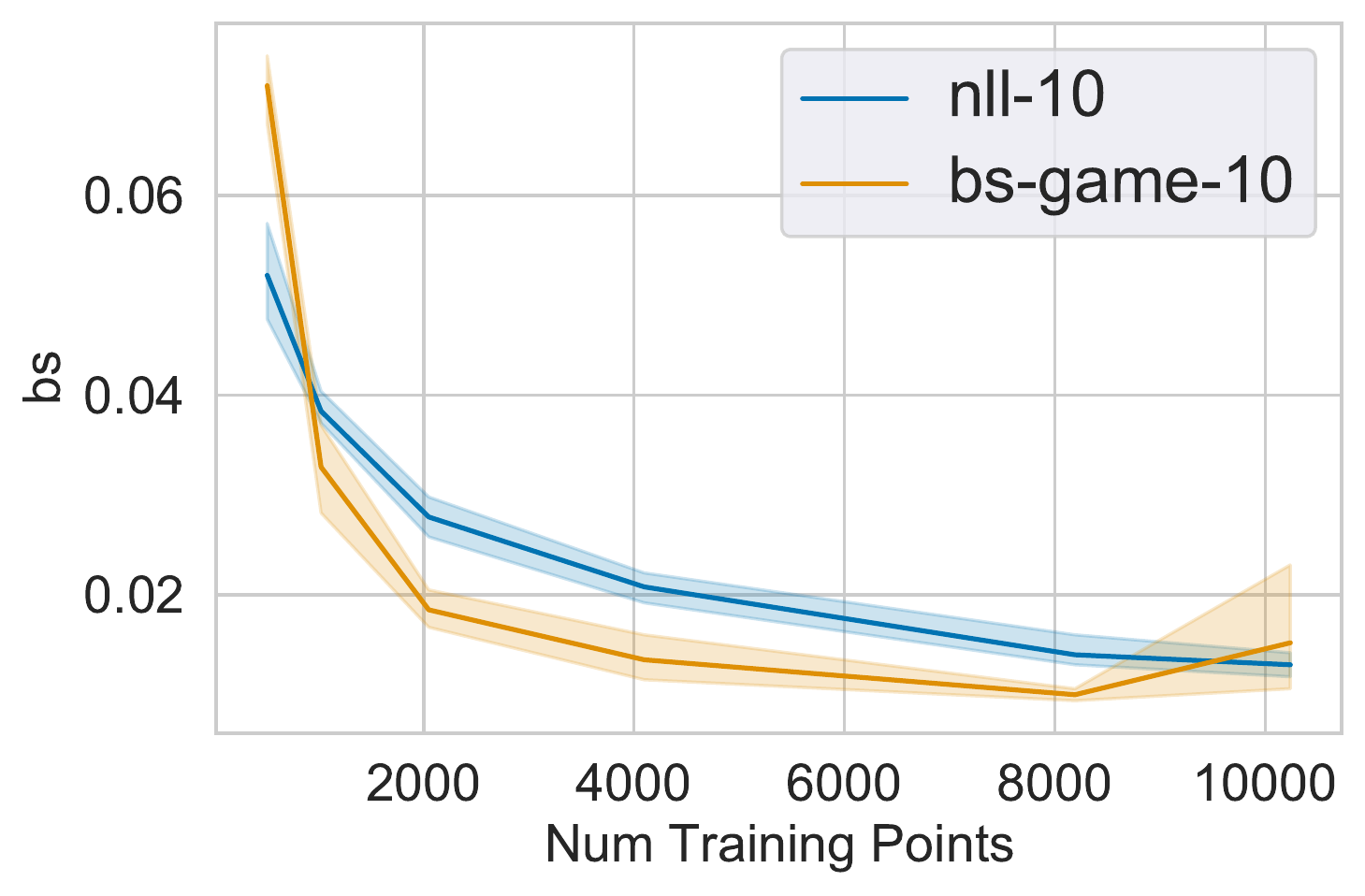}
    }
    \subfigure[Uncensored Neg \acrshort{bll}]{
        \includegraphics[width=32mm]{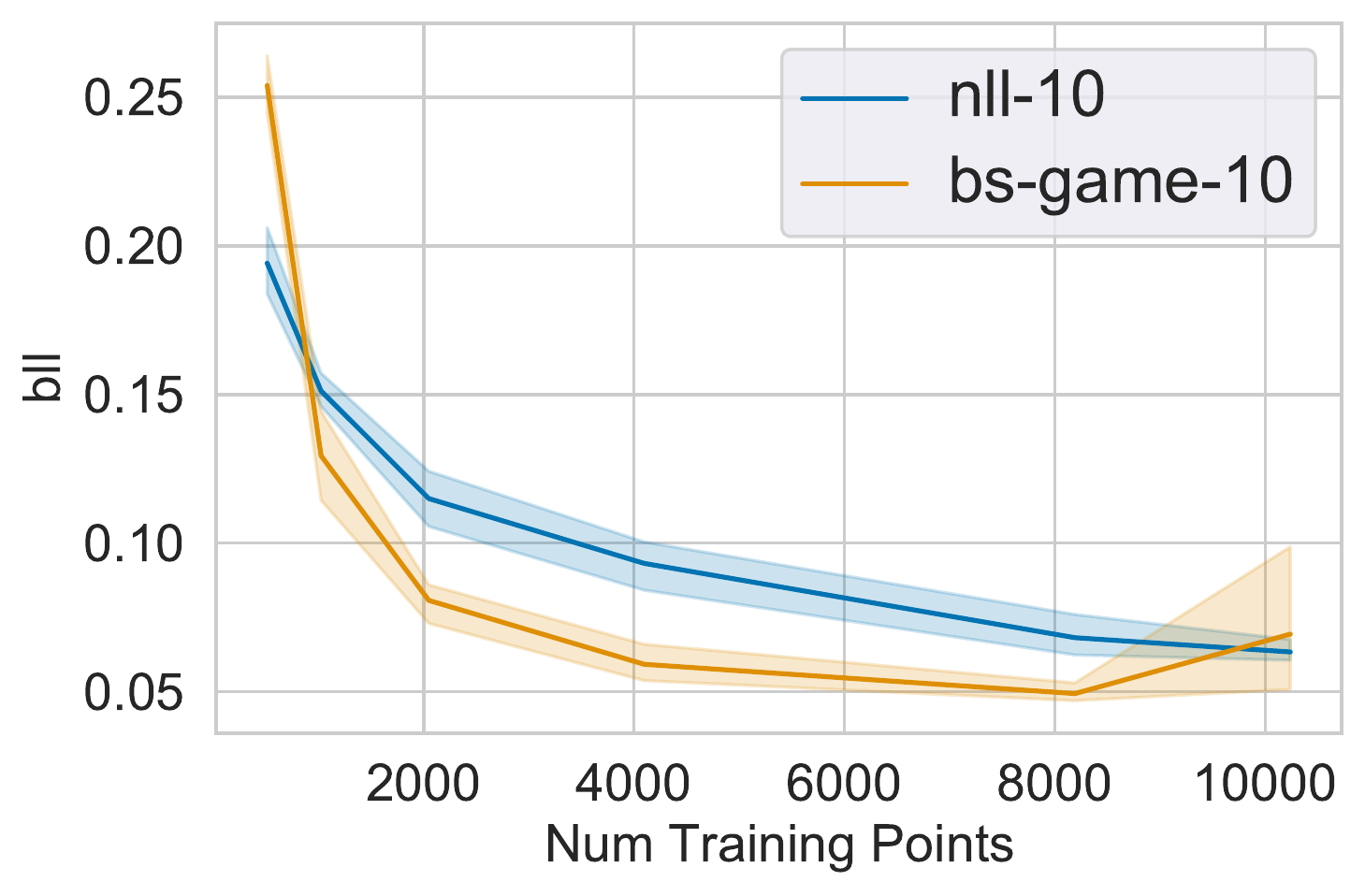}
    }
    \subfigure[Concordance]{
        \includegraphics[width=32mm]{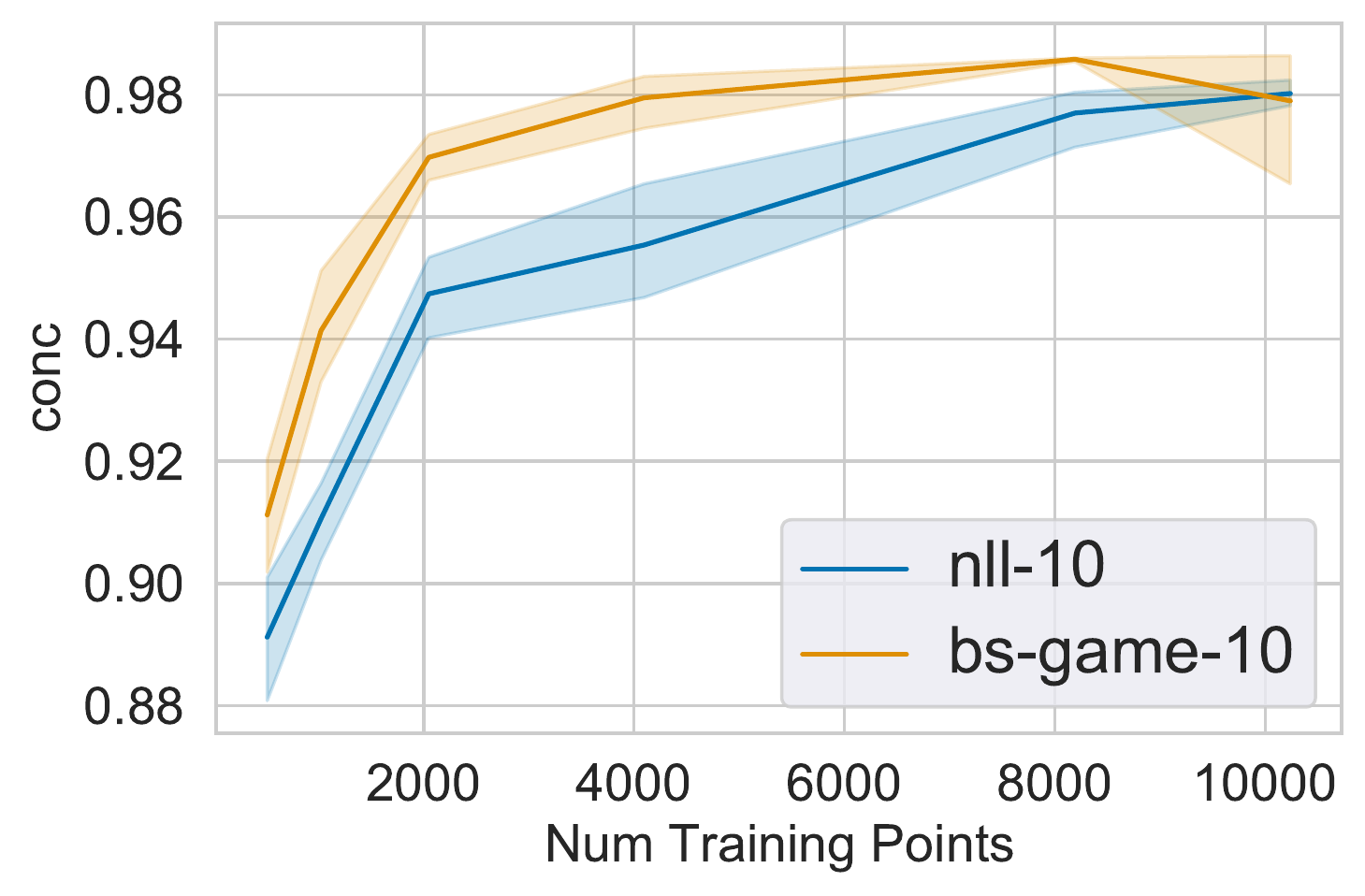}
    }
   \subfigure[Categorical \acrshort{nll}]{
        \includegraphics[width=32mm]{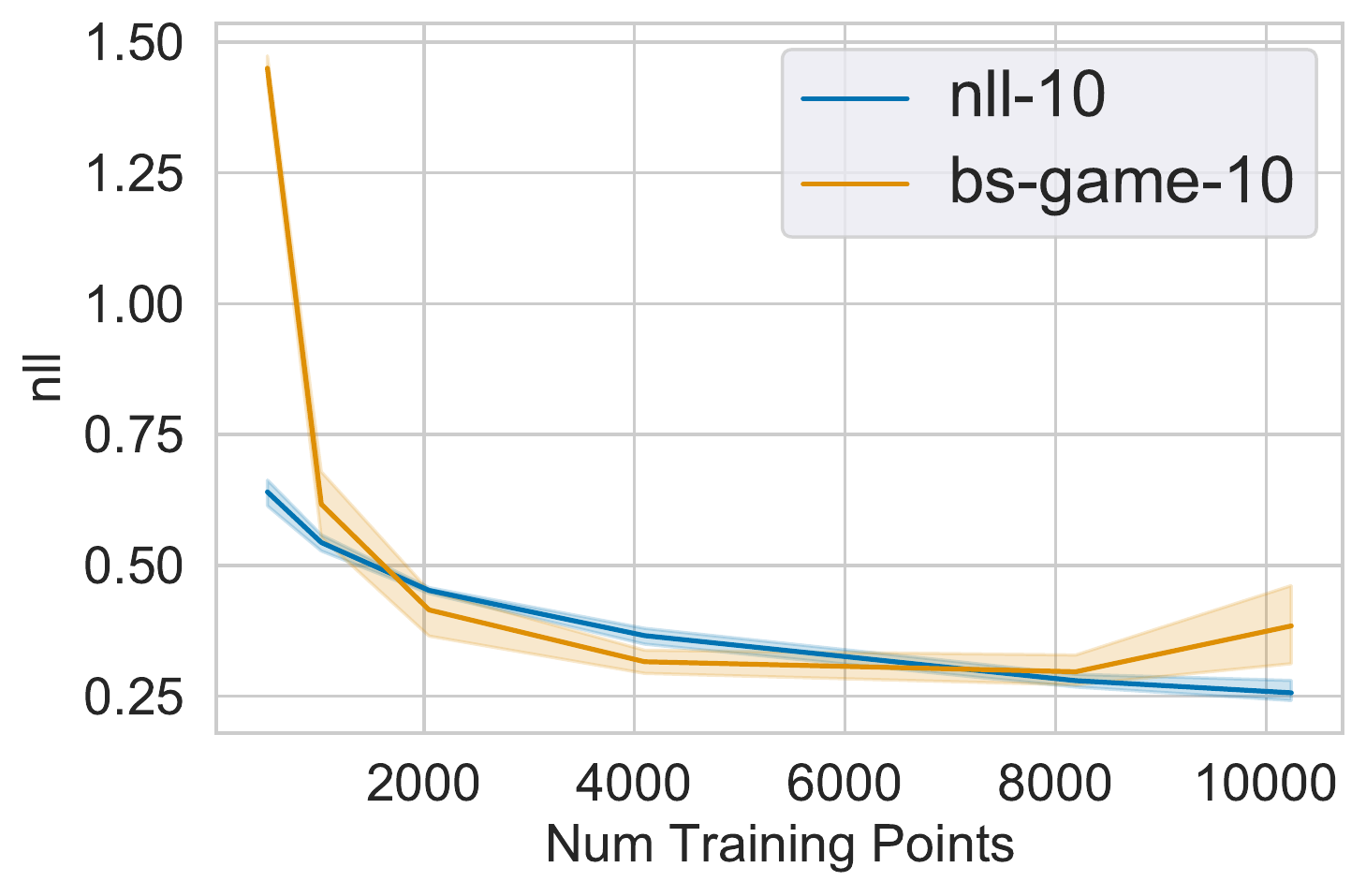}
    }
    \caption{10 bins.  \gls{nll} (Blue). \gls{bs}-Game (Orange).}
\end{figure}

\begin{figure}[h]
    \centering
    \subfigure[Uncensored \acrshort{bs}]{
        \includegraphics[width=32mm]{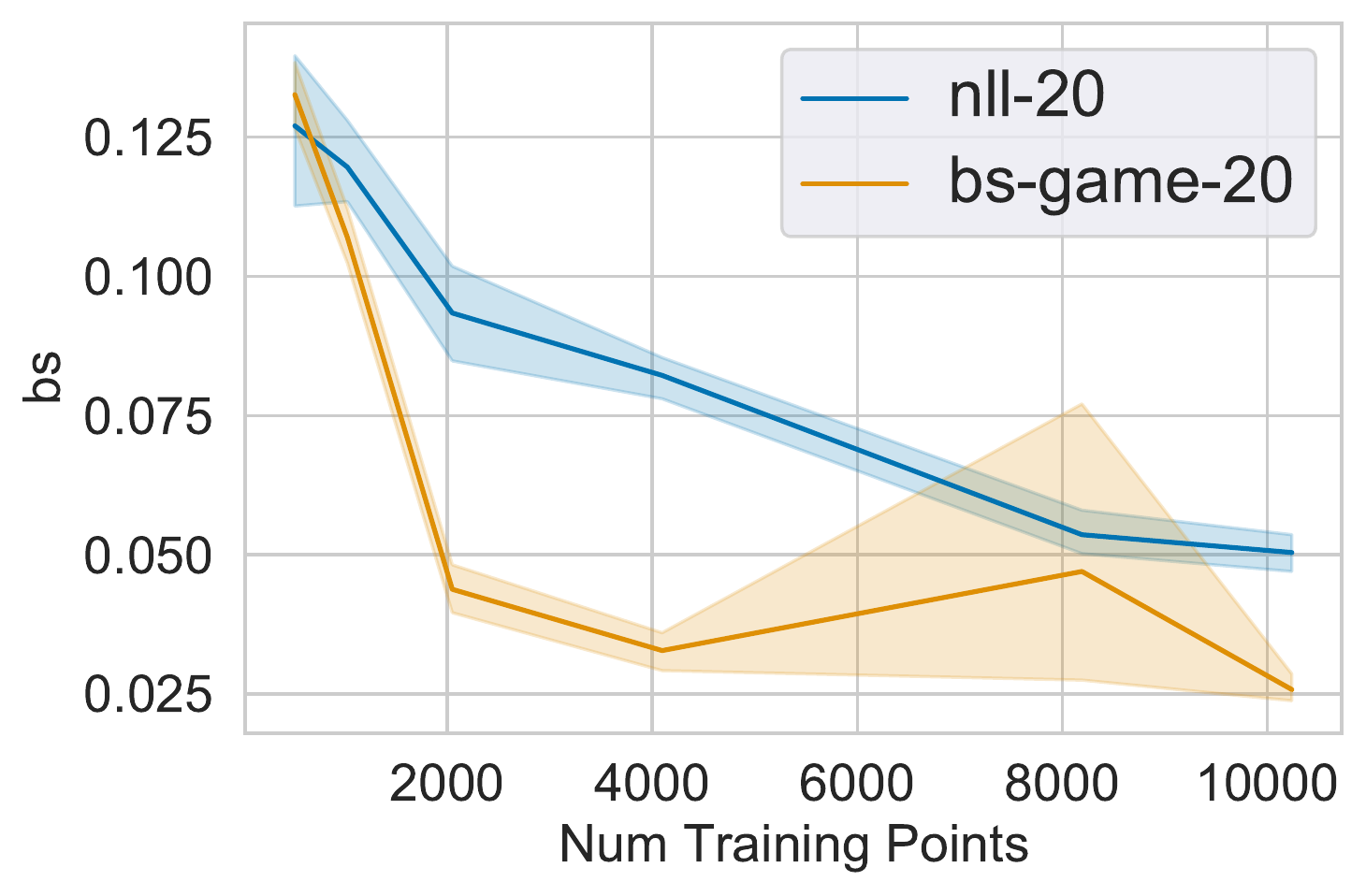}
    }
    \subfigure[Uncensored Neg \acrshort{bll}]{
        \includegraphics[width=32mm]{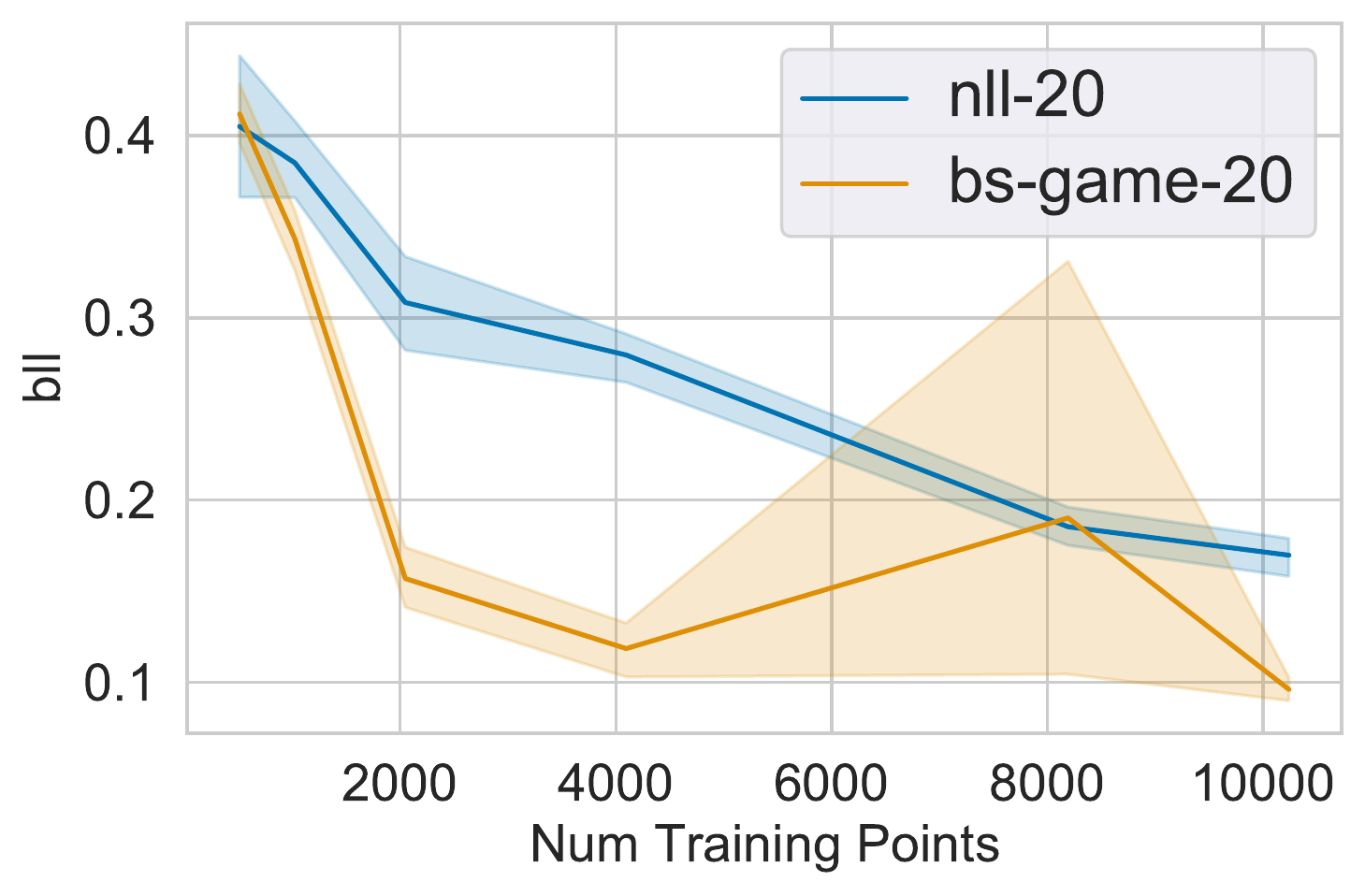}
    }
    \subfigure[Concordance]{
        \includegraphics[width=32mm]{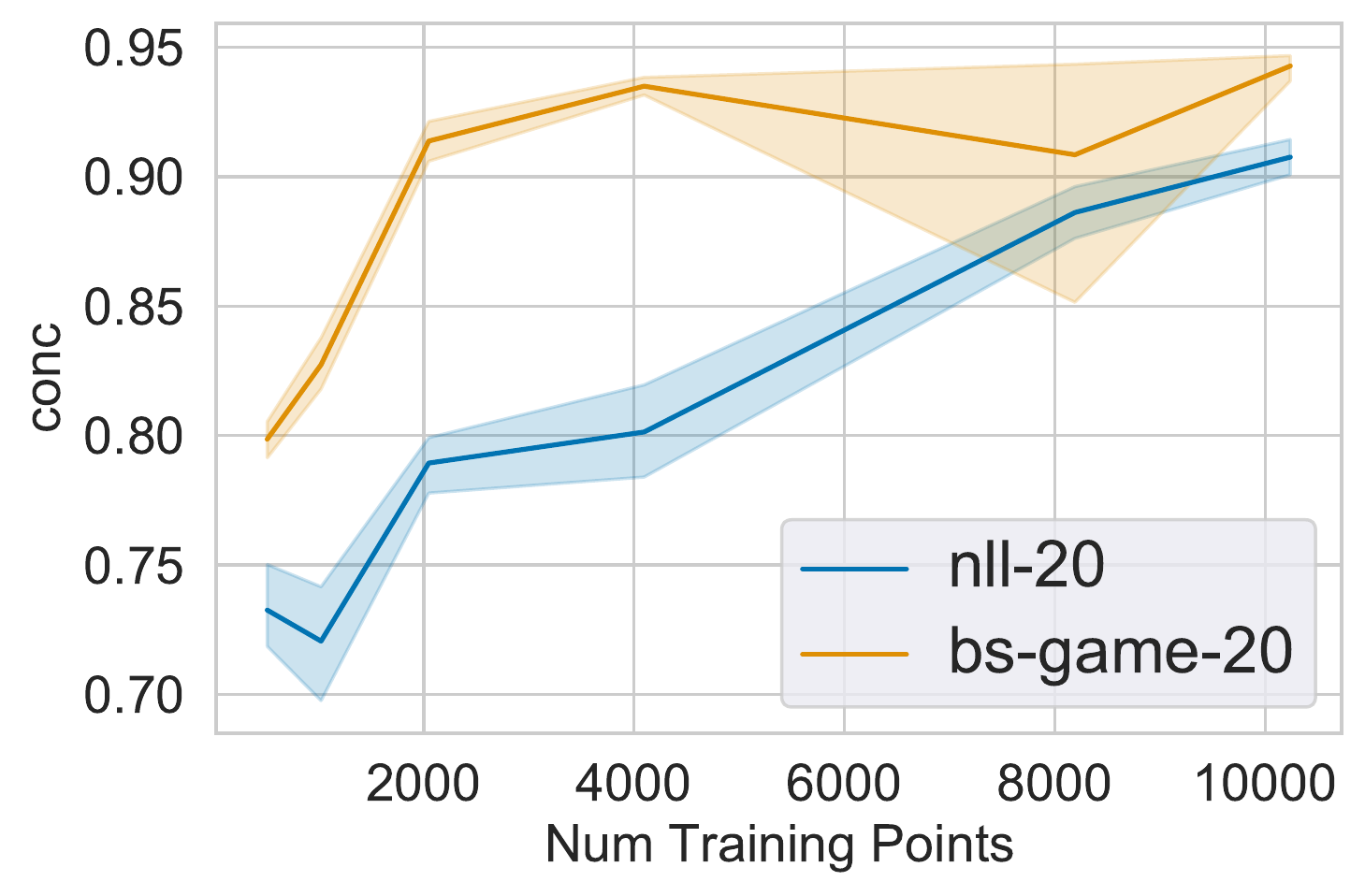}
    }
   \subfigure[Categorical \acrshort{nll}]{
        \includegraphics[width=32mm]{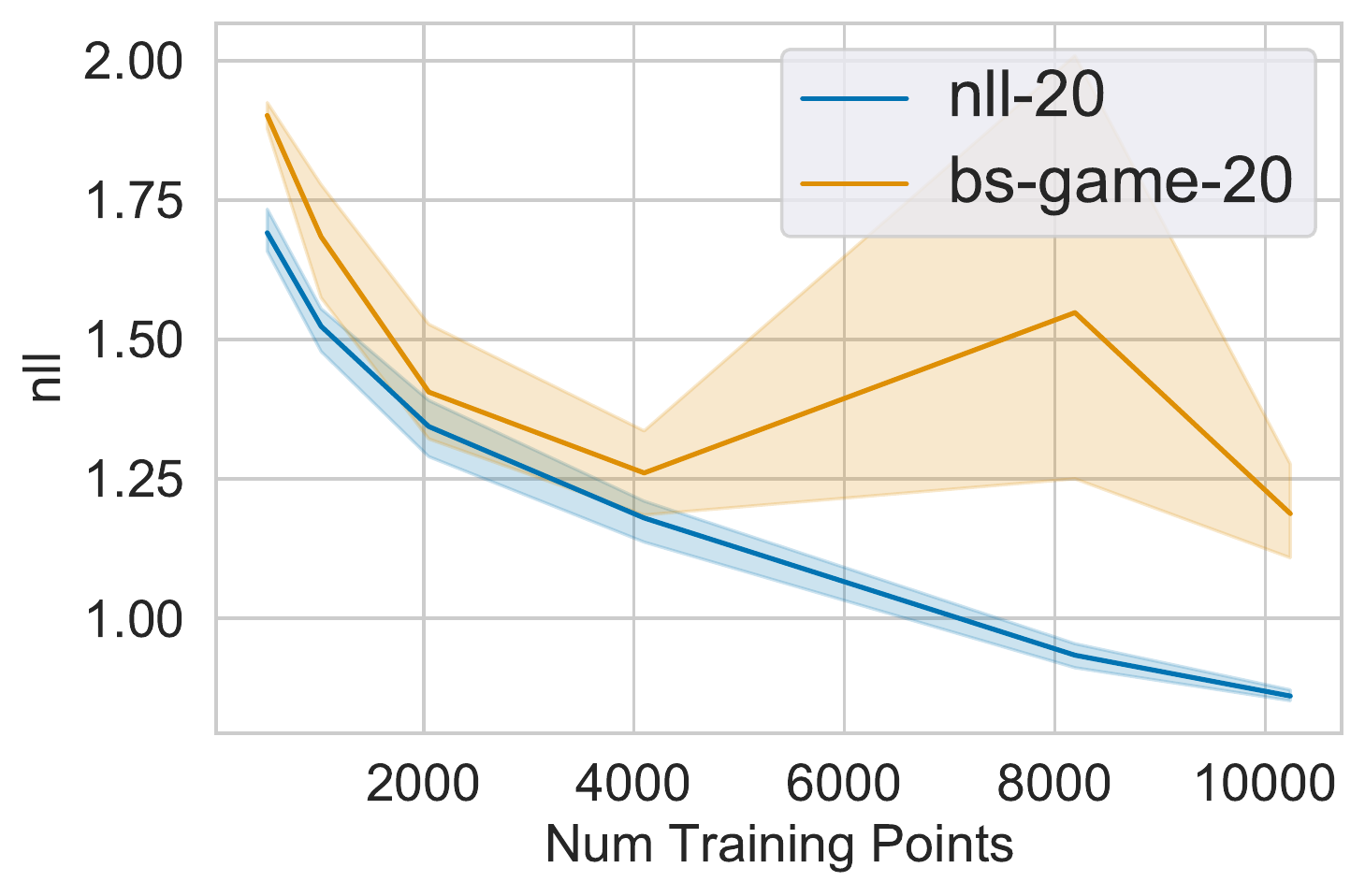}
    }
    \caption{20 bins.  \gls{nll} (Blue). \gls{bs}-Game (Orange).}
\end{figure}

\begin{figure}[h]
    \centering
    \subfigure[Uncensored \acrshort{bs}]{
        \includegraphics[width=32mm]{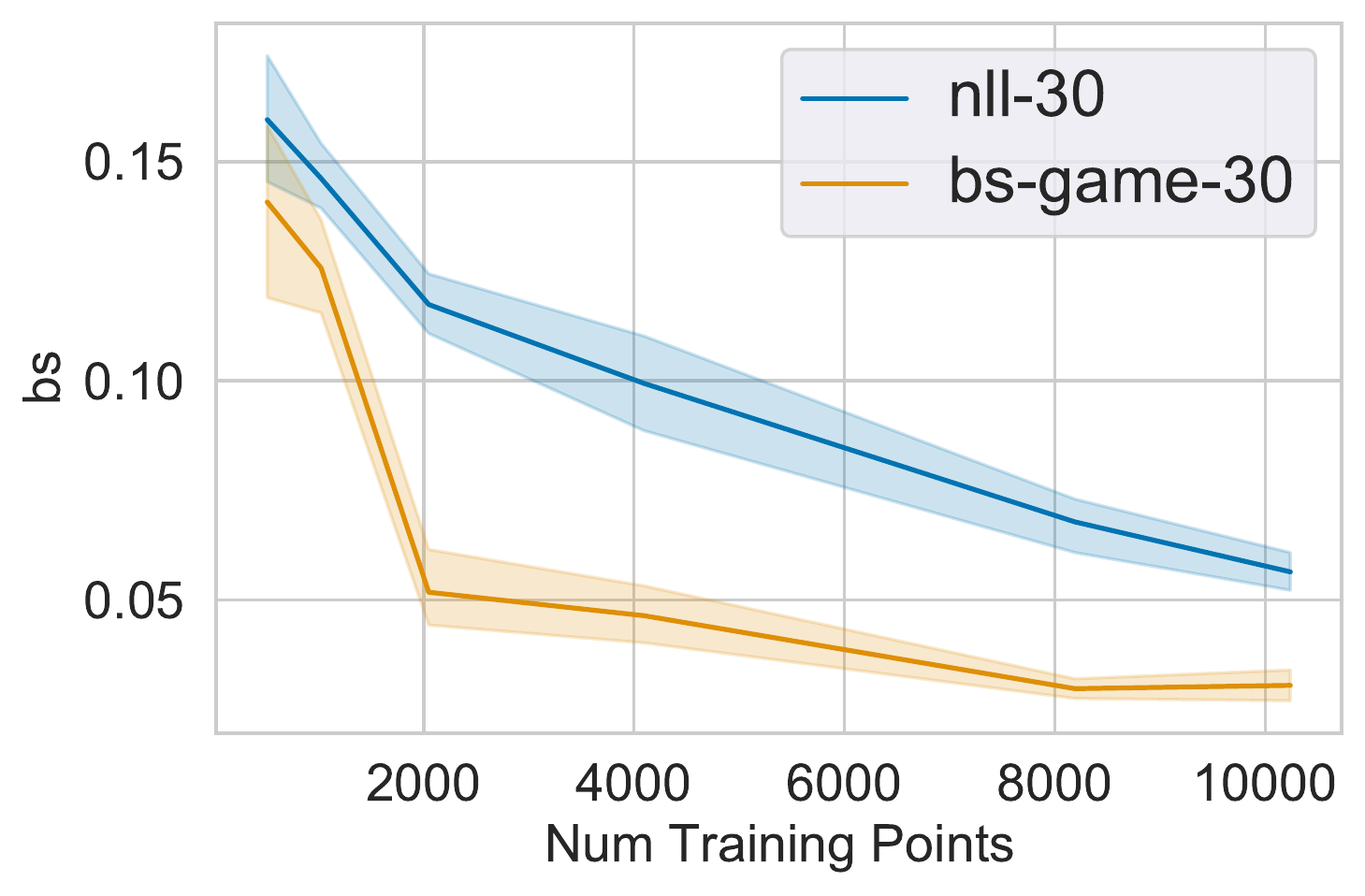}
    }
    \subfigure[Uncensored Neg \acrshort{bll}]{
        \includegraphics[width=32mm]{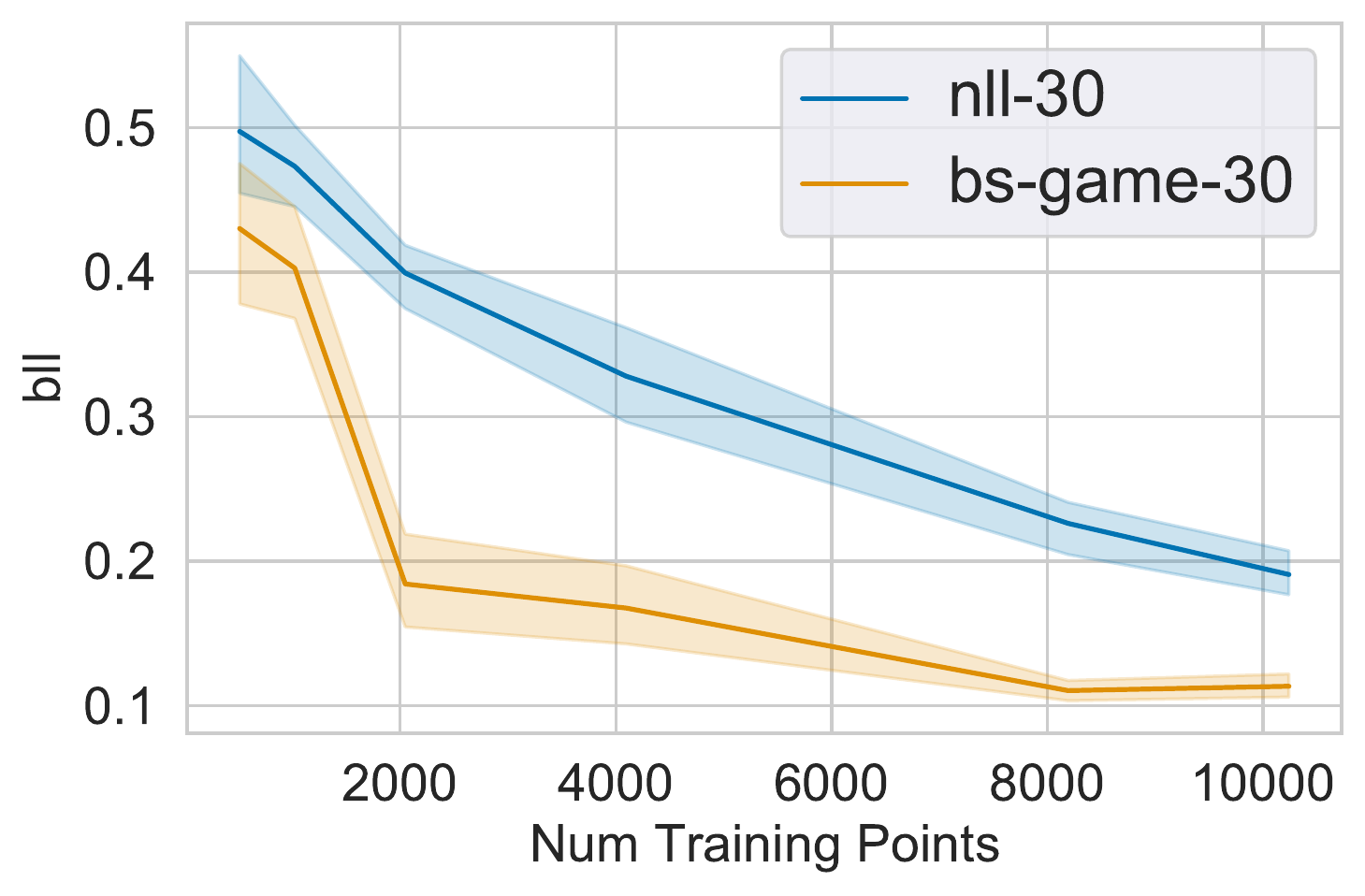}
    }
    \subfigure[Concordance]{
        \includegraphics[width=32mm]{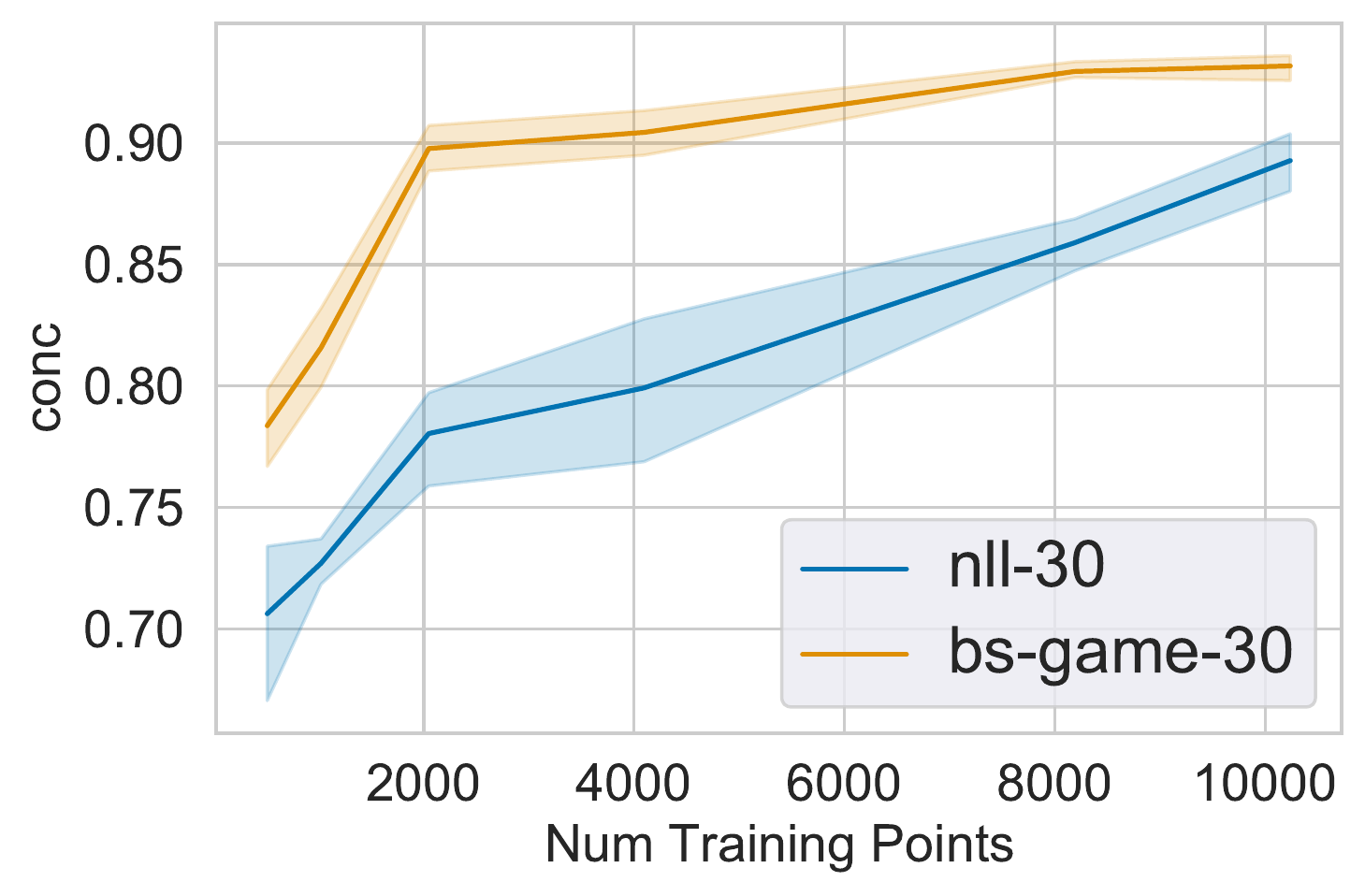}
    }
   \subfigure[Categorical \acrshort{nll}]{
        \includegraphics[width=32mm]{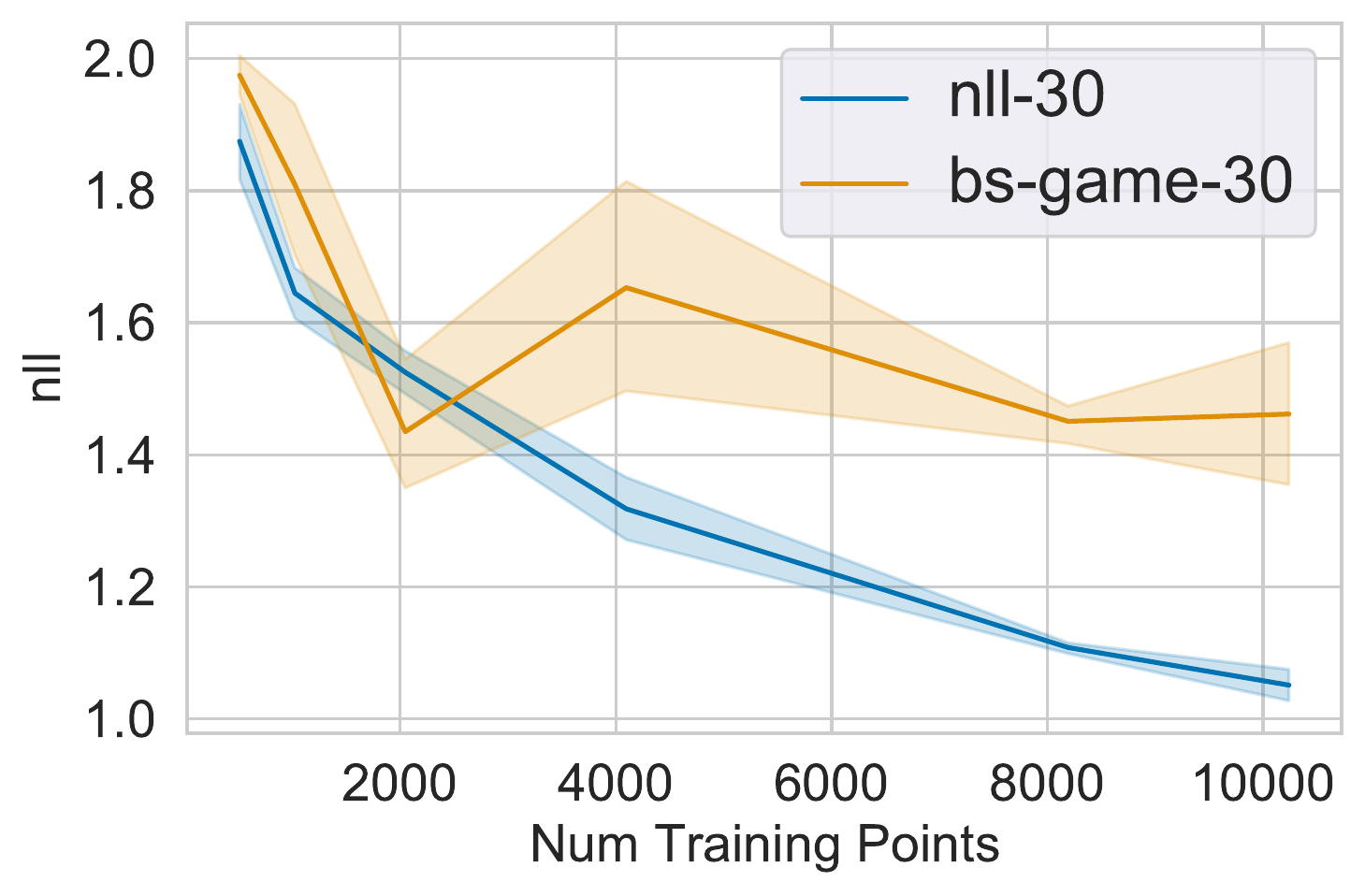}
    }
    \caption{30 bins.  \gls{nll} (Blue). \gls{bs}-Game (Orange).}
\end{figure}

\begin{figure}[h!]
    \centering
    \subfigure[Uncensored \acrshort{bs}]{
        \includegraphics[width=32mm]{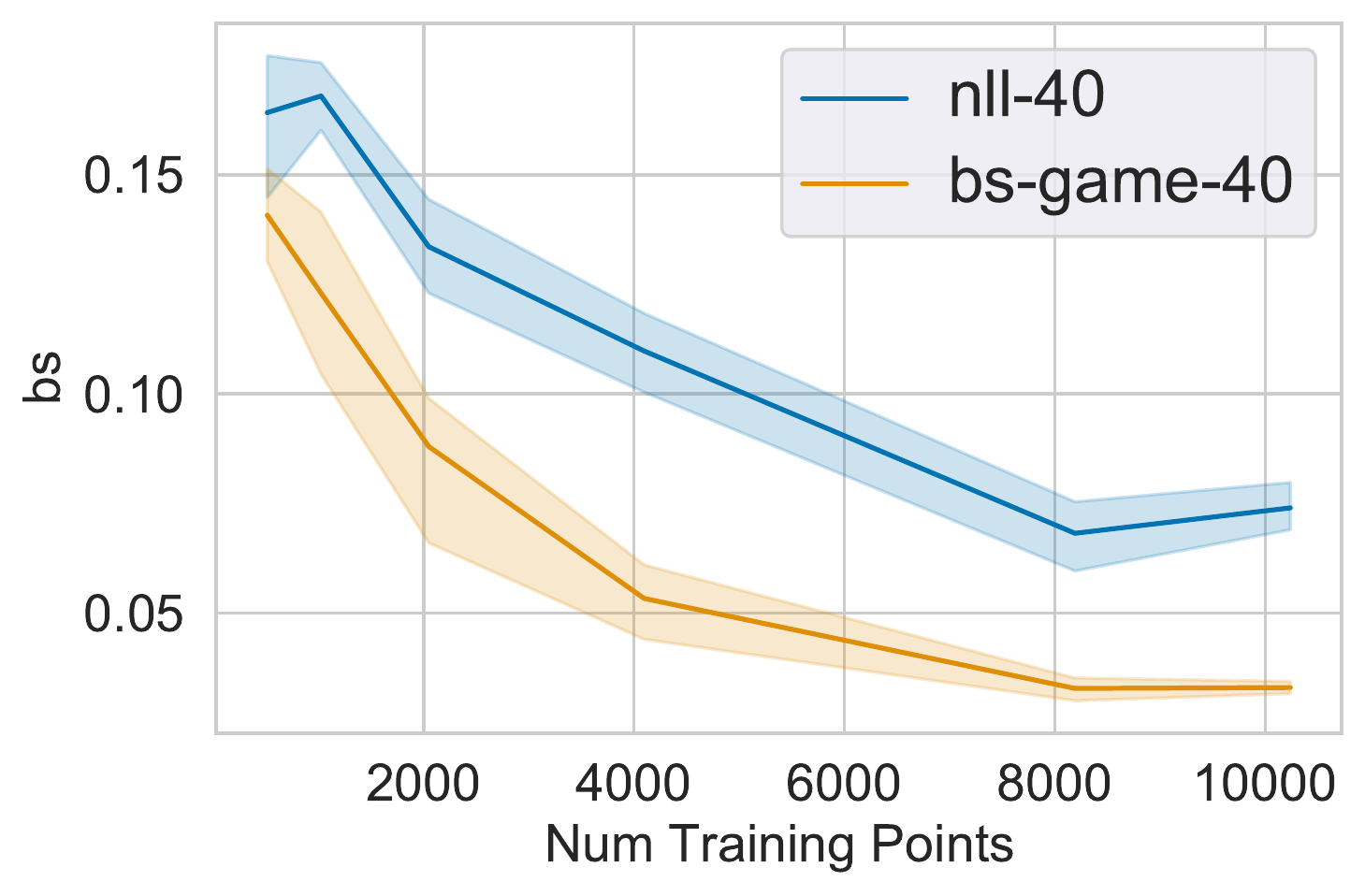}
    }
    \subfigure[Uncensored Neg \acrshort{bll}]{
        \includegraphics[width=32mm]{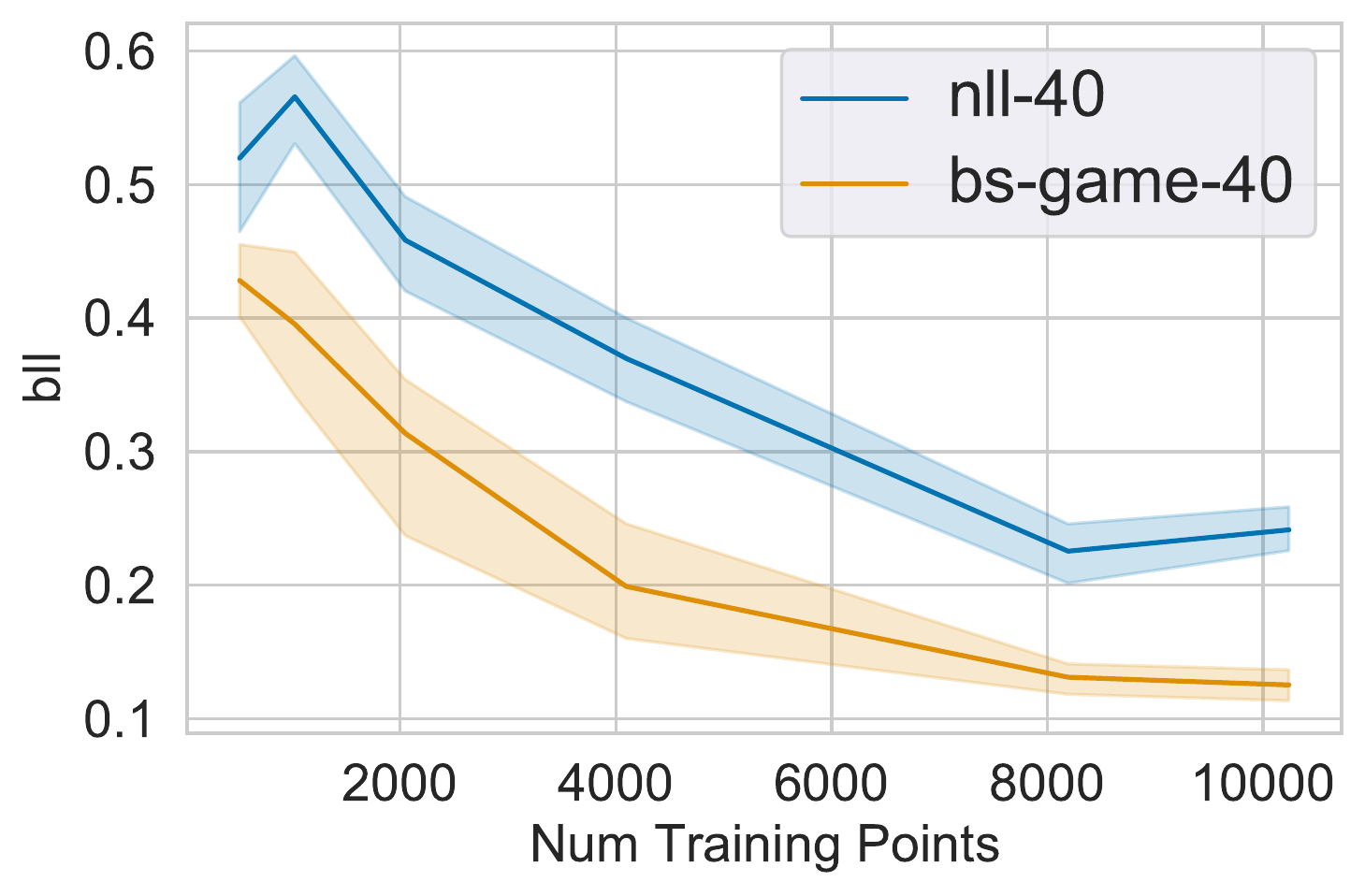}
    }
    \subfigure[Concordance]{
        \includegraphics[width=32mm]{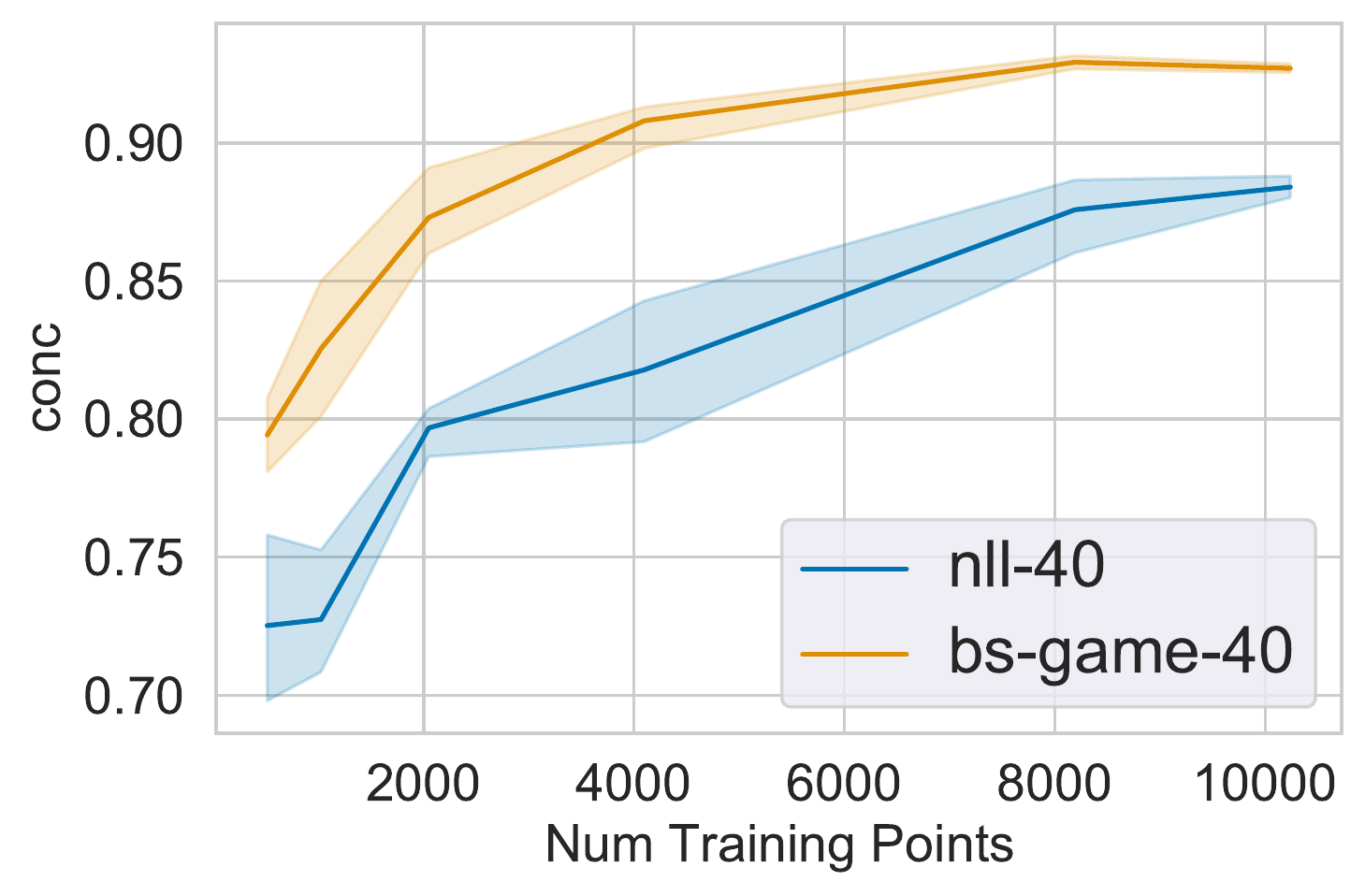}
    }
   \subfigure[Categorical \acrshort{nll}]{
        \includegraphics[width=32mm]{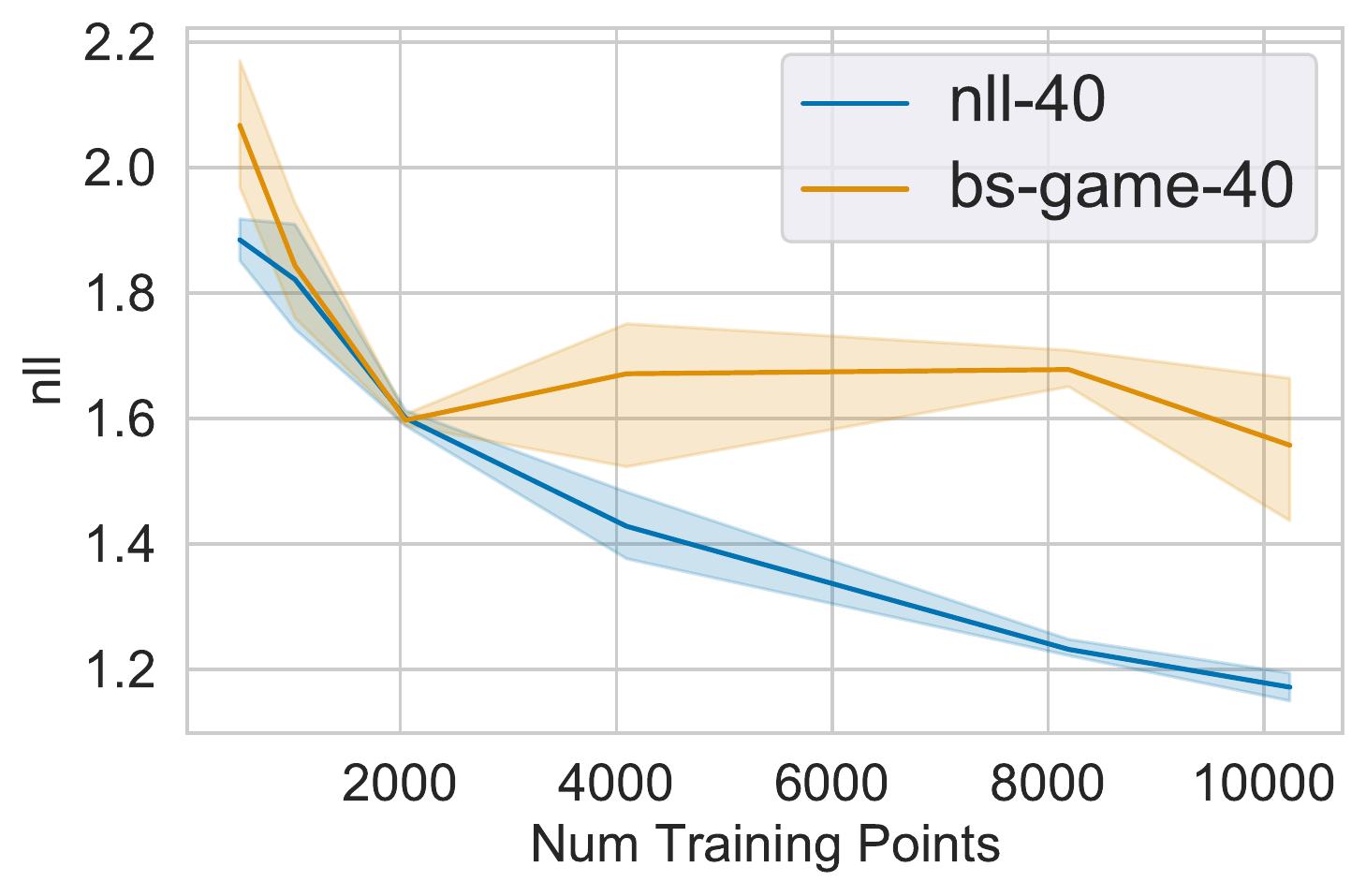}
    }
    \caption{40 bins.  \gls{nll} (Blue). \gls{bs}-Game (Orange).}
\end{figure}

\begin{figure}[h!]
    \centering
    \subfigure[Uncensored \acrshort{bs}]{
        \includegraphics[width=32mm]{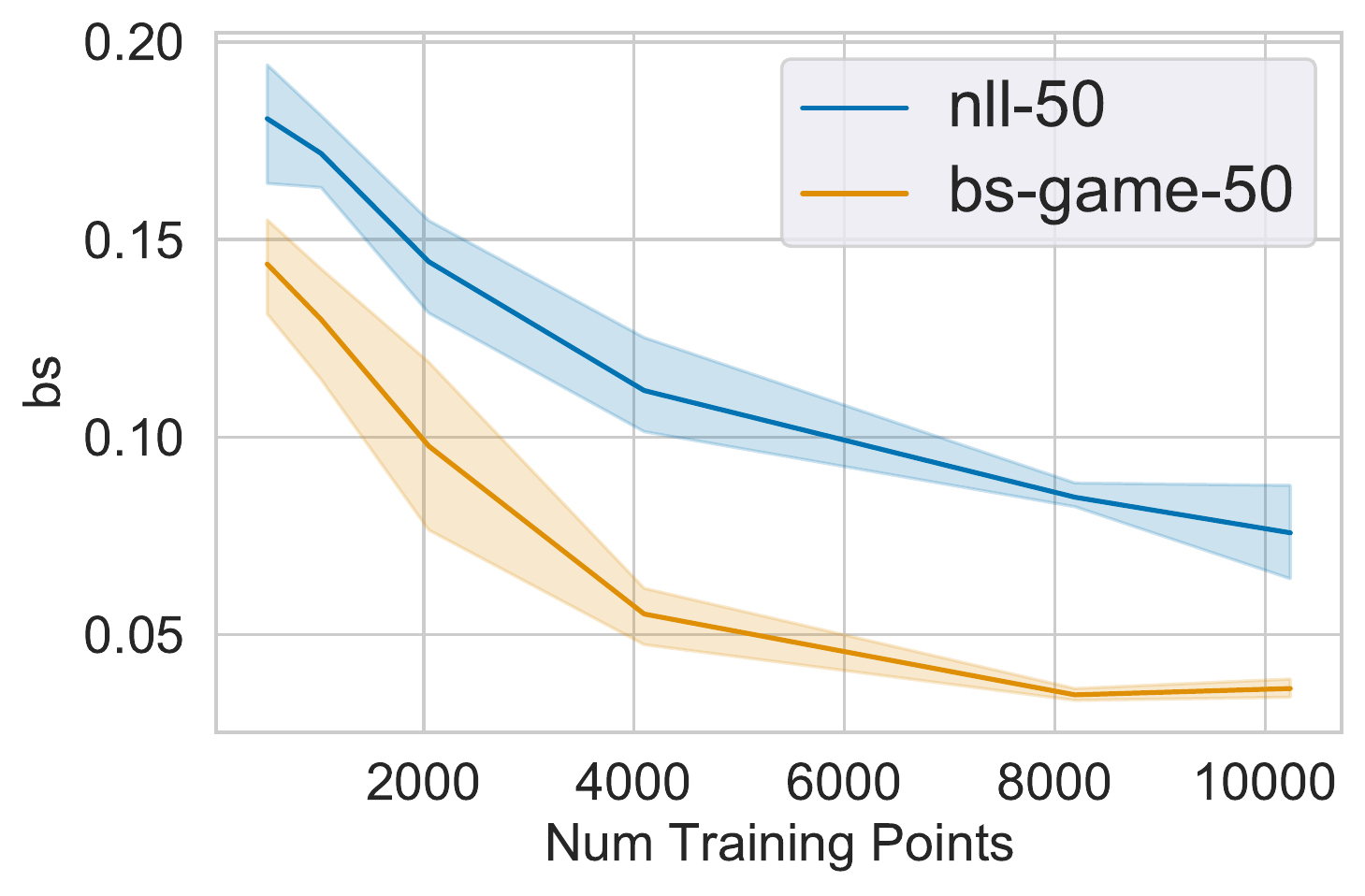}
    }
    \subfigure[Uncensored Neg \acrshort{bll}]{
        \includegraphics[width=32mm]{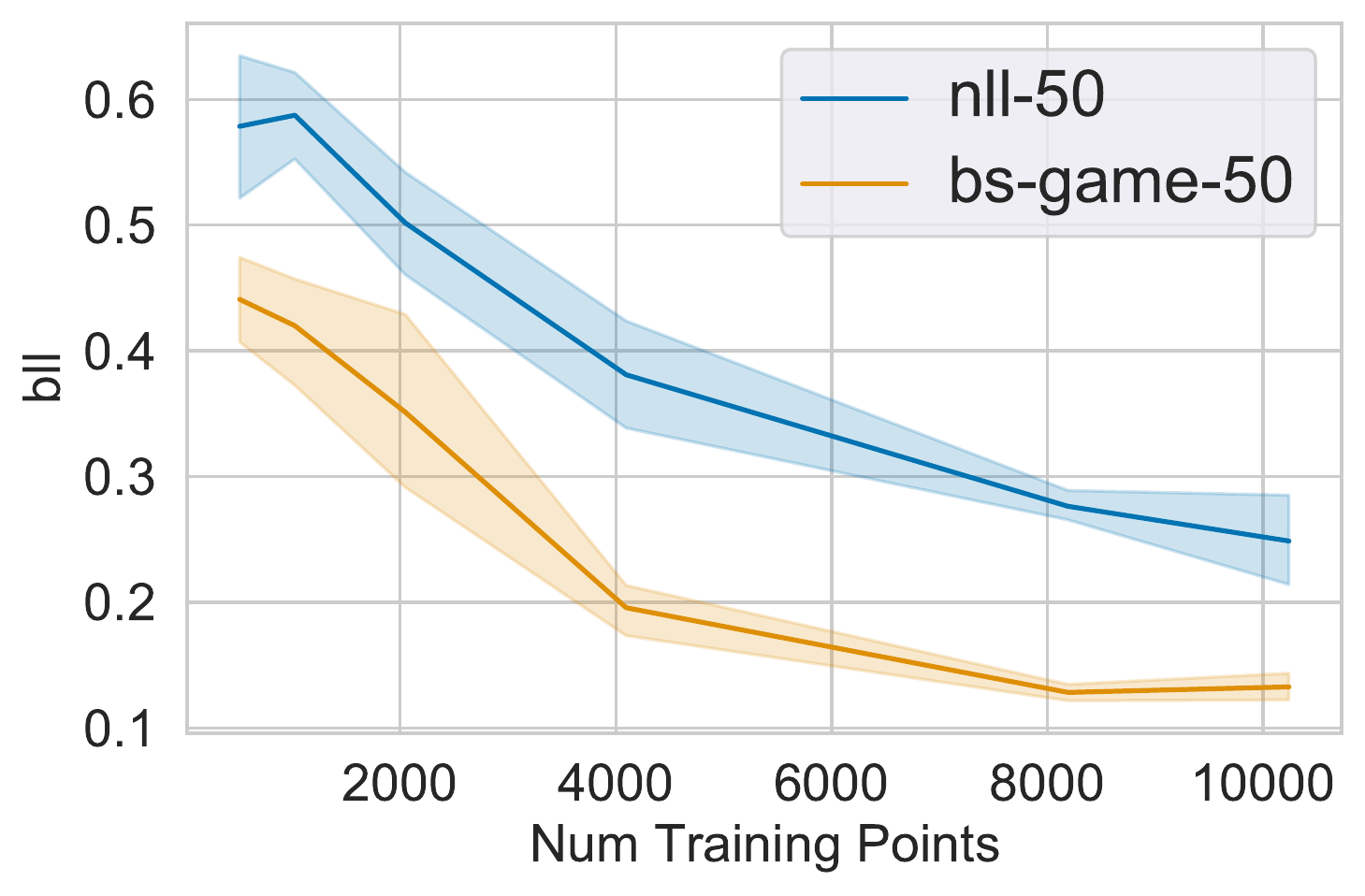}
    }
    \subfigure[Concordance]{
        \includegraphics[width=32mm]{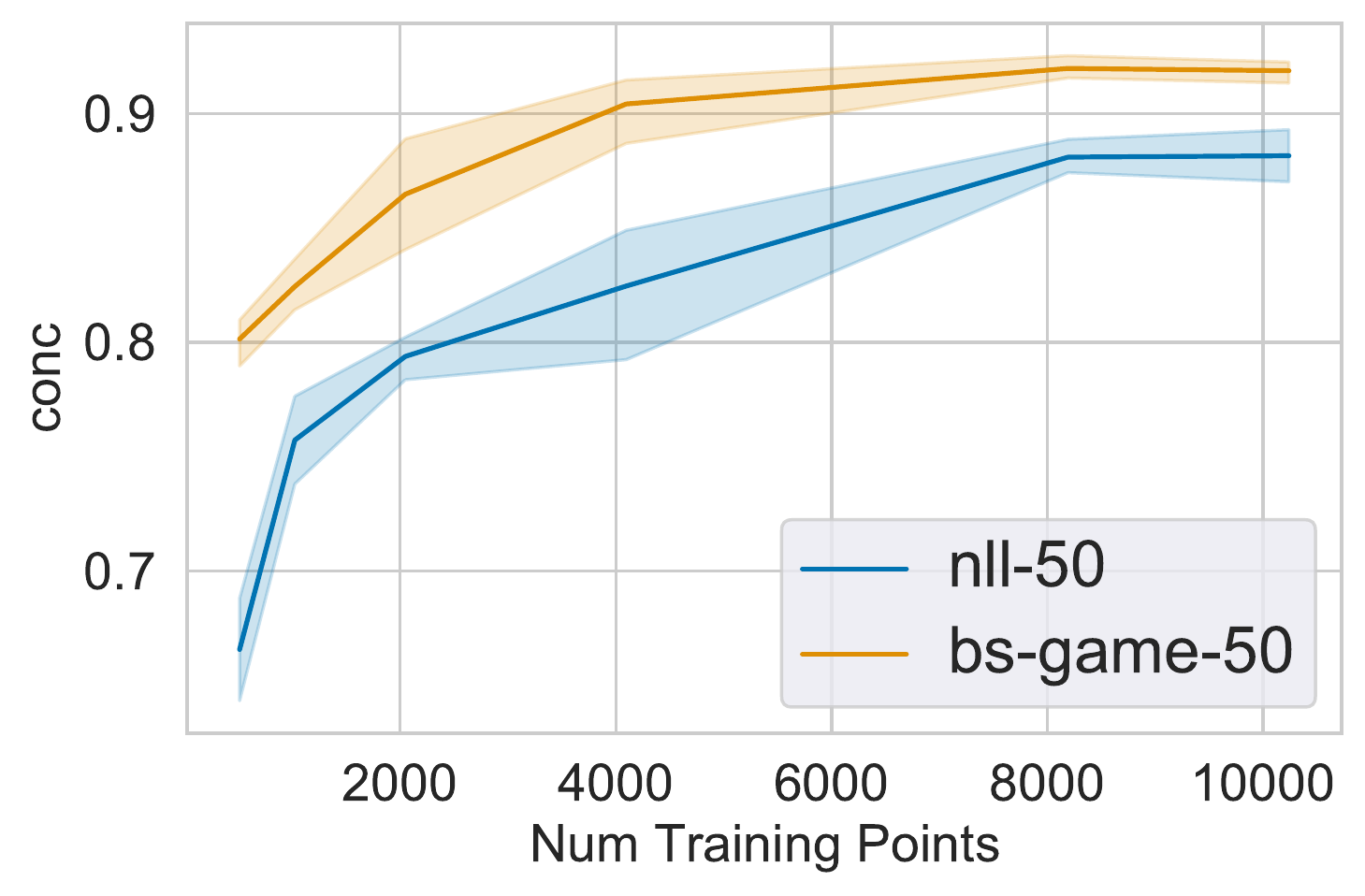}
    }
   \subfigure[Categorical \acrshort{nll}]{
        \includegraphics[width=32mm]{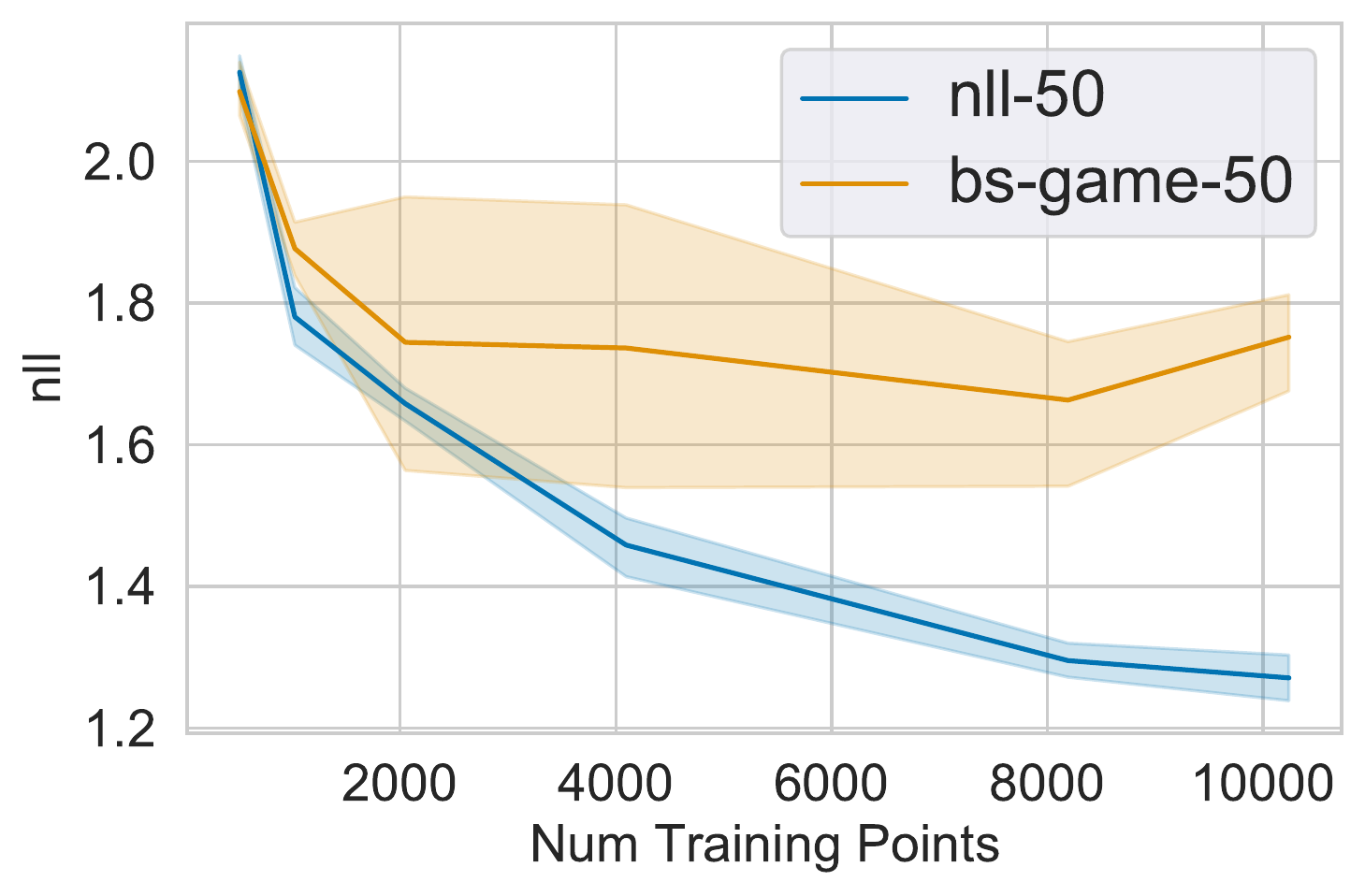}
    }
    \caption{50 bins.  \gls{nll} (Blue). \gls{bs}-Game (Orange).}
\end{figure}

\section{\label{appsec:sumgame} Proof of Summed or Integrated Brier Score to be proper}
\begin{proposition}
Assume we have a list of time $t_1, \dots, t_K$. Assume the true distribution for $T$ is $F^* =  F_{\theta_T^*}$ in \cref{eq:fbs}. We have:
\begin{itemize}
    \item The summed \gls{bs} $\sum_{i=1}^K BS(t_i; \theta)$ is proper, i.e., it has one minimizer at the true parameters $\theta_T^*$.
    \item The integrated \gls{bs} $\int_{t_1}^{t_K} BS(t; \theta) dt$ is proper, i.e., it has one minimizer at the true parameters $\theta_T^*$.
\end{itemize}
\end{proposition}
\begin{proof}
Since $\gls{bs}(t)$ is proper, it has one minimizer at $\theta_T^\star$, i.e., for $\theta_T \neq \theta_T^*$, $\gls{bs}(t; \theta_T^\star) \leq \gls{bs}(t; \theta_T)$ for all $t$. Since this holds for all $t$, we then have:
\[
\sum_{i=1}^K BS(t_i; \theta_T^\star) \leq \sum_{i=1}^K BS(t_i; \theta_T).
\]
This means that the summed Brier Score at $\theta_T^*$ is smaller than at any other $\theta_T$. The summed \gls{bs} has one minimizer at the true parameters $\theta_T^*$, i.e., it is  proper.
Since the BS inequality holds for all $t$, we also have 
\[
\int_{t_1}^{t_K} BS(t; \theta_T^\star) dt \leq \int_{t_1}^{t_K} BS(t; \theta_T) dt
\]
This means that the integrated Brier Score at $\theta_T^*$ is smaller than at any other $\theta_T$. The integrated \gls{bs} has one minimizer at the true parameters $\theta_T^*$, i.e., it is proper.
\end{proof}
\section{\label{appsec:ipcwgames}Proof of \cref{prop: exist}}
Here we prove that the true solution is a stationary point of the game. We restate the proposition here. 
\begin{proposition*}
Assume $\exists \theta_T^\star \in \Theta_T,\exists \theta_C^\star \in \Theta_C$ such that $F^\star=F_{\theta_T^\star}$ and
$G^\star=G_{\theta_C^\star}$.
Assume the game losses $\ell_F,\ell_G$ are based on proper losses $L$
and that the games are only computed at times for which positivity holds.
Then $(\theta_T^\star,\theta_C^\star)$ is a stationary point of the game \cref{eq:iwgame}.
\end{proposition*}

\begin{align}
\tag{\ref{eq:iwgame}}
\begin{split}
    \ell_{F}(\theta) = 
 L_{I}(F_{\theta_T};G_{\theta_C}), \quad 
    \ell_{G}(\theta) &=
    L_{I}(G_{\theta_C}; F_{\theta_T})
    \end{split}
\end{align}
\begin{proof}
In $\ell_F(\theta)$, by the definition of the \gls{ipcw} estimator, when $\theta_C = \theta_C^*$, $L_I(F_{\theta_T}; G_{\theta_C}) = L(F_{\theta_T})$. Due to the fact that $L$ is proper, $\theta_T^*$ is a minimizer for $ L(F_{\theta_T})$. Then at $(\theta_T, \theta_C) = (\theta_T^*, \theta_C^*)$,
we have 
$$
\left.
\frac{d \ell_F(\theta)}{ d \theta_T}\right|_{\substack{\theta_T = \theta_T^*\\ \theta_C=\theta_C^*}}=\left.
\frac{d L_I(F_{\theta_T}; G_{\theta_C^*})}{ d \theta_T}\right|_{\theta_T = \theta_T^*} = \left.\frac{d L(F_{\theta_T})}{ d \theta_T}\right|_{\theta_T = \theta_T^*} = 0
$$
Similarly for $\ell_G(\theta)$, we have
$$
\left.
\frac{d \ell_G(\theta)}{ d \theta_C}\right|_{\substack{\theta_C = \theta_C^* \\ \theta_T=\theta_T^*}}=\left.
\frac{d L_I(G_{\theta_C}; F_{\theta_T^*})}{ d \theta_C}\right|_{\theta_C = \theta_C^*} = \left.\frac{d L(G_{\theta_C})}{ d \theta_C}\right|_{\theta_C = \theta_C^*} = 0
$$
Since the two gradients are zero, the game will stay at the true parameters. Therefore, $(\theta_T^\star,\theta_C^\star)$ is a stationary point of the game \cref{eq:iwgame}.
\end{proof}

\section{\label{appsec:brierstationary} Proof of \cref{prop: uniq}}
Here we prove that under one construction of the game in \cref{alg:mul-step}, the true solution is the unique stationary point of the game. We restate the proposition here. 
\begin{proposition*}
Consider discrete distributions over $K$ times. Let $\theta_T = \{\theta_{T1},\cdots, \theta_{T(K-1)}\}$, $\theta_{Tt}=P_{\theta}(T=t)$, $F_{\theta_T}(t)= \sum_{k=1}^t \theta_{Tk}$, and likewise for $C,\theta_C$.
Assuming that $\theta^\star_{Tt}>0$ and $\theta^\star_{Ct}>0$,
the solution
 $(\theta_{T}^\star, \theta_{C}^\star)$ is the only stationary point for the multi-player \gls{bs} game
 shown in \cref{alg:mul-step}
 for times $t \in \{1,\ldots,K-1\}$ 
\end{proposition*}

\begin{proof} 
We show by induction on the time $t$ of the \gls{ipcw} \gls{bs} game that
the simultaneous gradient equations are only satisfied at $\hat{\theta_T}=\theta_T^\star$
and $\hat{\theta}_{C}=\theta_C^\star$.
There is a lot of arithmetic but eventually it comes down to (1) substitution of one variable for another (2) assuming all previous timestep parameters are correct (induction) (3)  finding the zeros of a quadratic (4) showing that one of the two solutions is the correct parameter and the other is invalid.

\textbf{Note:} this proof uses the notation that $\hat{\theta}$ is a model parameter and $\theta^\star$ is the correct one.

\subsection{BS(1) (base case)}

We can compute the expectations defining
$\text{F-BS-CW}(1)$ and  $\text{G-BS-CW}(1)$
in closed form. That gives us:
\begin{align*}
    \text{F-BS-CW}(1) &= \theta^\star_{T1}
    (1-\hat{\theta}_{T1})^2 
    +
    (1-\theta^\star_{T1})(1-\theta^\star_{C1})
        \frac{\hat{\theta}_{T1}^2}{1-\hat{\theta}_{C1}}  \label{eq:gweightedfbs} \\
    \text{G-BS-CW}(1) &= \frac{\theta^\star_{C1}(1 - \theta^\star_{T1})
    (1-\hat{\theta}_{C1})^2 }{1 - \hat{\theta}_{T1}}
    +
    (1-\theta^\star_{T1})
    (1-\theta^\star_{C1})
        \frac{\hat{\theta}_{C1}^2}{1-\hat{\theta}_{T1}}   
\end{align*}
The derivatives are
\begin{align*}
\frac{d\text{F-BS-CW}(1)}{d\hat{\theta}_{T1}}
&=
     2
 \frac{
 (1-\theta^\star_{T1} )
        (1-\theta^\star_{C1} )}{1-\hat{\theta}_{C1}}
    \hat{\theta}_{T1}
    -
    2(1-\hat{\theta}_{T1})
 \theta^\star_{T1}  = 0\\
 \frac{d\text{G-BS-CW}(1)}{d\hat{\theta}_{C1}}
&=
     2
 \frac{
  (1-\theta^\star_{T1})
 (1-\theta^\star_{C1} )
       }{1-\hat{\theta}_{T1}}
    \hat{\theta}_{C1}
    -
    2\frac{(1 - \theta^\star_{T1})(1-\hat{\theta}_{C1})
 \theta^\star_{C1}}{1 - \hat{\theta}_{T1}}  = 0
\end{align*}
We can take each derivative equation and write one variable in terms of the other. First, 
taking $d\text{F-BS-CW}/d\hat{\theta}_{T1}$
and writing $\hat{\theta}_{T1}$ in terms of 
$\hat{\theta}_{C1}$:
\begin{align*}
\frac{d\text{F-BS-CW}(1)}{d\hat{\theta}_{T1}}=
  2
 \frac{
 (1-\theta^\star_{T1} )
        (1-\theta^\star_{C1} )}{1-\hat{\theta}_{C1}}
    \hat{\theta}_{T1}
    -
    2(1-\hat{\theta}_{T1})
 \theta^\star_{T1}  = 0
 \end{align*}
 implies
 \begin{align*}
 \frac{
 (1-\theta^\star_{T1} )
        (1-\theta^\star_{C1} )}{1-\hat{\theta}_{C1}}
    \hat{\theta}_{T1}
     =  (1-\hat{\theta}_{T1})
 \theta^\star_{T1} \\
 \frac{
 (1-\theta^\star_{T1} )
        (1-\theta^\star_{C1} )}{1-\hat{\theta}_{C1}}
    \hat{\theta}_{T1}
    +
    \theta^\star_{T1}   \hat{\theta}_{T1}
     =  \theta^\star_{T1}  \\
   \Big(  \frac{
 (1-\theta^\star_{T1} )
        (1-\theta^\star_{C1} )}{1-\hat{\theta}_{C1}}
    +
    \theta^\star_{T1}
    \Big) \hat{\theta}_{T1}
     =  \theta^\star_{T1}  \\
 \hat{\theta}_{T1}
     = \frac{ \theta^\star_{T1} }{\Big(  \frac{
 (1-\theta^\star_{T1} )
        (1-\theta^\star_{C1} )}{1-\hat{\theta}_{C1}}
    +
    \theta^\star_{T1}
    \Big)}
\end{align*}
Now solving for $\hat{\theta}_{C1}$ in the \text{G-BS-CS} derivative:
\begin{align*}
    \frac{d\text{G-BS-CW}(1)}{d\hat{\theta}_{C1}}=2
 \frac{
  (1-\theta^\star_{T1})
 (1-\theta^\star_{C1} )
       }{1-\hat{\theta}_{T1}}
    \hat{\theta}_{C1}
    -
    2\frac{(1 - \theta^\star_{T1})(1-\hat{\theta}_{C1})
 \theta^\star_{C1}}{1 - \hat{\theta}_{T1}}  = 0
\end{align*}
implies
\begin{align*}
 \frac{
  (1-\theta^\star_{T1})
 (1-\theta^\star_{C1} )
       }{1-\hat{\theta}_{T1}}
    \hat{\theta}_{C1}
    =
    \frac{(1 - \theta^\star_{T1})(1-\hat{\theta}_{C1})}{1 - \hat{\theta}_{T1}}\theta^\star_{C1}
 \end{align*}
Given $1 - \theta^\star_{T1} \neq 0$ and $1 - \hat{\theta}_{T1} \neq 0$, we have
\[
 (1-\theta^\star_{C1})    \hat{\theta}_{C1} = (1-\hat{\theta}_{C1})\theta^\star_{C1}
\]
which gives us $\hat{\theta}_{C1}=\theta^\star_{C1}$.
Given $1 - \theta^\star_{T1} \neq 0$ and $1 - \hat{\theta}_{T1} \neq 0$, the above derivative equations jointly imply
\begin{align*}
    \hat{\theta}_{T1}
     =  
     \Big(\theta^\star_{T1} \Big)
     \Bigg( 
 \frac{
 (1-\theta^\star_{T1} )
        (1-\theta^\star_{C1} )}{1-\hat{\theta}_{C1}}
 +  \theta^\star_{T1}
\Bigg)^{-1}
,\quad 
     \hat{\theta}_{C1}
     =\theta^\star_{C1} 
\end{align*}
Substituting  $\hat{\theta}_{C1}
     =\theta^\star_{C1}$ in the formula for $\hat{\theta}_{T1}$ in terms of $\hat{\theta}_{C1}$, we have
     \[\hat{\theta}_{T1}
     = 
     \Big( \theta^\star_{T1}\Big)
     \Bigg( 
 \frac{
 (1-\theta^\star_{T1} )
        (1-\theta^\star_{C1} )}{1-\theta^\star_{C1}}
 +  \theta^\star_{T1} 
\Bigg)^{-1}
=\frac{\theta^\star_{T1} }
     {
 (1-\theta^\star_{T1} )
 +  \theta^\star_{T1}} =  \theta^\star_{T1}
\]
Therefore, under the assumptions, for the BS(1) case, we have the only stationary point at the two true 1st-timestep parameters: $\hat{\theta}_{T1}=\theta_{T1}^\star$
and
$\hat{\theta}_{C1}=\theta_{C1}^\star$.

\subsection{Induction step} 

We can proceed by induction over timesteps. Claim: given     $P_{\theta}(T \leq a) = P^\star(T \leq a)$
    and
   $
    P_{\theta}(C \leq a) = P^\star(C \leq a), \quad a=1,\dots,k$,
the stationary point of the game BS(k+1) has to satisfy
$ P_{\theta}(T = k+1) = P^\star(T=k+1)$
and
$P_{\theta}(C = k+1) = P^\star(C=k+1)$ i.e.
$\hat{\theta}_{T,k+1}=\theta^\star_{T,k+1}$
and
$\hat{\theta}_{C,k+1}=\theta^\star_{C,k+1}$.We first simplify F-BS-CW.
\begin{align*}
 \text{F-BS-CW}(k+1) =   \E_{T,C} \Big[\frac{(1 - F_\theta(k+1))^2 \indicator{T \leq C} \indicator{U \leq k+1}}{P_\theta(C^\prime \geq U)}
        + \frac{F_\theta(k+1)^2 \indicator{U > k+1}}{P_\theta(C^\prime > k+1)}\Big] 
\end{align*}
We simplify each term of F-BS-CW separately. The left term of F-BS-CW is
\begin{align*}
     &\E_{T,C} \frac{(1 - F_\theta(k+1))^2 \indicator{T \leq C} \indicator{U \leq k+1}}{P_\theta(C^\prime \geq U)}\\
     =&P_\theta(T>k+1)^2 \E_{T,C}\frac{ \indicator{T \leq C} \indicator{U \leq k+1}}{P_\theta(C^\prime \geq U)}\\
     =&P_\theta(T>k+1)^2\sum_{a=1}^{K}\sum_{b=1}^{K}P^\star(T=a)P^\star(C=b)\frac{ \indicator{a \leq b} \indicator{\min(a,b) \leq k+1}}{P_\theta(C^\prime \geq \min(a,b))}\quad \\
          & \quad \quad \quad 
          \Big[ \text{condition }\indicator{a \leq b} \text{moves from indicator to sum limits and } \min(a,b)=a \Big] \\
     =&P_\theta(T>k+1)^2\sum_{a=1}^{K}\sum_{b=a}^{K} \frac{P^\star(T=a)P^\star(C=b)  \indicator{a \leq k+1}}{P_\theta(C^\prime \geq a)}
     \\
      &\quad \quad \quad
      \Big[ 
     \text{condition } \indicator{a\leq k+1} \text{ moves from indicator to sum limit}\Big]\\
     =&P_\theta(T>k+1)^2\sum_{a=1}^{k+1}\sum_{b=a}^{K} \frac{ P^\star(T=a)P^\star(C=b) }{P_\theta(C^\prime \geq a)}\\
      =&P_\theta(T>k+1)^2\sum_{a=1}^{k+1}P^\star(T=a)\sum_{b=a}^{K} \frac{  P^\star(C=b)}{P_\theta(C^\prime \geq a)} \\
   =&P_\theta(T>k+1)^2\sum_{a=1}^{k+1}P^\star(T=a)\frac{ P^\star(C\geq a) }{P_\theta(C^\prime \geq a)} \\
    &\Bigg[\text{induction hypothesis: } P_{\theta}(C \leq a) = P^\star(C \leq a), \quad a=1,\dots,k \implies  P_{\theta}(C > a) = P^\star(C > a), \quad a=1,\dots,k\Bigg]\\
     =&P_\theta(T>k+1)^2\sum_{a=1}^{k+1}P^\star(T=a) \cdot 1 \\
    =& P_\theta(T>k+1)^2P^\star(T\leq k+1) \\
    =& (1 - \sum_{i=1}^{k}\tihat-\hat{\theta}_{T(k+1)})^2\sum_{i=1}^{k+1}\theta^\star_{Ti}\\
        & \quad \quad \quad \Big[ \text{induction hypothesis: }
         P_{\theta}(T \leq a) = P^\star(T \leq a),  \quad a=1,\dots,k\Big] \\
    =& (1 - \sum_{i=1}^{k}\theta^\star_{Ti}-\hat{\theta}_{T(k+1)})^2\sum_{i=1}^{k+1}\theta^\star_{Ti} \\
  =&(1-p-x)^2(p+t)\\
  \overset{\Delta}{=}& A,
  \quad \quad
  \text{where }
  p=\sum_{i=1}^k \titrue, q=\sum_{i=1}^k \citrue,x=\tkplusonehat,
y=\ckplusonehat.
,t=\tkplusonetrue
c=\ckplusonetrue.
\end{align*}
The right term of F-BS-CW is
\begin{align*}
 &\E_{T,C} \frac{F_\theta(k+1)^2 \indicator{U > k+1}}{P_\theta(C^\prime > k+1)}  \\
 &=\frac{F_\theta(k+1)^2}{P_\theta(C^\prime > k+1)} \E_{T,C}  \indicator{U > k+1}  \\
 & \quad \quad \quad 
 \Bigg[ 
    T ~\text{and}~C~\text{are independent means } \indicator{U>z}=\indicator{T>z}\indicator{C>z}\Bigg]\\
 &=\frac{F_\theta(k+1)^2}{P_\theta(C^\prime > k+1)} P^\star(T>k+1)P^\star(C>k+1) \\
 &= \frac{(\sum_{i=1}^{k+1} \tihat)^2}{1-\sum_{i=1}^{k+1} \cihat}(1 - \sum_{i=1}^{k+1}\titrue)(1 - \sum_{i=1}^{k+1}\citrue) \\
 &
 \quad 
 \Bigg[ 
 \text{induction hypothesis: }
 P_{\theta}(T \leq a) = P^\star(T \leq a)
 \text{ and }
P_{\theta}(C \leq a) = P^\star(C \leq a), \quad a=1,\dots,k
 \Bigg]\\
 &=\frac{(\sum_{i=1}^{k} \titrue + \tkplusonehat)^2}{1-\sum_{i=1}^{k} \citrue -\ckplusonehat}(1 - \sum_{i=1}^{k+1}\titrue)(1 - \sum_{i=1}^{k+1}\citrue)\\
 &=\frac{(p + x)^2}{1-q-y}(1-p-t)(1-q-c) \triangleq B
 \end{align*}
where again $ p=\sum_{i=1}^k \titrue, q=\sum_{i=1}^k \citrue,x=\tkplusonehat,
y=\ckplusonehat
,t=\tkplusonetrue,
c=\ckplusonetrue$. To summarize, F-BS-CW$(k+1)=A+B$:
\begin{align*}
\text{F-BS-CW}(k+1) =  (1-p-x)^2(p+t) + 
    \frac{(p + x)^2}{1-q-y}(1-p-t)(1-q-c)
\end{align*}
Then we simplify G-BS-CW.
\begin{align*}
    \text{G-BS-CW}(k+1)= \E_{T,C} \Big[\frac{(1 - G_\theta(k+1))^2 \indicator{C < T} \indicator{U \leq k+1}}{P_\theta(T^\prime > U)}
        + \frac{G_\theta(k+1)^2 \indicator{U > k+1}}{P_\theta(T^\prime > k+1)}\Big]\\
\end{align*}
The left term of G-BS-CW
\begin{align*}
&\E_{T,C} \frac{(1 - G_\theta(k+1))^2 \indicator{C < T} \indicator{U \leq k+1}}{P_\theta(T^\prime > U)}\\
=&  (1 - G_\theta(k+1))^2   \E_{T,C} \frac{ \indicator{C < T} \indicator{U \leq k+1}}{P_\theta(T^\prime > U)}\\
=& (1 - G_\theta(k+1))^2  \sum_{a=1}^K\sum_{b=1}^K P^\star(C=a)P^\star(T=b) \frac{ \indicator{a < b} \indicator{\min(a,b) \leq k+1}}{P_\theta(T^\prime > \min(a,b))}\\
 & \quad \quad \quad \text{condition }\indicator{a < b} \text{moves from indicator to sum limits and} \min(a,b)=a\\
=&(1 - G_\theta(k+1))^2  \sum_{a=1}^K\sum_{b=a+1}^K  \frac{ P^\star(C=a)P^\star(T=b) \indicator{a \leq k+1}}{P_\theta(T^\prime > a)}\\
& \quad \quad \quad 
\text{condition } \indicator{a \leq k+1} \text { moves from indicator to sum limits}\\
=&(1 - G_\theta(k+1))^2  \sum_{a=1}^{k+1}\sum_{b=a+1}^K \frac{P^\star(C=a)P^\star(T=b) }{P_\theta(T^\prime > a)}\\
& \quad \quad \quad 
\Bigg[
    \text{split sum over a into two terms: 1 through k, and k+1, recall b starts at a+1}\Bigg]
\\
=&(1 - G_\theta(k+1))^2  
    \left(\sum_{a=1}^{k}\sum_{b=a+1}^K  \frac{P^\star(C=a)P^\star(T=b)}{P_\theta(T^\prime > a)} + \sum_{b=k+2}^K \frac{ P^\star(C=k+1)P^\star(T=b) }{P_\theta(T^\prime > k+1)}\right)\\
=&(1 - G_\theta(k+1))^2  \left(\sum_{a=1}^{k} P^\star(C=a)\sum_{b=a+1}^K \frac{P^\star(T=b) }{P_\theta(T^\prime > a)} +  P^\star(C=k+1)\sum_{b=k+2}^K \frac{P^\star(T=b)}{P_\theta(T^\prime > k+1)}\right)\\
=&(1 - G_\theta(k+1))^2 \left( \sum_{a=1}^{k}  \frac{ P^\star(C=a)P^\star(T\geq a + 1)}{P_\theta(T^\prime > a)} +  \frac{ P^\star(C=k+1)P^\star(T>k+1)}{P_\theta(T^\prime > k+1)}\right)\\
=&(1 - G_\theta(k+1))^2  \left(\sum_{a=1}^{k} \frac{  P^\star(C=a)P^\star(T> a )}{P_\theta(T^\prime > a)} +   \frac{P^\star(C=k+1)P^\star(T>k+1)}{P_\theta(T^\prime > k+1)}\right)\\
&\Bigg[\text{induction hypothesis: }P_{\theta}(T \leq a) = P^\star(T \leq a), \quad a=1,\dots,k \implies  P_{\theta}(T > a) = P^\star(T > a), \quad a=1,\dots,k\Bigg]\\
=&(1 - G_\theta(k+1))^2  \left(\sum_{a=1}^{k} P^\star(C=a)+  \frac{ P^\star(C=k+1)P^\star(T>k+1)}{P_\theta(T^\prime > k+1)}\right)\\
 =&(1-\sum_{i=1}^{k}\cihat-\ckplusonehat)^2\left(\sum_{i=1}^k\citrue + \frac{\ckplusonetrue (1-\tkplusonetrue-\sum_{i=1}^k \titrue)}{1-\sum_{i=1}^k \tihat -\tkplusonehat}\right) \\
  & \quad 
 \Bigg[ 
 \text{induction hypothesis: } 
 P_{\theta}(T \leq a) = P^\star(T \leq a)
    \quad \text{ and }
    \quad 
    P_{\theta}(C \leq a) = P^\star(C \leq a), \quad a=1,\dots,k
    \Bigg]\\
 =&(1-\sum_{i=1}^{k}\citrue-\ckplusonehat)^2\left(\sum_{i=1}^k\citrue + \frac{\ckplusonetrue (1-\tkplusonetrue-\sum_{i=1}^k \titrue)}{1-\sum_{i=1}^k \titrue -\tkplusonehat} \right)\\
=& (1-q-y)^2(q+\frac{c(1-t-p)}{1-p-x})
\overset{\Delta}{=} C
\end{align*}
By symmetry with $B$, the right term is
\begin{align*}
    &\E_{T,C}\frac{G_\theta(k+1)^2 \indicator{U > k+1}}{P_\theta(T^\prime > k+1)}
    =\frac{(q + y)^2}{1-p-x}(1-q-c)(1-p-t)
    \overset{\Delta}{=}D
\end{align*}

Again using $ p=\sum_{i=1}^k \titrue, q=\sum_{i=1}^k \citrue,x=\tkplusonehat,
y=\ckplusonehat
,t=\tkplusonetrue,
c=\ckplusonetrue$,
we have 
\begin{align*}
\text{G-BS-CW}(k+1) &= C + D \\
&=
(1-q-y)^2(q+\frac{c(1-t-p)}{1-p-x})
+
\frac{(q + y)^2}{1-p-x}(1-q-c)(1-p-t)
\end{align*}
The stationary point satisfies
\begin{align*}
    \frac{\partial \text{G-wt-FBS(k+1)}}{\partial x}
    &= \frac{\partial A}{\partial x} +  \frac{\partial B}{\partial x}\\
    &=-2(1-p-x)(p+t) + 2\frac{(p + x)}{1-q-y}(1-p-t)(1-q-c)
   \\
   &=0\\
 \frac{\partial \text{F-wt-GBS(k+1)}}{\partial y}
 &=\frac{\partial C}{\partial y} + \frac{\partial D}{\partial y}\\ 
 &= -2(1-q-y)(q+\frac{c(1-t-p)}{1-p-x}) + 2\frac{(q + y)}{1-p-x}(1-q-c)(1-p-t)
 = 0
\end{align*}
It's a system of quadratic equations with two unknowns. The system has analytical solutions. Solving the above equations for $x, y$ by \textit{Mathematica} (it is quite a long derivation manually), the solutions are 
\[
x = t, y = c
\]
or 
\begin{align*}
    x =& (1/(-q + q^2 + 
   q c))(c p - q c p - q t + q^2 t + c t \\
  & - (
    p (-1 + q + c + q p - q^2 p - c p + q t - q^2 t - 
       c t))/((-1 + q) (p + t)) \\
       &+ (
    q p (-1 + q + c + q p - q^2 p - c p + q t - q^2 t - 
       c t))/((-1 + q) (p + t)) \\
       &- (
    t (-1 + q + c + q p - q^2 p - c p + q t - q^2 t - 
       c t))/((-1 + q) (p + t)) \\
       &+ (
    q t (-1 + q + c + q p - q^2 p - c p + q t - q^2 t - 
       c t))/((-1 + q) (p + t)))\\
 y =& (-1 + q + c + q p - q^2 p - c p + q t - q^2 t - 
   c t)/((-1 + q) (p + t))\\
\end{align*}
To check if this second solution is valid,
it would need to be the case that $q+y < 1$ because we only consider $k+1 < K$. If we ask mathematica to simplify q+y that satisfies the above solution, then this holds:
\begin{align*}
    q+y = \frac{-1+q-c(-1+p+t)}{(-1+q)(p+t)}
\end{align*}
The numerator and the denominator are both negative. If $k+1 <K$ (we know \gls{bs} at K is 0 and also we only have K-1 parameters), the numerator minus denominator =
\begin{align*}
    -1+q-c(-1+p+t) - (-1+q)(p+t) &= (-1+q)(1-p-t) -c(-1+p+t)\\
    &=(-1 +q+c)(1-p-t) \\
    &<0
\end{align*}
Therefore, 
\[
\sum_{i=1}^{k} \citrue + \ckplusonehat  = q + y >1
\]
This is invalid. So 
\[
x = t, y = c
\]
is the only solution, i.e., $\tkplusonehat = \tkplusonetrue, \ckplusonehat = \ckplusonetrue$. By induction, we conclude that 
\[
\tihat = \titrue, \cihat = \citrue, i= 1,\dots, K-1
\]
By $\hat{\theta}_{TK} = 1- \sum_{i=1}^{K-1} \tihat$ and $\hat{\theta}_{CK} = 1- \sum_{i=1}^{K-1} \cihat$, we have 
\[
\hat{\theta}_{TK} = \theta^\star_{TK}, \hat{\theta}_{CK} = \theta^\star_{CK}
\]
Therefore, 
\[
\tihat = \titrue, \cihat = \citrue, i= 1,\dots, K
\]
is the only stationary point for the game.
\end{proof}

\end{document}